\documentclass{article}
 
\usepackage{graphicx}
\usepackage{epstopdf}
\usepackage{url}
\usepackage{amsmath,amssymb,amsthm}
\usepackage{amsfonts}
\newcommand{\Bem}[1]{}
\newcommand{\deleted}[1]{}
\newcommand{\inserted}[1]{{ #1 }}%{{\bf  #1 }}
\newcommand{\mysvli}[1]{{\bf \it #1 }}
\newcommand{\changed}[2]{\deleted{#1}\inserted{#2}}
\newtheorem{ax}{Property}%[section]
\newtheorem{definition}{Definition}%[section]
\newtheorem{theorem}{Theorem}%[section]
\newcommand{\refax}[1]{Property \ref{#1}\/}

  \newcommand{\correspondingauthor}[1]{}
  \newcommand{\Runauthors}[1]{}
  \newcommand{\address}[1]{ \\  #1 }
\correspondingauthor{Mieczys{\l}aw K{\L}OPOTEK}

\Runauthors{R. K{\L}OPOTEK and M. K{\L}OPOTEK}

\newcommand{\figaddr}[1]{#1}

\begin{document}

\title{On the Discrepancy  Between Kleinberg's Clustering Axioms and $k$-Means Clustering Algorithm Behavior}
\author{Robert K{\L}OPOTEK   \address{% 
Faculty of Mathematics and Natural Sciences,\\
Stefan Cardinal  Wyszy\'{n}ski University in Warsaw
\\ul. Wóycickiego 1/3 building 21, 01-938 Warszawa, Poland
\\ e-mail: \url{robert@klopotek.com.pl}} \and  
 Mieczys{\l}aw K{\L}OPOTEK  \address{%
  Institute of Computer Science,  Polish Academy of Sciences
\\   ul. Jana Kazimierza 5, 01-248 Warszawa, 
Poland
 \\ e-mail: \url{klopotek@ipipan.waw.pl}}
}{}{}

\maketitle

\begin{abstract}
This paper investigates the  validity of Kleinberg's axioms for clustering functions with respect to the quite popular clustering algorithm called $k$-means.%, in its "natural" application domain that is in the Euclidean space. 
We suggest that the reason why this algorithm does not fit  Kleinberg's axiomatic system stems from missing match between informal intuitions and formal formulations of the axioms.
\inserted{While Kleinberg's axioms have been discussed heavily in the past, we concentrate here on the case predominantly relevant for $k$-means algorithm, that is behavior embedded in Euclidean space.}
We point at some contradictions and counter intuitiveness aspects of this axiomatic set within $\mathbb{R}^m$ that were evidently  not discussed so far. 
Our results suggest that apparently without defining clearly what kind of clusters we expect we will not be able to construct a valid axiomatic system. In particular we look at the shape and the gaps between the clusters.   
%We show that the consistency and scale-invariance axioms, even ignoring %richness axiom, yield a contradiction to the intuitions behind clustering. 
%Furthermore we show that consistency alone leads to contradictions. 
%We demonstrate also that the richness axiom is counterintuitive by itself. 
Finally we demonstrate that there exist several ways to reconcile the formulation of the axioms with their intended meaning and that under this reformulation  
the axioms stop to be contradictory 
and the real-world $k$-means algorithm    conforms to this axiomatic system.
% its contradictions with the learnability theory. 
%We propose a way of modifying definitions underlying Kleinberg axioms so that %there is no contradiction in his axiomatic system.
%Finally we develop a new set of axioms for clustering, 
%that is more practice relevant. 
\end{abstract}

\section{Introduction}
 
One of important application areas of machine learning is the so-called \emph{cluster analysis} or \emph{clustering}, referred to also as \emph{unsupervised learning} or \emph{learning without a teacher}. It seeks to split a set of items into subsets (usually disjoint, though not necessarily, possibly with the subsets forming a hierarchy) called 
clusters or groups that should be "\emph{similar}" within the clusters and "\emph{dissimilar}" between them. Additional criteria like group balancing, group size limits from below and above etc. may be also taken into account. Subsequently let us restrict somehow the meaning of these terms. By \emph{partition} we will understand 
the output of the process of \emph{cluster analysis}. So the partition would be an object - a set of objects called clusters that are sets of original items (called elements). 
   
As the diversity of clustering methods grows, there exists a strong pressure for finding some formal framework to get a systematic overview of the expected properties of the partitions obtained. 

A number of axiomatic frameworks have been devised for methods of clustering, the most cited probably 
the Kleinberg's system \cite{Kleinberg:2002}\footnote{Google Scholar lists about 400 citations.}.
Kleinberg defines \cite{Kleinberg:2002}  clustering function as
\begin{definition}\label{def:KleinbergClusteringFunction}
Clustering function is "a function $f$ 
that takes a distance function $d$ on [set] $S$ [of size $n\ge 2$] and
returns a partition $\Gamma$  of $S$. The sets in $\Gamma$  will 
be called its  \emph{clusters}." 
We are interested only in such partitions 
$\Gamma$  of $S$ that 
$\cup_{C\in \Gamma} C=S$, $C_i\ne \emptyset$  and 
for any two distinct $C_i,C_j\in \Gamma$ $C_i\cap C_j=\emptyset$. 
\end{definition}

Additionally, he defines the distance as 

\begin{definition}\label{def:KleinbergDistance}
"with
the
set
$S=\{1,2,\dots,n\}$ 
[...] we define 
a
distance function
to be any function
$d
:
S
\times 
S
\rightarrow
\mathbb{R}$
such that
for
distinct
$i,j\in S $ 
 we have
$di,j)>0, d(i,j)=d(j,i)$ and $d(i,i)=0$. 
  One
can
optionally
restrict
attention
to distance
functions
that
are
metrics
by imposing
the
triangle
inequality: $d(i,k)\le d(i,j)+d(j,k)$, 
   for all $i,j,k\in S$.  
  We will not require the triangle inequality [...]%in the discussion here
,
but the results to follow 
both negative and positive still
hold
if one
does require"
\end{definition}

Jon Kleinberg  \cite{Kleinberg:2002}   claims that a good partition may only be 
a result of a reasonable method of clustering and he formulated axioms, for 
distance-based cluster analysis, that need to be met by the clustering method 
itself. He postulated that some quite "natural" axioms need to be met, when 
we manipulate the distances between objects. 
As, however, the axioms proved to be not applicable to all clustering algorithms, we will rather speak about properties that 
Kleinberg expects of clustering functions, following   e.g. Ackerman et al. \cite{Ackerman:2010NIPS}. 
These are: 
\begin{ax} \label{ax:richness}
The method should allow to obtain any partition of the objects (so-called \emph{richness} property)%
\footnote{
"let $Range(f)$ denote the set of all partitions $\Gamma$
 such that $f(d) = \Gamma$ for some distance
function $d$.
  $Range(f)$ is equal to the set of all partitions of $S$."
\cite{Kleinberg:2002}}, 
\end{ax} 
\begin{ax}  \label{ax:scaleinvariance}
   The method should deliver partitions invariant with respect to distance 
scale (so-called \emph{scale-invariance}  property)%
\footnote{
"For any distance function $d$ and any $\alpha  > 0$,
we have $f(d) = f(\alpha \cdot d)$."
\cite{Kleinberg:2002}},
\end{ax} 
\begin{ax}  \label{ax:consistency}
The method should deliver the same partition if we move  elements within a cluster closer to one another 
% (or closer to  cluster centers to which they are assigned) 
and elements from different clusters further away 
 (so-called \emph{consistency} property)%
\footnote{
"Let $\Gamma$ be a partition of $S$, and $d$ and
$d'$  two distance functions on $S$. We say that $d'$ 
 is a $\Gamma$-transformation of $d$ if (a) for
all $i,  j \in  S$ belonging to the same cluster of 
$\Gamma$, we have $d'(i, j) \le d(i, j)$   and (b) for
all $i,  j \in  S$ belonging to different clusters of $\Gamma$,
 we have $d'(i, j) \ge d(i, j)$.
  Let $d$ and $d'$ be two distance functions. If $f(d) =\Gamma$,
and $d'$ is a $\Gamma$-transformation of $d$, then $f(d') = \Gamma$" 
\cite{Kleinberg:2002}.
This should reflect the property of reducing distance within a cluster and enlarging that between the clusters. 
}. 
\end{ax} 
%Though the axioms may seem to be reasonable,

Note that invariance and consistency properties  assume a transformation on the clusters. With respect to this transformations we will speak about invariance transform(ation) and consistency transform(ation). 

Subsequently, while referring to Kleinberg's axiomatic systems, we will use the term "axioms", but keeping in mind, that researchers treat them rather as properties that some algorithms have and other don't. 

 Kleinberg demonstrated that 
the above three "axioms" (properties) cannot be met all at once.
% (but only pair-wise). 
So Kleinberg's work points 
at an important issue that we shall first of all revise our expectations 
towards the obtained partition, as the seemingly obvious axiom set is 
apparently not sound.  
In particular he stated the   Impossibility Theorem.

\begin{theorem}{\cite[Theorem 2.1]{Kleinberg:2002}} \label{thm:KleinbergImpossibility}
For
each
$n\ge 2$,
there is no   clustering  function 
 $f$
that satisfies
Scale-Invariance,
Richness,
and
Consistency.
\end{theorem}

Kleinberg himself proved this theorem in the above-mentioned paper. 
Another proof can be found in 
a paper by Ambroszkiewicz and Koronacki 
\cite{Ambroszkiewicz:2010}, along with some discussion of the Kleinberg's concepts. 
Ackerman et al. \cite{Ackerman:2010NIPS} prove a bit more general impossibility theorem (engaging so called inner-consistency and outer-consistency).

Beside providing a proof that his axioms are contradictory, Kleinberg showed that the axioms can be met pairwise. 
He uses for purpose of this demonstration  
versions of the well-known statistical  single-linkage procedure. 
The versions differ by the stopping condition:  
\begin{itemize}  
\item  
 $k$-cluster stopping condition (which stops adding edges when the sub-graph first 
consists of $k$ connected components)   
-  not "rich",
\item  
 distance-$r$ stopping condition (which adds edges of weight at most $r$ only) -  not scale-invariant,  
\item  
 scale-stopping condition (which adds edges of weight being at most some percentage of the largest distance between nodes)  
- not consistent%
\footnote{ 
Notice that, as demonstrated by Kleinberg in his paper,   
also  
 $k$-median   and  
$k$-means clustering do not  
have the consistency property.  
}. 
\end{itemize}
Note, however, that   Ben-David and Ackerman \cite{Ben-David:2009} drew attention by an illustrative example (their Figure 2), that consistency is a problematic property by itself as it may give rise to new clusters at micro or macro-level. 

Let us draw attention to the fact that by introduction of his definition of clustering function, Kleinberg introduces implicitly  two additional axioms onto the clustering function: 
\begin{ax}\label{ax:nonFalsifiability}
A clustering function always returns a clustering
(\emph{non-refutability}).
\end{ax}
\begin{ax}\label{ax:nonEmbedding}
A clustering function works even if the distances cannot be embedded in Euclidean space 
(\emph{permission of non-embeddability}).
\end{ax}

The well-known $k$-means clustering algorithm 
seeks to minimize the function\footnote{
The considerations would apply also to 
 kernel $k$-means algorithm 
using the quality function 
$$Q(\Gamma)=\sum_{i=1}^m\sum_{j=1}^k u_{ij}\|\Phi(\textbf{x}_i - \boldsymbol{\mu}^\Phi_j\|^2$$
\noindent 
where $\Phi$ is a non-linear mapping from the original space to the so-called feature space.
} 
\begin{equation} \label{eq:Q::kmeans}
Q(\Gamma)=\sum_{i=1}^m\sum_{j=1}^k u_{ij}\|\textbf{x}_i - \boldsymbol{\mu}_j\|^2
=\sum_{j=1}^k \frac{1}{n_j} \sum_{\mathbf{x}_i, \mathbf{x}_l \in C_j} \|\mathbf{x}_i - \mathbf{x}_l\|^2 
\end{equation} 
for a dataset $\mathbf{X}$
under some partition $\Gamma$ into the predefined number $k$ of clusters, 
where  $u_{ij}$ is an indicator of the membership of data point $\textbf{x}_i$ in the cluster $C_j$ having the center at $\boldsymbol{\mu}_j$. 

We will call   "$k$-means-ideal" such an algorithm that finds a $\Gamma_{opt}$ that attains the minimum of function $Q(\Gamma)$. 
It is known that it is a hard task.
\footnote{There exists a whole stream of research papers that attempt to approximate 
$k$-means-ideal within a reasonable error bound via cleverly initiated $k$-means type algorithms, e.g. $k$-means++, like 
\cite{Song:2010}, but it has to be stated that at the current point 
these algorithms are rather of theoretical value. 
}
 Hence in practice an algorithm is used with the following structure:
\begin{enumerate}
\item Initialize $k$ cluster centers $\boldsymbol\mu_1,\dots,\boldsymbol\mu_k$.
\item Assign each data element $\mathbf{x}_i $ to the cluster $C_j$ identified by the closest $\boldsymbol\mu_j$.
\item Update $\boldsymbol\mu_j$ of each cluster $C_j$ as the gravity center of the data elements in $C_j$. 
\item Repeat steps 2 and 3 until reaching a stop criterion (usually no change of cluster membership). 
\end{enumerate}

If step 1 is performed as random uniform sampling (without replacement), then we will speak about $k$-means-random algorithm. 
If step 1 is performed according to $k$-means++ heuristics
proposed by Arthur and Vassilvitskii 
 \cite{CLU:AV07},  then we will speak about $k$-means++ algorithm. 
Note that both attain a local minimum at worst. 
We will also touch the "incremental" $k$-means discussed by Ackerman and Dasgupta  \cite{Ackerman:2014nips}. 
This $k$-means version does not guarantee to reach a local minimum and has purely theoretical virtues. 

The verification of Kleinberg's axioms for $k$-means is a bit difficult because even for $k$-means-ideal we cannot guarantee that there exists a single (global) minimum of the $Q$ function. But if we talk instead of the set of all possible minimizing $\Gamma$s, then 
it is easily seen that it is scale-invariant, but 
one sees immediately that it is not rich (only partitions with $k$ clusters are considered).
It has also been 
demonstrated by Kleinberg that it is not consistent. 

With $k$-means-random and $k$-means++ it is even worse, as $Q$ hits usually a local minimum there. So we can talk about a random variable assuming particular $\Gamma$ with some probability. 
Under this assumption, again it is easily seen that both are scale-invariant, but  one sees immediately that none is rich (only partitions with $k$ clusters are considered).
As $k$-means-ideal is not consistent, so neither of the realistic variants is so.  

Hence the widely used algorithm violates in practice two of three Kleinberg's axioms. 
 so that it cannot be considered to be a "clustering function". 
We perceive this to be at least counterintuitive.
\inserted{
Ben-David and Ackerman in \cite{Ben-David:2009} in section 4.2., 
raised also similar concern from the perspective   of what an axiomatic system should accomplish. 
They state that one would expect, for the axiomatised set of objects, a kind of soundness and completeness.
By soundness they mean  that most useful clustering algorithms would fit the axioms.
The completeness expresses that apparent non-clustering algorithms would fail on at least one axiom. 
While Kleinberg's axioms explicitly address the distance-based clustering algorithms (and not e.g. density based ones), they fall apparently short of reaching this goal. 
In this paper we demonstrate that even for a narrower set of algorithms,
ones over data embedded in Euclidean space, the axioms fail. 
}%inserted{

\deleted{ 
and in this paper we demonstrate that Kleinberg's axioms are not that obvious as one could expect at first glance. }
\inserted{

There exist a number of open questions on why it is so. 
Recall that in \cite{vanLaarhoven:2014} it has been observed 
by van Laarhoven and Marchiori 
that Kleinberg's proof of Impossibility Theorem stops to be valid 
in case of graph clustering. 
This raises immediately the question of its validity in $\mathbb{R}^m$ Euclidean space. Note that Kleinberg did not bother about embedding the distance in such a space. 
So one may ask whether or not $k$-means does not fit Kleinberg's axioms 
because this is a peculiar property of $k$-means or because 
any algorithm embedded in Euclidean space would fail to fit. 

Paying a special attention to $k$-means algorithm does not constitute a too restrictive limitation.
$k$-means is applied in many domains, not only in its natural domains of data embedded in $\mathbb{R}^m$, where clusters may be enclosed into Voronoi regions, 
but also to non-linearly separable clusters (via kernel functions \cite{Handhayania:2015}), to manifolds \cite{Wei:2010}, 
in spectral clustering \cite{Dhillon:2004} and in community detection in social networks \cite{Liu:2010}. 
It has been demonstrated by Dhillon et al. \cite{Dhillon:2004} that $k$-means is equivalent in some sense to normalized cut method of graph clustering,
which in turn can be viewed as equivalent to balanced Newman's modularity, used in community detection, as shown by Bolla \cite{Bolla:2011}. 
Crisp and fuzzy versions are used. 

Among others, the equivalence results on $k$-means and graph clustering in \cite{Dhillon:2004} and the impossibility theorem challenge for graphs in \cite{vanLaarhoven:2014}
encourage to investigate Kleinberg's axioms in the context of Euclidean space. 

}%inserted{

  Therefore we made an effort to identify and overcome at least some reasons for the difficulties connected with axiomatic understanding of  research area of cluster analysis and hope that this may be a guidance for further generalizations to encompass if not all then at least a considerable part of the real-life algorithms. 
This paper investigates why the $k$-means algorithm violates the  Kleinberg's axioms for
clustering functions. We claim that the reason is a mismatch between informal intuitions and formal formulations of these axioms. We claim also that there is a way to reconcile $k$-means with Kleinberg's consistency requirement via introduction of centric consistency and motion consistency which are neither a subset nor superset of Kleinberg's consistency, but rather a $k$-means clustering model specific adaptation of the general idea of shrinking the cluster or moving cluster away. 

%Bem 

%   

Our contribution is as follows:
\begin{itemize} 
\item To substantiate our claim that there is a mismatch between informal intuitions and formal formulations of Kleinberg's axioms, we present a series of  carefully constructed examples. 
\item We show in Section \ref{sec:kernel} that by the implicit non-embeddability axiom   alone Kleinberg precludes consideration of $k$-means as a clustering algorithm. 
\item 
  We demonstrate in Section \ref{sec:badscalabilityconsistency} that richness and scaling-invariance alone may lead to a contradiction for a special case.
This denies   Kleinberg's claims that his axioms can be fulfilled pair-wise.
\item In Section \ref{sec:kmeansContraKleinberg} we show that known relationships between 
Kleinberg's axioms and $k$-means apply also for Euclidean space, that is 
$k$-richness  is granted, richness or near richness is not achievable, consistency is violated. We show also that the refinement consistency is violated too. 
\item We show in Section \ref{sec:consistecycontrascaling} that in $\mathbb{R}^m$ scaling invariance transformations, by interference, annihilate effects of consistency transformation that is clusters being further away may get closer to one another. 
\item Furthermore we show in Section \ref{sec:badconsistency} that consistency alone leads to contradictions. 
  We demonstrate that in practical settings of application of many algorithms. In a metric $m$-dimensional space where $m$ is the number of features, it is impossible to contract a single cluster without moving the other ones and as a consequence running at risk of moving some clusters closer together. Also we  show   that $k$-means version  where we allow for $k$ to range over a set, will change the optimal clustering $k$ when Kleinberg's $\Gamma$ operation (consistency transform) is applied.
\item 
We demonstrate in Section \ref{sec:badrichness} that also  the richness axiom denies common sense by itself, as it is unrealistic to be achieved by $k$-means-ideal, $k$-means-random and $k$-means++. 
\item We propose a reformulation of the Kleinberg's axioms in accordance with the intuitions and demonstrate   that under this reformulation  
the axioms stop to be contradictory 
 (Section
\ref{sec:KleinbergNotBad}). 
In particular we introduce the notion of centric consistency which is an  
 adaptation of the general idea of shrinking the cluster. It relies simply on moving cluster elements towards its center. 
We provide an example of a clustering function that fits the axioms of near-richness, scale-invariance and \inserted{possesses the property of} centric consistency, \inserted{so that it is clear that they are not contradictory.}
\item We show that $k$-means is centric-consistent (Section \ref{sec:kmeanscentricconsistent}).
This implies that   even a real-world algorithm like $k$-means conforms to the above-mentioned augmented axiomatic system (section \ref{sec:kmeanscentricconsistent}). 
\item 
As the centric consistency imitates only the consistency inside a cluster,  
we introduce also the notion of motion consistency, to approximate the  consistency property outside a cluster and show that $k$-means, in order to be motion-consistent (Section \ref{sec:motionconsistency}), must impose the requirement of a gap between clusters. The introduction of gap requirement, on the other hand, violates Kleinberg's non-refutability axiom (\refax{ax:nonFalsifiability}).  
\item We investigate the issue of gaps between clusters and show appropriately designed gaps induce local minima (section \ref{sec:kmeansperfectballclustering}) for $k$-means and formulate conditions under which the gap leads to a global minimum for $k$-means (section \ref{sec:kmeansabsoluteballclustering}). 
\item Based on the above, we propose an alternative approach to reconcile Kleinberg's axioms with $k$-means. 
We demonstrate that under assumption of appropriate gaps we can either relax centric consistency to inner cluster consistency or go over from $k$-richness to an approximation of richness (sections  \ref{sec:kmeansperfectballclustering} and \ref{sec:kmeansabsoluteballclustering})
\end{itemize}

We  start this paper with a review of the previous work on development of an axiomatic system 
(Section \ref{sec:prevWork})
% and some remarks on the relationship between Kleinberg's axioms and $k$-means, (Section \ref{sec:kmeansContraKleinberg}) 
and round the paper up 
with a discussion of some open problems (Section \ref{sec:conclusions}).

\section{Previous work}\label{sec:prevWork}

Axiomatic systems may be traced back to as early as 1973, when Wright
\cite{Wright:1973} 
  proposed axioms of clustering functions
creating unsharp partitions, similar to fuzzy systems. 
  In  his framework every domain object was attached  a positive real-valued weight, that could be    distributed among multiple clusters.

In general, as exposed by van Laarhoven and Marchiori
\cite{vanLaarhoven:2014} and Ben-David and Ackerman  \cite{Ben-David:2009}
 the clustering axiomatic frameworks address either:
\begin{itemize} 
\item 
required properties of  clustering  functions, or
\item 
required properties of the   values of a clustering quality function,
or
\item 
required properties of the relation between qualities of different partitions (ordering   of  partitions  for  a  particular set of objects and distance or similarity or dissimilarity relations).
\end{itemize}

One of prominent axiomatic sets, that were later fiercely discussed, was that of Kleinberg, as already stated. 
From the point of view of the above classification, it imposes restrictions on the clustering function itself.

We have already discussed the Impossibility Theorem of Kleinberg that demonstrates the contradiction between the axioms of the set. 
However, there are further problems with this set, not covered by that Theorem.
So Ben-David and Ackerman \cite{Ben-David:2009}, as mentioned,  pointed at the problems with consistency as such. They showed in an example in their Fig.2 that when moving clusters away the clusters themselves can create new groups. 
In this paper we repeat their findings for fix-dimensional Euclidean space, but we go beyond that. We draw  attention to the fact that in $\mathbb{R}^m$ moving clusters away may be completely impossible without going into other dimension.
Furthermore we show that also shrinking of a single cluster in a consistent way is also impossible in $\mathbb{R}^m$.
We demonstrate that interaction of consistency transformation and scaling-invariance transformation actually does something contrary to intuition behind consistency, that is it pulls cluster closer instead of pushing them away.

A number of relaxations of axioms related to clustering functions have been proposed in order to overcome the Kleinberg's impossibility result.
We recall several of them here, based on 
 an overview by Ackerman \cite{Ackerman:2012:phd}
and tutorial by Ben-David \cite{Ben-David:2005}. 

So it was proposed to weaken   Kleinberg's richness (by Kleinberg himself) to so-called $k$-richness as follows: 
\begin{ax}[Zadeh and  Ben-David \cite{Zadeh:2009}] \label{ax:k-richness}
For any partition $\Gamma$ of the set $\mathbf{X}$  consisting of exactly $k$ clusters there exists such a distance function $d$ that   the   clustering function $f(d)$ 
returns this partition $\Gamma$.  %bendavid
\end{ax} 
This relaxation%
\footnote{Still another relaxation of richness was proposed by Hopcroft and   Kannan  \cite{Hopcroft:2012}: 
Richness II: For any set ${K}$ of ${k}$ distinct points in the given Euclidean space, there is an ${n}$ and a set of ${S}$ of ${n}$ points such that the algorithm on input ${S}$ produces ${k}$ clusters, whose centers are the respective points in ${K}$.
Here the weakness lies in the fact that the $k$ points may be subject to clustering themselves in reasonable algorithms. 
} 
 allows for some algorithms splitting the data into a fixed number of clusters, like $k$-means, 
not to be immediately discarded   as "clustering algorithms", given that no cluster is allowed to be empty.% 
\footnote{Even $k$-richness is still a problematic issue because as demonstrated by Ackerman et al. \cite{Ackerman:2013}, a useful property of stability of clusters under malicious addition of data points holds only for balanced clusters.}

However, this weakening of Kleinberg's axioms  does not suffice to make $k$-means a "clustering function" as it still   violates consistency axiom. 

Ackerman et al.  \cite{Ackerman:2010NIPS}  
 propose the concept of \emph{outer-consistency} 
\begin{ax}  \label{ax:outer-consistency}
The method is said to be \emph{outer-consistent}
if it  delivers the same partition if   one increases only distances between elements from different clusters and lets the distances within clusters unchanged. 
\end{ax}
$k$-means algorithm is said to be in this sense outer-consistent.%
\footnote{We show, however, that this is not true.}
They propose also so-called \emph{inner consistency}  
\begin{ax}  \label{ax:inner-consistency}
The method is said to be \emph{inner-consistent}
if it  delivers the same partition 
when   one decreases only distances between elements from
same  cluster  and lets the distances between elements of different clusters unchanged. 
\end{ax}
$k$-means algorithm is in this sense not inner-consistent.
Later we will discuss the representation problem for this type of consistency with $k$-means. 
Let us mention here that they prove that 
(1) no  general  clustering  function  can  simultaneously  satisfy  outer-consistency,  scale-
invariance, and richness,
and
(2) no  general  clustering  function  can  simultaneously  satisfy  inner-consistency,  scale-
invariance, and richness. 
They claim also that  
$k$
-means-ideal has the properties of outer-consistency and locality\footnote{
A clustering function clustering into $k$ clusters has the locality property, 
if whenever a set $S$ for a given $k$ is clustered by it into the partition   $\Gamma$,
and we take a subset $\Gamma'\subset \Gamma$ with $|\Gamma'|=k'<k$,
then clustering of $\cup_{C\in \Gamma'}$ into $k'$ clusters will yield exactly $\Gamma'$. 
}. None of these properties is claimed to be satisfied by $k$-means-random nor by  a $k$-means with furthest element initialization.   
Furthermore,   $k$-richness (in probabilistic sense) is not matched by $k$-means-random algorithm.  
In this paper we point at the fact that in Euclidean space 
even inner-consistency alone (see our Theorem  \ref{thm:noinnerConsistency}) / outer-consistency alone (see our Theorem  \ref{thm:noouterConsistency}) are self-contradictory.
Also the consistency alone poses problem (see our Theorem  \ref{thm:noConsistentSingleClusterGamma})
 So they do so   for $k$-means-ideal.  
But on the other hand we show  $k$-richness (in probabilistic sense) \emph{is   matched} by $k$-means-random algorithm (see our Theorem \ref{th:krichnesskmeans}).  
  
Still another relaxation of the Kleinberg's consistency is called   Refinement Consistency. It 
is a modification of the consistency axiom by 
\emph{replacing the requirement that 
$f(d)=f(d')$ with the requirement that 
 one of $f(d), f(d')$ is a refinement of the other}.
A partition $\Gamma'$ is a refinement of a partition $\Gamma$ if for each cluster $c'\in\Gamma'$ there exists a cluster $c\in\Gamma$ such that $c'\subseteq c$.  
Obviously the replacement of the consistency requirement with refinement consistency breaks the impossibility  proof of Kleinberg's axiom system. 
But there is a practical concern:
In general, refinement consistency means that by the $\Gamma$ transformation and scaling you may transform any clustering in any other. The usefulness of such an axiom is hence questionable.
In this paper (Section \ref{sec:kmeansabsoluteballclustering}) we show that under some circumstances unidirectional refinement consistency may be achieved, which makes much more sense. 

Zadeh   Ben-David \cite{Zadeh:2009} propose  instead the 
\emph{order-consistency} so that some versions of single-linkage algorithm 
can be classified as "clustering algorithm". 
\emph{
For any two distance
functions $d$ and $d'$, if the orderings of edge lengths are the same 
then $f(d)=f(d')$. }
$k$-means is not order-consistent.

One could also   relax  Scale-Invariance instead to e.g.  \emph{Robustness}, that is,
\emph{ "Small changes in distance function  $d$ should result in
small changes of partition $f(d)$"}.
The basic problem here is that partitions are discrete and the term "small" is hard to define reasonably.  
Small changes in distances may result in major changes of partitions obtained from $k$-means algorithm.

Let us also mention here some works like that by Dunn \cite{CLU:Dun74} or 
Ackerman and Dasgupta \cite{Ackerman:2014nips}
that seemingly have nothing to do with Kleinberg's axioms, but this is only a superficial impression.
Papers discussing the issue of "well separated clusters" or "nicely separated", or "perfectly separated  " point in fact at
the weakness of the non-refutability axiom, because  it is apparent that we do not want to get any partition but rather one that is meaningful. 

Ackerman and Dasgupta  \cite{Ackerman:2014nips} handle incremental clustering algorithms. They introduce an incremental version of $k$-means algorithm. 
The clusters are "nicely separated", as defined by   \cite{Ackerman:2014nips}, if a distance between an element and any other element of the same cluster is lower than the distance from this element to an element outside of the cluster. The authors demonstrate that no incremental algorithm of space complexity linear in $k$ can (routinely) discover the clusters that are nicely separated. This is contrary to single-link algorithm which can identify a set of $2^{k-1}$ candidate elements among which $k$ are from different clusters, if   a nice clustering is unique. A nice clustering can be only detected in this sense (a set of candidates) by an incremental algorithm with memory linear in $2^{k-1}$. But it cannot be detected by the incremental $k$-means even with such a large memory. However, when looking at the issue with randomly generated sequence of data, a  memory linear in $k$ suffices for incremental $k$-means with some probability.   
Then they introduce the "perfect clustering" with the property that the smallest distance between elements of distinct clusters is larger than  the distance between any two elements of the same cluster. 
They demonstrate that there exists an incremental algorithm discovering the "perfect clustering" that is linear in $k$ with respect to space. But the incremental $k$-means fails to do so.
We will discuss this issue in section \ref{sec:kmeansperfectballclustering}.

Ackerman and Ben-David
\cite{Ben-David:2009}
propose another direction of resolving the problem of Kleinberg's axiomatisation impossibility.
Instead of  axiomatising  the clustering function, one should rather create axioms for cluster quality function. 
  
A number of further characterizations of clustering functions has been proposed to overcome Kleinberg axiom problems, 
e.g. \cite{Ackerman:2010} for linkage algorithms, 
\cite{Carlsson:2010} for hierarchical algorithms, 
\cite{Carlsson:2008} for multiscale clustering.

Note that beside Kleinberg's axioms there exist other "impossible" characterizations of clustering functions. 
Meila \cite{Meila:2005} demonstrates that 
one 
can't compare  partitions in a manner  that agrees with  the lattice of partitions, is convexly additive  and  bounded.

General tendency of researchers wanting to overcome Kleinberg's contradiction was to weaken one or more axioms of Kleinberg. 
While in this way the contradiction was removed, the removal relied on 
weakening the reasoning capabilities so that no strong conclusions can be reached.
In this research we try the different way - one of strengthening the assumptions so that for example a proof of $k$-means consistency becomes possible. 

But before we present a consistent set of \changed{axioms}{algorithm properties} and show its validity for $k$-means algorithm, we will investigate counter-intuitiveness of Kleinberg's formalization of his axioms. 

Let us still mention briefly, that other characteristics of $k$-means algorithms were studied in the past, 
see e.g. papers by Ackerman et al. \cite{Ackerman:2012,Ackerman:2013}. 
\cite{Ackerman:2013} deals with the susceptivity of among others the $k$-means algorithm to hostile addition of new points to the data set. 
It turns out that $k$-means is stable under such disturbances given that the clusters are well balanced (cluster sizes do not differ very much) and there are sufficient gaps between the clusters. 
\cite{Ackerman:2012} demonstrates that one can put any two data points into different clusters if one applies  weighting functions to data points. 
Both of these papers, though not explicitly addressing the $k$-richness, demonstrate problems resulting from this axiom. 
\cite{Ackerman:2013} implies that too small clusters may be disintegrated by hostile new points so that for practical purposes one shall be only interested in larger clusters. 
\cite{Ackerman:2012} allows to conclude that poor estimates of densities for sparse clusters may lead to erroneous drawing of cluster boundaries.

\subsection{Counter-intuitiveness of Scale-invariance and Consistency Axioms} \label{sec:badscalabilityconsistency}

 Kleinberg in his paper proved so-called  anti-chain theorem that implies that 
by scaling and contraction ($\Gamma$-transform) one can transform any 
clustering into any other.\footnote{ It is why Kleinberg proposed in his paper the "refinement consistency".} 
This fact combined with the richness axiom leads directly to contradiction in the three axioms.  
 
First of all let us state that
\begin{theorem}{} \label{thm:KleinbergImpossibilityStronger}
For
$n=2$, for data in $\mathbb{R}^m$, 
there is no   clustering  function 
 $f$
that satisfies
Scale-Invariance
and
Richness.\footnote{Contrary to Kleinberg's intuitions, scale-invariance plus richness alone lead to a contradiction. }
\end{theorem}
\begin{proof} 
Any set $S=\{e_1,e_2\}$ consisting of only two elements 
has potentially two partitions:
$\Gamma_1=\{\{e_1\},\{e_2\}\}$ ("singleton partition") 
and 
$\Gamma_2=\{ e_1 , e_2 \}$ ("no-split-partition"). 
Let $f(d_1,S)=\Gamma_1$ and
$f(d_2,S)=\Gamma_2$.
Then 
$f(\frac{d_2(e_1,e_2)}{d_1(e_1,e_2)} d_1,S)=\Gamma_1$
according to scale-invariance
but by definition 
$f(\frac{d_2(e_1,e_2)}{d_1(e_1,e_2)} d_1=d_2,S)=\Gamma_2$, so we have an obvious contradiction. 
\end{proof}

Note that this theorem strengthens the result of Kleinberg stated in 
Theorem \ref{thm:KleinbergImpossibility} - two kleinberg's properties/axioms already (and not three) lead to a contradition.

As we will demonstrate later, a function matching richness axiom of Kleinberg
does not necessarily exhibit richness, if distances will be  confined 
to $\mathbb{R}^m$. But in case of the above theorem it does not matter because we talk about 2 data points only and hence automatically the validity as distance in $\mathbb{R}^m$ is granted.

It is further easy to show (also based on Kleinberg's anti-chain theorem) that\footnote{This shows that richness is not needed at all to get a contradiction from consistency and scale-invariance}
\begin{theorem} {} \label{thm:KleinbergImpossibilityEvenStronger}
For any $n>2$, for data in $\mathbb{R}^m$, 
no function $f$ can produce  no-split partition $\Gamma_1$
under some distance function $d_1$ 
and any other  partition $\Gamma_2$
under some distance function $d_2$ 
if it satisfies both consistency and scale-invariance properties.
\end{theorem}
By the way this  variant of Kleinberg's anti-chain theorem 
is a reason why he proposed to weaken richness requirement to "near-richness", omitting the "all-in-one" partition.%
\footnote{In fact the Kleinberg's anti-chain theorem implies that also a partition putting each element into a separate cluster should be excluded 
from "near-richness"} 

\begin{proof}
To show this let 
$mind=\min_{e_1,e_2\in S} d_1$ 
and
$maxd=\max_{e_1,e_2\in S} d_2$.
It is easy to see that $d_2$ is a $\Gamma$-transform of 
$\frac{maxd}{mind}d_1$ for partition $\Gamma_1$. 
Therefore as $f(d_1,S)=\Gamma_1$,
 because  of scale-invariance $f(\frac{maxd}{mind}d_1,S)=f(d_1,S)=\Gamma_1$,
hence by consistency $f(d_2,S)=f(\frac{maxd}{mind}d_1,S)=\Gamma_1$. 
This contradicts the assumption that $f(d_2,S)=\Gamma_2$. 
\end{proof}

As mentioned, \cite{vanLaarhoven:2014}
pointed at the fact that such a construction would not be possible in the realm of graph clustering. 
We shall ask then: what about $\mathbb{R}^m$? 
We provided the above proof to show that 
  forcing data points into the Euclidean space does not invalidate the construction because the scaling operation keeps the points in the original Euclidean space.  

So there is surely a need to redefine the richness property 
into a %"not-all-in-one-
"near-richness".\footnote{A similar reasoning is possible for singleton partition, but we choose this way.} 

But "near-richness"  is again not enough to resolve all contradictions (as by the way is visible from the Kleinberg's anti-chain theorem \cite{Kleinberg:2002}).
\begin{theorem} {}  
For any $n>1 $, for data in $\mathbb{R}^m$ ($m>2$), 
no function $f$ can produce 
a partition   $\Gamma_1$ consisting of two sets of elements
$\Gamma_1=\{\{1,2,\dots,n\},\{n+1,n+2\}\}$
under some distance function $d_1$ 
and any other  partition $\Gamma_2$ consisting of three sets of elements 
$\Gamma_2=\{\{1,2,\dots,n\},\{n+1\},\{n+2\}\}$
under some distance function $d_2$ 
if it satisfies both consistency and scale-invariance properties.
\end{theorem}

\begin{proof}
For $n\ge 1$  take a set of $n+2$ elements.
The richness property implies that under two distinct distance functions $d_1,d_2$  
the clustering function $f$ may 
form two partitions: $\Gamma_1,\Gamma_2$, resp., as defined in the theorem. 
By invariance property, 
we can derive from $d_2$ the distance function $d_4$ such that 
no distance between the elements under $d_4$ is lower than the biggest distance under $d_1$.
By invariance property, 
we can derive from $d_1$ the distance function $d_3$ such that 
 the  distance between  elements $n+1,n+2$ is bigger than 
  under $d_4$.
We have then $f(\{1,\dots,n+2\};d_4)=\Gamma_2,f(\{1,\dots,n+2\};d_3)=\Gamma_1 $. 
Now let us apply the consistency axiom. 
From $d_4$ we derive the distance function 
$d_6$ such that 
for elements $1,\dots,n$ $d_1$ and $d_6$ are identical,
the distance between $n+1,n+2$ is same as in $d_4$ and the distances between any element of $1,\dots,n$ and any of $n+1,n+2$ 
is some $l$ that is bigger than any distances between any elements under $d_1,\dots,d_4$. 
From $d_3$ we derive the distance function 
$d_5$ such that 
for elements $1,\dots,n$ $d_1$ and $d_5$ are identical,
the distance between $n+1,n+2$ is same as in $d_4$ and the distances between any element of $1,\dots,n$ and any of $n+1,n+2$ 
is same $l$ as above.  
We have then $f(\{1,\dots,n+2\};d_6)=\Gamma_2,f(\{1,\dots,n+2\};d_5)=\Gamma_1 $. 
But then we have a contradiction because by construction $d_5$ and $d_6$ are identical. 

In this proof, however, assumptions are made that may possibly be not  
correct if we require the distances to be  distances in Euclidean space. 
So not for any configuration of $n$ points a $n+1$-st point may be found to be equidistant to all the other ones. 
And even if it is so, it is not guaranteed that a second distinct $n+2$-nd point exists with the same property. 
Hence in the above construction of the proof, an initial step 
is needed, matching using consistency property, that will pose the points $1,\dots, n$ onto a sphere both for $\Gamma_1$ and $\Gamma_2$, 
and points $n+1,n+2$ on a  line orthogonal to the subspace containing $1,\dots,n$ and passing through the origin of the sphere.  

In the end, of course, the contradiction is still valid in Euclidean space, but this exercise shows that proofs of Kleinberg need to be rewritten if we deal with Euclidean spaces. 
But note that if we restrict ourselves to $\mathbb{R}^2$, posing the points onto a sphere does not work anymore.  Points $n+1$ ad $n+2$ will become identical.  
\end{proof}

So, 
there is still an open question,  
whether or not we can have a clustering function 
matching Kleinberg's axioms, that is still not contradictory.
We will at this issue below. 
in $\mathbb{R}^m$.  

%-----------------------------------------------
\section{To embed or not to embed}\label{sec:kernel}

%Few remarks on kernel-trick k-means 

Kleinberg's permission of non-embeddability  axiom (\refax{ax:nonEmbedding})  assumes that distances can be any non-negative symmetric functions over the set of pairs of  objects.

$k$-means normally operates in an Euclidean space, but by using so-called kernel-trick\footnote{
We will not dive deeper in this paper into the discussion of properties of kernel $k$-means.
Let us only make the remark that 
  kernel $k$-means, given that there exists an embedding in $\mathbb{R}^m$), is in fact $k$-means in the feature space. So all the findings related to $k$-means would apply also in the feature space. The weighted  version of kernel $k$-means may be considered a bit tricky, but it can be "approximated" by multiplying the unweighted points, under the restriction that all multiplied points will go into the same cluster, but this doss not seem to invalidate any findings. 
A separate question of course is whether or not we can invert the kernel function (if it is given explicitly) in order to find points transformed by e.g. centric consistency transform in the feature space. 
} one can operate on the objects as if they were embedded in a (highly dimensional) space  without actually finding the  embedding (just working on a kernel matrix derived from distances). And one can get a clustering in that space optimizing the $Q$ function. 

\Bem{ 
http://www.kyb.mpg.de/fileadmin/user_upload/files/publications/attachments/scholkopf00kernel_3781%5b0%5d.pdf
}%Bem 

It is well known that if there exists an embedding of a set of $n$ points in an Euclidean space, then we do not need to consider more than $n-1$ dimensions. 

But it is well known that not for each distance function in the sense of Kleinberg's definition there exists an embedding. Just consider the points in the table \ref{tab:antikernel}.

\begin{table}
\centering
\caption{Distances between points $A,B,C,D,E,F$   }
\label{tab:antikernel}
\begin{tabular}{|r|l|l|l|l|l|l|} 
\hline 
  & A&B&C&D&E&F\\ 
\hline 
 A      & 0     & 10    & 2.236         & 20    & 22.361        & 20.125 \\ 
\hline 
 B      & 10    & 0     & 6.708         & 22.361        & 20    & 21.095 \\ 
\hline 
 C      & 2.236         & 6.708         & 0     & 20.125        & 21.095        & 20 \\ 
\hline 
 D      & 20    & 22.361        & 20.125        & 0     & 10    & 2.236 \\ 
\hline 
 E      & 22.361        & 20    & 21.095        & 10    & 0     & 6.708 \\ 
\hline 
 F      & 20.125        & 21.095        & 20    & 2.236         & 6.708         & 0 \\ 
\hline 
\end{tabular}
\end{table}

It is visible at the first glance that not even the triangle inequality  holds in this data set (just look at points $A,B,C$ alone).
So no embedding in Euclidean space is possible. 

But what if we still apply the kernel trick?
One can easily find an embedding in a three-dimensional space if one allows for "imaginary" coordinates (allowing for square-rooting negative eigenvalues).
See table \ref{tab:embedantikernel}. 
The distances are kept if we rigidly 
use the distance formula 
$$d(P,T)= \sqrt{(x_{P,1}- x_{T,1})^2
+(x_{P,2}- x_{T,3})^2
+(x_{P,2}- x_{T,3})^2}
$$

\begin{table}
\centering
\caption{"Complex embedding" of points $A,B,C,D,E,F$   }
\label{tab:embedantikernel}
\begin{tabular}{|r|l|l|l|} 
\hline 
 point & $x_1$&$x_2$&$x_3$\\ 
   \hline 
A       & 5+0i  & 10+0i         & 0+1i \\ 
   \hline 
B       & -5+0i         & 10+0i         & 0+1i \\ 
   \hline 
C       & 2+0i  & 10+0i         & 0-1i \\ 
   \hline 
D       & 5+0i  & -10+0i        & 0+1i \\ 
   \hline 
E       & -5+0i         & -10+0i        & 0+1i \\ 
   \hline 
F       & 2+0i  & -10+0i        & 0-1i \\ 
\hline 
\end{tabular}
\end{table}

A quick look into the table \ref{tab:antikernel} would suggest 
that points $A,B,C$ form one cluster, and $D,E,F$ form another. 

However, if we take the centers of the respective clusters  
$\mu_1=(0.667+0.000i, 10.000+0.000i,  0.000+0.333i)$
and 
$\mu_2=(0.667+0.000i, -10.000+0.000i,  0.000+0.333i)$, then the $Q$ function for such 2-means 
amounts to 100. 
But if we take the points  
$S_1=( 0+ 0.00i, 0+ 0.00i, 0-10.18i)$,
$S_2=(0+0.000i, 0+0.000i, 0+9.198i)$ as cluster centers,
then clusters $\{A,B,D,E\}$ and $\{C,F,\}$ are formed around them with $Q$ function value equal $6\cdot 10^{-6}$. 

So the Kleinberg's non-embeddability axiom 
is not suitable for clustering algorithms for which position of other points in space needs to be anticipated.
Under the assumption of Euclidean embedding this problem is clearly solved. 

From now on we will always assume that, if not stated otherwise, we constrain the Kleinberg's consistency transform to the cases embeddable into a fixed dimensional Euclidean space.

Note that with this result also the Kleinberg's non-refutability axiom is indirectly questioned.%
\footnote{It does not mean that there do not exist versions of $k$-means for distances other than Euclidean distance.
What we wanted to demonstrate here is that the notion of embedding is needed if we want to look at $k$-means from Kleinberg's axioms perspective.

}

%-----------------------------------------------
 
\section{Kleinberg's axioms and  $k$-means -- Conformance and Violations   }\label{sec:kmeansContraKleinberg}

Let us briefly discuss here the relationship of $k$-means algorithm to the already mentioned axiomatic systems, keeping in mind that we apply it in $\mathbb{R}^m$ Euclidean space.

Scale-invariance is fulfilled because $k$-means
qualifies objects into clusters based on   relative distances to cluster centers and not their absolute values as may be easily seen from equation (\ref{eq:Q::kmeans}).%
\footnote{However, this quality function fails on the axiom of  Function Scale Invariance, proposed in \cite{Ben-David:2009}. }

On the other hand richness, a property denial of which   has nothing to do with distances, hence with embedding in an Euclidean space, as already known from mentioned publications, e.g. \cite{Zadeh:2009}, is obviously violated because $k$-means returns only partitions into $k$ clusters.

But what about its relaxation that is $k$-richness. 
Let us briefly show here that
\begin{theorem}\label{th:krichnesskmeans}
 $k$-means algorithm  is $k$-rich
\end{theorem}
\begin{proof} We proceed 
  by constructing a data set for each required partition.  
Let us consider $n$ data points arranged on a straight line and we want to split them into $k$ clusters fitting a concrete partition $\Gamma_0$. 
For this purpose arrange the clusters on the line (left to right) in non-increasing order of their cardinality. Each cluster shall occupy (uniformly) a unit length.
The space between the clusters (distance between closest elements of  $i$th and $(i+1)$st cluster) should be set as follows: For $i=1,\dots,k-1$ let $dce(j,i)$ denote the distance between the most extreme data points of clusters $j$ and $i$, $cardc(j,i)$ shall denote the combined cardinality of clusters $j,j+1,\dots,i$. The distance between closest elements of clusters $i$ and $i+1$ shall be then set to $2*dce(1,i)\frac{cardc(1,i)+cardc(i+1,i+1)}{cardc(i+1,i+1)}$. 

In this case application of $k$-means algorithm ($k$-means-ideal, $k$-means-random, $k$-means++)  will lead  the desired partition. 
The reasons are as follows:

In case of $k$-means-ideal, let $A$ be the most right   cluster of a partition $\Gamma$ different from $\Gamma_0$, containing the "space between clusters".
The definition of this distance is chosen in such a way that if we split $A$  into two parts along this "space between clusters" and attach the left and the right part to the neighboring clusters, and splitting any cluster if the number of clusters falls below $k$ in this way, then the resulting new partition will be more optimal.
Hence the optimal $k$-means-ideal clustering will not contain any "spaces between clusters" within the clusters and be identical with the intended $\Gamma_0$. 
This implies we can construct any partition in this way.

In case of $k$-means-random, 
if each of the clusters of $\Gamma_0$ is seeded, the spaces between clusters of $\Gamma_0$ are so large, 
that the clustering resulting from such a seeding is identical with $\Gamma_0$ and upon subsequent steps the partition will not change any more. 
So consider now the case that   after the random initialization (or at any later step) we get a partition $\Gamma$ with cluster centers $\mu_1,\dots \mu_k$ such that there be a cluster $C$ of $\Gamma_0$ that has not been seeded (does not contain a $\mu_i$ in its range). 
No cluster of $\Gamma$  with center  to the right of $C$ would nonetheless contain any data element from $C$. 
Consider therefore  only clusters of $\Gamma_0$ to the left of $C$ and let $\mu_r$ be the cluster center most to the right in this set.
 The cluster  $C_r$ formed from elements closest to $\mu_r$   will contain $C$. 
Therefore during the cluster center update step of $k$-means-random $\mu_r$ will move to the right 
  to a position, from which only $C$ will be the set of points closest to $\mu_r$.
Therefore  after 3 steps $\mu_r$ will become the center of $C$.
Later on the same process will happen with other not seeded clusters of $\Gamma_0$ to the left of it till we get the partition $\Gamma_0$. 

As $k$-means++ behaves similarly to $k$-means-random after initial seeding, the same effect will be reached. 
\footnote{
Let us formulate the argument more precisely. 
As mentioned earlier, following \cite{Ackerman:2010NIPS}, 
for probabilistic algorithms we will talk about probabilistic $k$-richness, that is one obtainable with some probability, independent of the actual clustering that is intended to be obtained. The probability can further be increased if one wishes to.

 In the above scheme we see that whenever during the initialization at each place more seeds are there than clusters, then they will spread to the right, given there are clusters without seeds there, ensuring that each cluster gets its cluster center. 
Note that if there are clusters lacking seeds to the left, there is no way to move seeds there. 
Hence, as the cluster s are sorted   in decreasing size order from the left to right, the probability,
that we have a seeding upon which by moving cluster centers to the right we can assign each cluster a cluster center   amounts to at least  $k!/k^k$.
This is computed as follows: The favorable seeding occurs, if the first seed is in the first cluster, and the $i$th seed in a cluster 1 or 2 or ... or $i$ from the left. 
As the clusters are sorted non-increasingly, the probability of hitting the first cluster is at least $\frac1k$, that of first or second $\frac2k$, that of first, or second, or,\dots,or $i$th is $\frac i k$. This results in the aforementioned estimation.

Note that  the estimated probability is independent of the sample size and the actual distribution of sizes of clusters. It depends on $k$ only.

Furthermore, the targeted clustering is the absolute minimum of the $k$-means-ideal, hence we can run $k$-means-random  multiple time in order to achieve the desired probability of $k$-richness. 
E.g. if we need 95\% certainty, we need to rerun $k$-means-random $r$ times with $r$ such that 
$1-(1-k!/k^k)^r \ge 95\%$. 

The issue with $k$-means++ is a bit more complex due to the way how probabilities of seeding are computed. 
In fact, we do not rely on the $k$-means iterating process, but have rather to ensure that each cluster gets a seed during the seeding phase. 

When the first seed is distributed, like in $k$-means, we have the assurance that an unhit cluster will be hit. 
The probability that a cluster is hit during the seeding step after the first one  
is proportional to the sum of squared distances of cluster elements to the closest seed assigned earlier.
Consider the $i$th cluster (from the left) that was not hit so far. 
The closest hit cluster to the left can lie  at least a distance 
$$\alpha*dce(1,i-1)\frac{cardc(1,i-1)+cardc(i,i)}{cardc(i,i)}$$
and that to the right at
$$\alpha*dce(1,i)\frac{cardc(1,i)+cardc(i+1,i+1)}{cardc(i+1,i+1)}$$
(note that the first cluster has no left neighbor, and the $k$th - no right neighbor). $\alpha=2$. 
So the contribution of the $i$th cluster to the sum of squares estimation for hitting probability in a current state amounts to at least the smaller number of the following two:
$$cardc(i,i)\left(\alpha*dce(1,i-1)\frac{cardc(1,i-1)+cardc(i,i)}{cardc(i,i)}\right)^2$$
$$cardc(i,i)\left(\alpha*dce(1,i)\frac{cardc(1,i)+cardc(i+1,i+1)}{cardc(i+1,i+1)}\right)^2\ge
$$ $$\ge 
cardc(i,i)\left(\alpha*dce(1,i)\frac{cardc(1,i)+cardc(i+1,i+1)}{cardc(i,i)}\right)^2  
$$
Obviously the first expression is the smaller one so we will consider it only. 
$$cardc(i,i)\left(\alpha*dce(1,i-1)\frac{cardc(1,i-1)+cardc(i,i)}{cardc(i,i)}\right)^2
=
  \alpha^2*dce(1,i-1)^2\frac{cardc(1,i)^2}{cardc(i,i)}
$$
Note  that due to the non-increasing order of cluster sizes, 
$cardc(1,i)\ge i\cdot cardc(i,i)$, 
and $cardc(1,i)\ge i \frac n k$.
Therefore 
$$
  \alpha^2*dce(1,i-1)^2\frac{cardc(1,i)^2}{cardc(i,i)}
\ge
  \alpha^2*dce(1,i-1)^2 i \frac n k 
$$
Furthermore, 
$dce(1,1) =1$, and 
$dce(1,i) =dce(1,i-1)+1+
\alpha*dce(1,i-1)\frac{cardc(1,i-1)+cardc(i,i)}{cardc(i,i)}
\ge 
dce(1,i-1)+1+
\alpha*dce(1,i-1)\cdot i  
=1+dce(1,i-1)\cdot(1+i\alpha)
\ge dce(1,i-1)\cdot(1+i\alpha)
 $
Hence $$dce(1,i) \ge (1+2\alpha)^{i-1}$$
So
$$
  \alpha^2*dce(1,i-1)^2 i \frac n k 
\ge \alpha^2 (1+2\alpha)^{2(i-1)}  i \frac n k
$$
 
After $s$ seeds were distributed the sum of squared distances to the closest seed for hit clusters   amounts to   at most the combined cardinality of the clusters with seeds times $1$ so this  does not exceed $n$.   

Therefore the probability of hitting an unhit cluster after $s$ seeds were already distributed and hit different clusters is not bigger than  
$$\frac
{ \alpha^2 (1+2\alpha)^{2}    \frac n k
+\sum_{i=2}^{k-s} \alpha^2 (1+2\alpha)^{2(i-1)}  i \frac n k
}
{
n+
  \alpha^2 (1+2\alpha)^{2}    \frac n k
+\sum_{i=2}^{k-s} \alpha^2 (1+2\alpha)^{2(i-1)}  i \frac n k
}
=
\frac
{ \alpha^2 (1+2\alpha)^{2}     
+\sum_{i=2}^{k-s} \alpha^2 (1+2\alpha)^{2(i-1)}  i  
}
{
k+
  \alpha^2 (1+2\alpha)^{2}     
+\sum_{i=2}^{k-s} \alpha^2 (1+2\alpha)^{2(i-1)}  i  
}
$$

So the probability that during the seeding all clusters are hit by a seed amounts to at least. 
$$
\prod_{s=1}^{k-1}
\frac
{ \alpha^2 (1+2\alpha)^{2}     
+\sum_{i=2}^{k-s} \alpha^2 (1+2\alpha)^{2(i-1)}  i  
}
{
k+
  \alpha^2 (1+2\alpha)^{2}     
+\sum_{i=2}^{k-s} \alpha^2 (1+2\alpha)^{2(i-1)}  i  
}
$$

 In order to increase the success probability we can now repeat the seed independently sufficiently many times, or we can increase the distances by 
letting $\alpha$ be (much) greater than 2.

 }%footnote

\end{proof}

Let us stress here that there exist attempts to upgrade $k$-means algorithm  to choose the proper $k$. 
The portion of variance explained 
by the clustering is used as quality criterion\footnote{Such a quality function would satisfy axiom of  Function Scale Invariance, proposed in \cite{Ben-David:2009}}.
 It is well known that increase of $k$ increases the value of this criterion.
The optimal $k$ is deemed to be one when this increase stops to be "significant". 
The above construction could be extended to cover a range of $k$ values to choose from.
However, the full richness is not achievable because a split into two clusters will be better than keeping a single cluster, and the maximum is attained for this criterion if $k=n$. So either the clustering will be trivial or quite a large number of partitions will be excluded. 
However, even $k$-richness offers a large number of partitions to choose from.

Kleinberg himself proved via a bit artificial example (with unbalanced samples and an awkward distance function)
that $k$-means algorithm with $k$=2 is not consistent.
Kleinberg's counter-example would require an embedding in a very high dimensional space, non-typical for $k$-means applications. 
Also  $k$-means  tends to produce rather balanced clusters, 
so Kleinberg's example could be deemed to be eccentric.

Let us illustrate by a more realistic  example (balanced, in Euclidean space) that this is a real problem.
Let $A,B,C,D,E,F$ be points in three-dimensional space with coordinates:
$A(1,0,0)$, 
$B(33,32,0)$,
$C(33,-32,0)$,
$D(-1,0,0)$, 
$E(-33,0,-32)$,
$F(-33,0,32)$. 
Let $S_{AB}$,  $S_{AC}$, $S_{DE}$, $S_{DF}$ be sets of say 1000 points randomly uniformly distributed 
over line segments (except for endpoints) $AB, AC, DE, EF$ resp. 
Let $X=S_{AB}\cup  S_{AC}\cup  S_{DE}\cup  S_{EF}$.
$k$-means with $k=2$ applied to $X$ yields a partition
  $\{S_{AB}\cup  S_{AC},  S_{DE}\cup  S_{DF}\}$.
But let us perform a $\Gamma$ transformation consisting in 
rotating line segments $AB,BC$ around the point $A$ in the plane spread by the first two coordinates 
towards the first coordinate axis so that the angle between this axis and $AB'$ and $AC'$ is say one degree.
Now the $k$-means with $k=2$ yields a different partition, splitting line segments $AB'$ and $AC'$.% 
\footnote{In a test run with 100 restarts,
in the first case we got clusters of equal sizes,
with cluster centers at (17,0,0) and (-17,0,0), (between\_SS / total\_SS =  40 \%)
whereas after rotation 
we got clusters of sizes 1800, 2200
with centers at (26,0,0), (-15,0,0) (between\_SS / total\_SS =  59 \%)
}

With this example not only \emph{consistency violation} is shown, but also \emph{refinement-consistency violation}.

\section{Problems with Consistency in Euclidean Space}\label{sec:consistecycontrascaling} 
How does it happen that seemingly intuitive axioms lead to such a contradiction. 
We need to look more carefully at the consistency axiom in conjunction with scale-invariance. 
$\Gamma$-transform  does not do what Kleinberg claimed it should that is 
describing a situation when 
  moving elements from distinct clusters apart and elements within a cluster closer to one another.%
\footnote{
Recall that the intuition behind clustering is to partition the data points in such a way
that members of the same cluster are "close" to one another, that is their distance is low,
and members of two different   clusters are "distant" from one another, that is their distance is high. So it is intuitively obvious that   moving elements from distinct clusters apart and elements within a cluster closer to one another should make a partition "look better".    
}

We shall now demonstrate that 
application of scaling invariance axiom leads to violation of the  consistency axiom  of Kleinberg. 
More precisely: 
\begin{theorem}
For a clustering algorithm $f$, conforming to consistency and scaling invariance axioms, 
if distance $d_2$ is derived from the distance $d_1$ by consistency transformation, 
and $d_3$ is obtained from $d_2$ via scaling, then 
the existence of a 
$d_3$ cannot always be obtained from $d_1$ via consistency axiom transformation. 
\end{theorem}
\begin{proof} 
We prove the Theorem by finding a suitable example.  
let $S$ consist of four elements $e_1,e_2,e_3,e_4$ 
and let a clustering function partition it into 
$\{e_1\},\{e_2,e_3\},\{e_4\}$ under some distance function $d_1$. 
One can easily construct a distance function $d_2$ being a $\Gamma$-transform of $d_1$
such that $d_2(e_2,e_3)=d_1(e_2,e_3)$ 
and $d_2(e_1,e_2)+ d_2(e_2,e_3)=d_2(e_1,e_3)$
and $d_2(e_2,e_3)+ d_2(e_3,e_4)=d_2(e_2,e_4)$
and $d_2(e_1,e_2)+d_2(e_2,e_3)+ d_2(e_3,e_4)=d_2(e_1,e_4)$
which implies that these points under $d_2$ can be embedded 
in  the space $\mathbb{R}$ that is the straight line. 
Without restricting the generality (the qualitative illustration) assume
that the coordinates of these points in this space are
located at points $0, 0.4, 0.6, 1$ resp.
Now assume we want to perform $\Gamma$-transformation of Kleinberg (obtaining the distance function $d_3$) in such a manner that the data points remain in  $\mathbb{R}$ and move  elements of the second set i.e. $\{e_2,e_3\}$ ($d_2(e_2,e_3)=0.2$) closer to one another so that   $e_2=(0.5), e_3=(0.6)$ ($d_3(e_2,e_3)=0.1$). $e_1$ may then stay where it is but $e_4$ has to be shifted at least to $(1.1)$ (under $d_3$ the clustering function shall yield same clustering). 
Now apply rescaling into  the original interval that is multiply the coordinates (and hence the distances, yielding $d_4$) by $1/1.1$.
$e_1$ stays at (0), $e_2=(\frac{5}{11}), e_3=\frac{6}{11}, e_4=(1)$. $e_3$ is now closer to $e_1$ than before.
We could have made the things still more drastic 
by transforming $d_2$ to $d_3'$  in such a way that 
 instead of $e_4$ going to $(1.1)$, as under $d_3$, we    set it at $(2)$. 
In this case the rescaling would result in 
$e_1=(0), e_2=(0.25), e_3=(0.3), e_4=(1)$ (with the respective distances $d_4'$) which means a drastic relocation of the second cluster towards the first - the distance between clusters decreases instead of increasing as claimed by Kleinberg. 
This is a big surprise. 
The $\Gamma$ transform should have moved elements of a cluster closer together and further apart those from distinct clusters and rescaling should not disturb the proportions. It turned out to be the other way. 
This contradicts the consistency assumption. 
\end{proof}

So something is wrong either with the idea of 
 scaling or of   $\Gamma$-transformation. 
We shall be reluctant to blame the scaling, except for the 
practical case when scaling down leads to indiscernibility between points with respect to measurement errors. 

Note that we do not observe such a clash between invariance and richness. 
If a set of distance functions demonstrates the richness of a clustering function conforming to richness and scaling, 
then after scaling all these distance functions demonstrate the richness of the same clustering function again. 
Scaling does not impair the richness.

\section{Counter-intuitiveness of Consistency Axiom Alone}  \label{sec:badconsistency}

So we will consider counter-intuitiveness of consistency axiom. 
To illustrate it, recall first the fact that a large portion of known clustering algorithms uses data points embedded   in an $m$ dimensional feature space, usually $\mathbb{R}^m$ and the distance is the Euclidean distance therein.  
Now imagine that we want to perform a $\Gamma$-transform  on a single cluster of a partition that is the $\Gamma$-transform 
shall provide distances compatible with the situation that only elements of a single cluster change position in the embedding space. 
\begin{theorem} {} \label{thm:noConsistentSingleClusterGamma}
Under the above-mentioned circumstances
it is impossible to perform $\Gamma$-transform reducing distances within a single cluster.
\end{theorem}
\begin{proof}
Assume the cluster is an "internal" one that is for a point $e$ in this cluster any hyperplane containing it has points from some other clusters on each side. 
Furthermore assume that other clusters contain together more than $m$ data points, which should not be an untypical case. 
Here the problem starts. The position of $e$ is determined by the distances from the elements of the other clusters in such a way that the increase of distance from one of them would necessarily decrease the distance to some other (except for  strange configurations), contrary to consistency requirement.  
Hence the claim 
\end{proof}

So the $\Gamma$-transform enforces either adding a new dimension and moving the affected single cluster along it (which does not seem to be quite natural) or to change positions of elements in at least two clusters within the embedding space.
Therefore vast majority of such algorithms does not meet not only the consistency but also inner consistency requirement.

\begin{theorem} {} \label{thm:noinnerConsistency}
No algorithm operating in a fix-dimensional space under Euclidean distance 
can conform to inner-consistency axiom.%  
\footnote{This impossibility does not mean that there is an inner contradiction 
when executing the inner-consistency transform. 
Rather it means that considering inner-consistency is pointless because inner-consistency transform is in general impossible.
}
\end{theorem}

\begin{figure}
\centering
\includegraphics[width=0.8\textwidth]{\figaddr{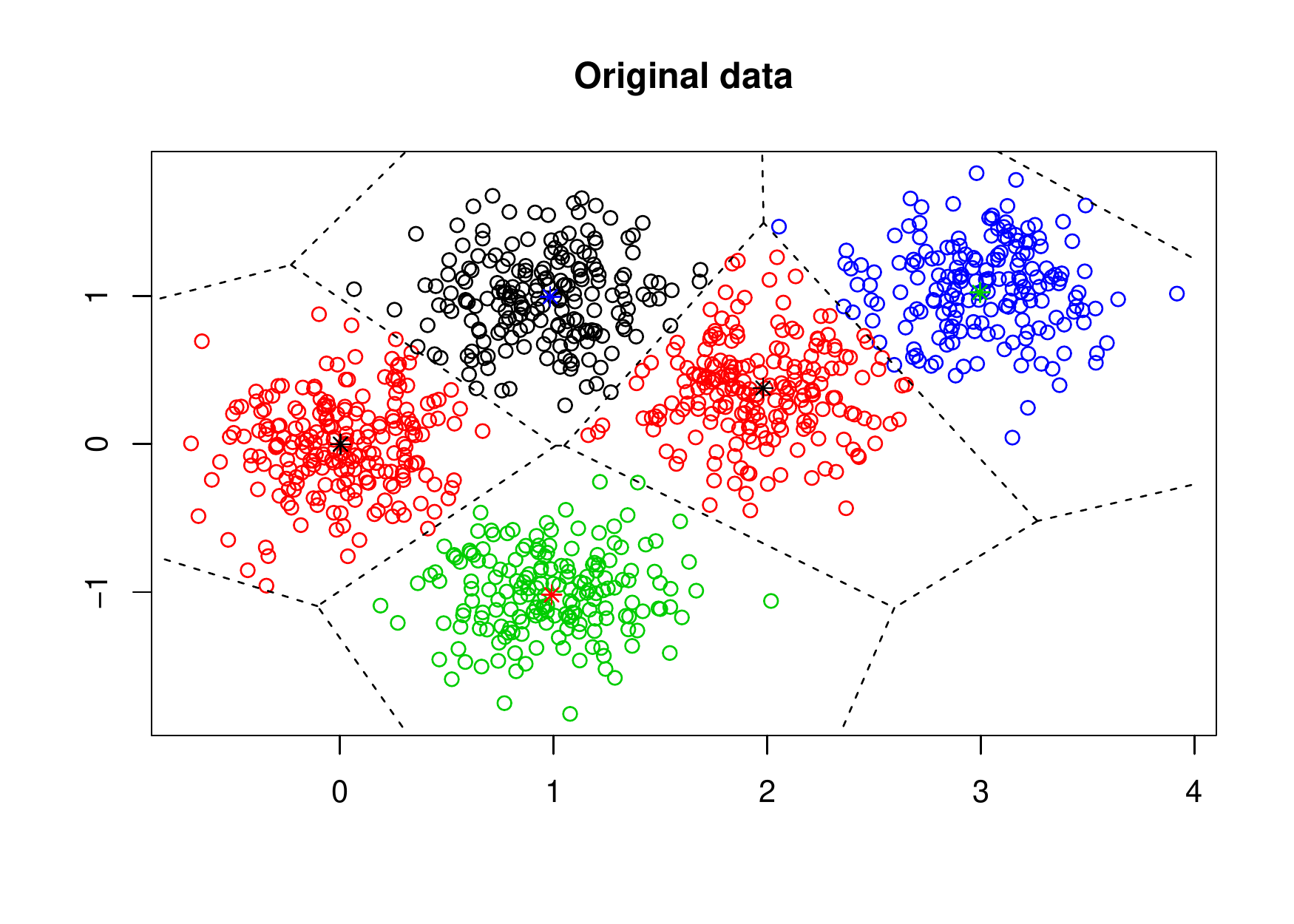}}  %
\caption{A mixture of 5 normal distributions as clustered by $k$-means algorithm (Voronoi diagram superimposed).  
}\label{fig:5clusters}
\end{figure}

Why not moving a second cluster is so problematic?
Let us illustrate the difficulties with the original Kleinberg's consistency by looking at an application of the known $k$-means algorithm, with $k$ being allowed to cover a range, not just a single value, to the two-dimensional data set visible in Figure 
\ref{fig:5clusters}\footnote{Already Ben-David \cite{Ben-David:2009} indicated problems in this direction.
}. This example is a mixture of data points sampled from  5 normal distributions. The $k$-means algorithm with $k=5$, as expected, separates quite well the points from various distributions.   
As visible from the second column of Table \ref{tab:comparison}, in fact $k=5$ does the best job in reducing the unexplained variance. 
Figure \ref{fig:5clustersKleinberg} illustrates a result of a $\Gamma$-transform on the results of the former clustering. 
Visually we would tell that now we have two clusters. 
A look into the third column of the  Table \ref{tab:comparison} convinces that really $k=2$ is the best choice for clustering these data with $k$-means algorithm. 
This of course contradicts Kleinberg's consistency axiom. 
And demonstrates the weakness of outer-consistency concept as well. 

\begin{figure}
\centering
\includegraphics[width=0.8\textwidth]{\figaddr{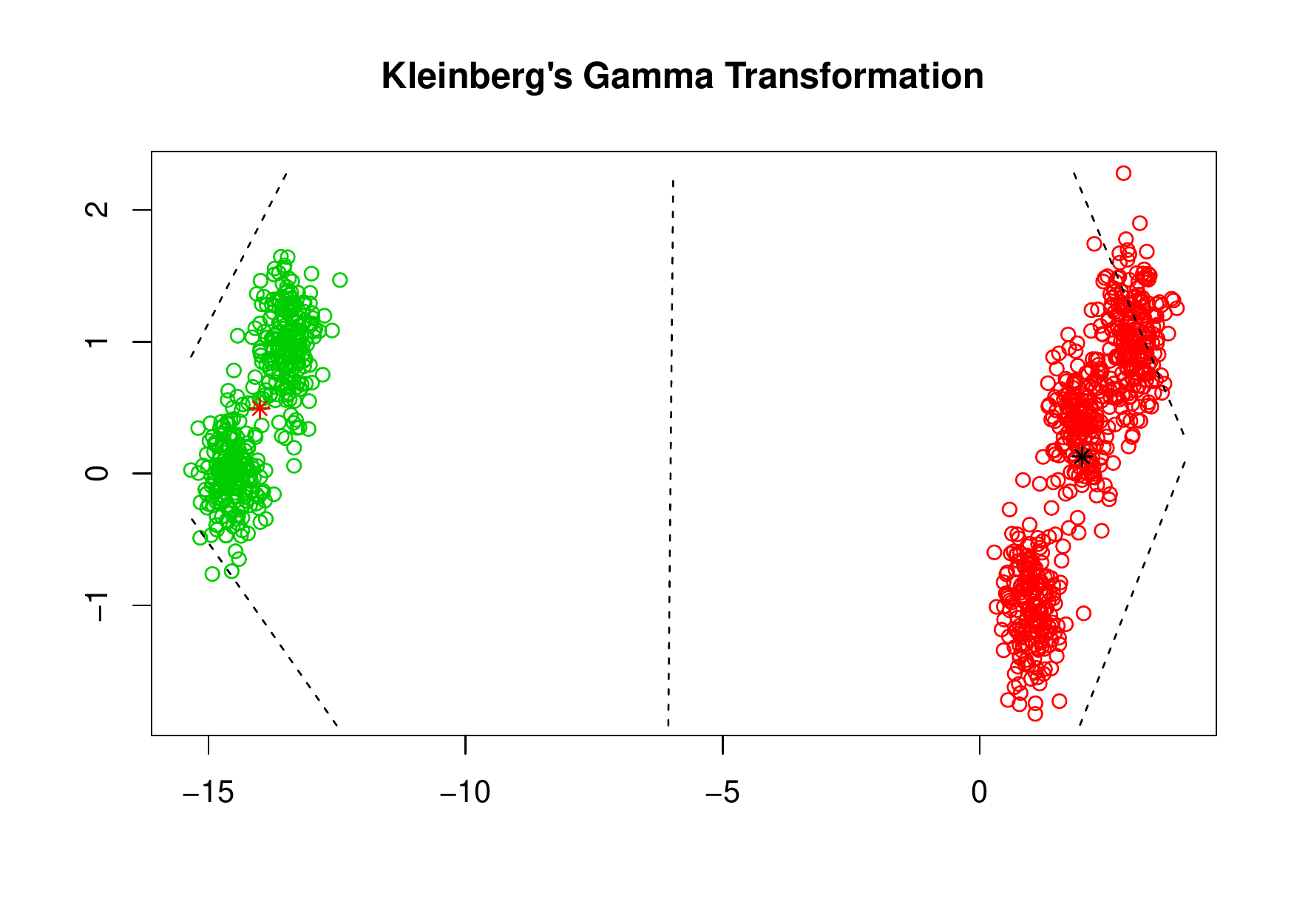}}  %
\caption{Data from 
Figure \ref{fig:5clusters} after Kleinberg's $\Gamma$-transformation 
 clustered by $k$-means algorithm  into two groups. 
}\label{fig:5clustersKleinberg}
\end{figure}

\begin{theorem} {} \label{thm:noouterConsistencykmeans}
$k$-means with $k$ allowed to range over a set of values (approximating richness) 
with limited variance increase criterion for choice of $k$ 
  operating in a fix-dimensional space under Euclidean distance 
cannot conform to outer-consistency axiom. 
\end{theorem}

And finally have a look at Figure \ref{fig:5clusters} once again. 
If we ignore the most right cluster, it turns out that each cluster has points being "surrounded" by points in other clusters. 
Therefore, it is not possible to move a cluster by infinitely small distance without decreasing distances of some different clusters. 
Therefore  

\begin{theorem} {} \label{thm:noouterConsistency}
No algorithm operating in a fix-dimensional space under Euclidean distance 
can conform continuously to outer-consistency axiom. 
\end{theorem}

This theorem contradicts  apparently \cite{Ackerman:2010NIPS}  claim that $k$-means possesses the property of outer-consistency. 
The key word in this theorem is however "continuously" in strict conjunction with "Euclidean distance". It is the embedding into the Euclidean space that causes the problem.

\section{Problems of Richness  Axiom}  \label{sec:badrichness}

As already mentioned, richness or near-richness forces the introduction of "refinement-consistency" which is a too weak concept. 
But even if we allow for such a resolution of the contradiction in Kleinberg's framework, it still does not make it suitable for practical purposes. 
The most serious drawback of Kleinberg's axioms is the richness requirement.

But we may ask whether or not it is possible 
to have richness, that is for any partition there exists always a distance function that the clustering function will return this partition, and yet
if we restrict ourselves to $\mathbb{R}^m$, the very same clustering function is not rich any more, or even it is not anti-chain. 

Consider the following clustering function $f()$. 
If it takes a distance function $d()$ that takes on only two distinct values $d_1$ and $d_2$ such that $d_1<0.5 d_2$ and for any three data points $a,b,c$ if $d(a,b)=d_1, d(b,c)=d_1$ then $d(a,c)=d_1$, it creates clusters of points in such a way that $a,b$ belong to the same cluster if and only if $d(a,b)=d_1$, and otherwise they belong to distinct clusters. If on the other hand $f()$ takes a distance function not exhibiting this property, it works like $k$-means. Obviously, function $f()$ is rich, but at the same time, if confined to $\mathbb{R}^m$, if $n>m+1$ and $k\ll n$, then it is not rich -- it is in fact $k$-rich, and hence not anti-chain. 

Can we get around the problems of  all three Kleinberg's axioms in a similar way in $\mathbb{R}^m$?
Regrettably, 

\begin{theorem}{} \label{thm:KleinbergImpossibilityInRm}
If $\Gamma$ is a partition of $n>2$ elements returned by 
a clustering function $f$ under some distance function $d$,
and  
 $f$
 satisfies 
Consistency, 
then there exists a distance function $d_E$ embedded in 
$\mathbb{R}^m$ for the same set of elements such that 
$\Gamma$ is the partition of this set under $d_E$.
\end{theorem}
The consequence of this theorem is of course that the constructs of contradiction of Kleinberg axioms are simply transposed from the domain of any distance functions to distance functions in $\mathbb{R}^m$. 

\begin{proof}
To show the validity of the theorem, we will construct the appropriate distance function 
$d_E$ by embedding in the $\mathbb{R}^m$.
Let $dmax$ be the maximum distance between the considered elements under $d$. 
Let $C_1,\dots,C_k$ be all the clusters contained in $\Gamma$. 
For each cluster $C_i$ we construct a ball $B_i$ with radius $r_i$ equal to 
$r_i=\frac12 \min_{x,y\in C_i, x\ne y} d(x,y)$.
The ball $B_1$ will be located in the origin of the coordinate system.
$B_{1,\dots,i}$  be the ball of containing all the balls $B_1,\dots,B_i$.
Its center be at $c_{1,\dots,i}$ and radius $r_{1,\dots,i}$. 
The ball $B_{i}$ will be located on the surface of the ball 
with center at $c_{1,\dots,i-1}$ and radius $r_{1\dots,i-1}+dmax+r_{i}$.
For each $i=1,\dots,k$ select distinct locations for elements of $C_i$ within the ball $B_i$. 
The distance function $d_E$ define as the Euclidean distances within $\mathbb{R}^m$  in these constructed locations. 

Apparently, $d_E$ is a $\Gamma$-transform of $d$, as distances between 
elements of $C_i$ are smaller than  or equal to  $2 r_{i}=\min_{x,y\in C_i, x\ne y} d(x,y)$, 
and the distances between elements of different balls exceed 
$dmax$. 
\end{proof}

\begin{table}
\centering
\caption{Variance explained (in percent)
when applying $k$-means algorithm with 
$k=2,\dots,6$ to data from Figures 
\ref{fig:5clusters} (Original),  
\ref{fig:5clustersKleinberg} (Kleinberg)   and 
\ref{fig:5clustersKlopotek} (Centric)  }
\label{tab:comparison}
\begin{tabular}{llll}
\hline 
$k$ & Original & Kleinberg & Centralized \\
\hline 
2 & 54.3 & 98.0 &  54.9\\ 
 3 & 72.2 &   99.17 &   74.3 \\ 
 4  & 83.5 & 99.4 & 86.0\\
 5  & 90.2 &  99.7& 92.9\\
 6 &  91.0 &  99.7& 93.6 \\ 
\hline 
\end{tabular}
\end{table}

But richness is not only a problem in conjunction with scale-invariance and consistency, but rather it is a problem by itself.

It has to be stated first that richness is easy to achieve. 
Imagine the following 'clustering function".
You order nodes by average distance to other nodes, on tights on squared distance and so on, and if no sorting can be achieved, the unsortable points are set into one cluster. 
Then we create an enumeration of all clusters and map it onto unit line segment.
Then we take the quotient of the lowest distance to the largest distance and state that this quotient mapped to that line segment identifies the optimal clustering of the points. 
Though the algorithm is simple in principle (and useless also), and meets axioms of richness and scale -invariance, we have a practical problem:
As no other limitations are imposed, one has to check 
up to 
$ \sum_{k=2}^n
\frac{1}{k!}\sum_{j=1}^{k}(-1)^{k-j}\Big(\begin{array}{l}k\\j\end{array}\Big)j^n \label{CLU:eq-liczebnosc} 
$
possible partitions (Bell number) in order to verify which one of them is the best for a given distance function 
because there must exist at least one distance function suitable for each of them. 
This is prohibitive and cannot be done in reasonable time even if each check is polynomial (even linear) in the dimensions of the task ($n$).

Furthermore, most algorithms of cluster analysis are constructed in an incremental way.
But this can be useless if the clustering quality function is designed in a very unfriendly way.
For example as an XOR function over logical functions of class member distances and non-class member distances
(e.g. being true if the distance rounded to an integer is odd between class members and divisible by a prime number for distances between class members and non-class members, or the same with respect to class center or medoid). 

\begin{table}
\centering
\caption{Data points to be clustered using a ridiculous clustering quality function  }
\label{tab:richnessdata}
\begin{tabular}{lll}
\hline 
id & $x$ coordinate  & $y$ coordinate   \\
\hline 
 1& 4.022346 &5.142886	\\
  2& 3.745942& 4.646777	\\
  3& 4.442992& 5.164956	\\
  4& 3.616975& 5.188107	\\
  5& 3.807503& 5.010183	\\
  6& 4.169602& 4.874328	\\
  7& 3.557578& 5.248182	\\
  8& 3.876208& 4.507264	\\
  9& 4.102748& 5.073515	\\
10& 3.895329& 4.878176\\ 
\hline 
\end{tabular}
\end{table}

Just have a look at sample data from Table \ref{tab:richnessdata}.
A cluster quality function was invented along the above line and exact quality value was computed for partitioning first $n$ points from this data set as illustrated in Table 
\ref{tab:bestpartitionrichnessdata}.
It turns out that the best partition 
for $n$ points does not give any hint for the best partition for $n+1$ points 
therefore each possible partition needs to be investigated in order to find the best one.% 
\footnote{
Strict separation   \cite{Blum:2009} mentioned earlier is another kind of 
a weird cluster quality function, requiring visits to all the partitions 
}

\begin{table}
\centering
\caption{Partition of the best quality 
(the lower the value the better)
after including $n$ first points from Table \ref{tab:richnessdata}.   }
\label{tab:bestpartitionrichnessdata}
\begin{tabular}{llp{5.5cm}}
\hline
$n$ & quality & partition \\
\hline 
  2&1270  & \{ 1, 2 \}\\                                  
  3&1270  & \{ 1, 2 \}  \{ 3 \}\\                           
  4&823   &\{ 1, 3, 4 \}  \{ 2 \}\\                         
  5&315   &\{ 1, 4 \}  \{ 2, 3, 5 \}\\                      
  6&13   &\{ 1, 5 \}  \{ 2, 4, 6 \}  \{ 3 \}\\                
  7&3   &\{ 1, 6 \}  \{ 2, 7 \}  \{ 3, 5 \}  \{ 4 \}\\          
  8&2   &\{ 1, 2, 4, 5, 6, 8 \}  \{ 3 \}  \{ 7 \}\\           
  9&1   &\{ 1, 2, 4, 5 \}  \{ 3, 8 \}  \{ 6, 9 \}  \{ 7 \}\\    
10&1   &\{ 1, 2, 3, 5, 9 \}  \{ 4, 6 \}  \{ 7, 10 \}  \{ 8 \}\\
\end{tabular}
\end{table}

Summarizing these examples, 
the learnability theory points at two basic weaknesses of the richness or even near-richness axioms.
On the one hand the hypothesis space is too big for learning 
a clustering from a sample (it grows too quickly  with the sample size).   
On the other hand an exhaustive search in this space is prohibitive sop that some theoretical clustering functions do not make practical sense.

There is one more problem. 
If the clustering function can fit any data, we are practically unable to learn any structure of data space from data \cite{MAK0:1991}. And this learning capability is necessary at least in the cases: either when  
the data may be only representatives of a larger population 
or the distances are measured with some measurement error (either systematic or random) or both. 
Note that we speak here about a much broader aspect than so-called cluster stability or cluster validity,   pointed at by Luxburg \cite{Luxburg:2011,Luxburg:2009}.

\section{Correcting Formalization of Kleinberg Axioms} \label{sec:KleinbergNotBad}

It is obvious that richness axiom of Kleinberg needs to be replaced 
with a requirement of the space of hypotheses to be "large enough".
For $k$-means algorithm it has been shown via Theorem \ref{th:krichnesskmeans} that $k$-richness is satisfied (and the space is still large, a Bell number of partitions to choose from). 
$k$-means satisfies the scale-invariance axiom, so that 
only the consistency axiom needs to be adjusted to be more realistic. 

Therefore a meaningful redefinition of Kleinberg's $\Gamma$-transform is urgently needed. 
It must not be annihilated by scaling and it must be executable.

Let us create for $\mathbb{R}$ a working definition of   the $\Gamma^*$ transform as follows:
Distances in only one cluster X are changed by 
moving a point along the axis connecting it to cluster X center 
reducing them within the cluster X by the same factor, the distances between any elements outside the cluster X are kept  
[as well as to the gravity center of the cluster X]\footnote{Obviously, 
  for any element outside the cluster X the distance to the closest element of X before the transform will not be smaller than its distance to the closest element of X after the transform. 
Note the shift of attention. We do not insist any longer that the distance to each element of other cluster is increased, rather only the distance to the cluster as a "whole" shall increase. 
This is by the way a stronger version of  inner-consistency which would be insufficient for our purposes.
}

Consider the following one-dimensional clustering function:
For a set of $n\ge 2$ points two elements belong to the same cluster if their distance is strictly lower than $\frac{1}{n+1}$ of the largest distance between the elements. 
When $a,b$ belong to the same cluster and $b,c$ belong to the same cluster, then  $a,c$ belong to the same cluster. 
As a consequence, the minimum distance between elements of distinct clusters is $\frac{1}{n+1}$ of the largest distance between the elements of $S$.
It is easily seen that the weakened richness is fulfilled. 
The scale-invariance is granted by the relativity of inter-cluster distance. And the  consistency under redefined $\Gamma$ transform holds also. 
In this way all three axioms hold. 

A generalization to an Euclidean space of higher dimensionality seems to be quite obvious if there are no ties on distances (the exist one pair of points the distance between which is unique and largest among distances\footnote{otherwise some tie breaking measures have to be taken that would break the any symmetry and allow to choose a unique direction}). 
We embed the points in the space, 
and then say that two points belong to the same cluster 
if the distance along each of the dimensions is lower than 
$\frac{1}{n+1}$ of the largest distance between the elements along the respective dimension. 
 The distance is then understood as the maximum of distances along all dimensions. 

\begin{definition}
Let $\Gamma$ be a partition embedded in $\mathbb{R}^m$. 
Let $C\in \Gamma$ and let $\boldsymbol\mu_c$ be the center of the cluster $C$.
We say that we execute the \emph{$\Gamma^*$} transform (or a centric consistency  transformation)  
if for some $0<\lambda\le 1$ 
we create a set $C'$ with cardinality identical with $C$
 such that for each element  $\textbf{x}\in C$ there exists 
 $\textbf{x'}\in C'$  such that $\textbf{x'}=\boldsymbol\mu_c+\lambda(\textbf{x}-\boldsymbol\mu_c)$,
and then substitute $C$ in $\Gamma$ with $C'$. 
\end{definition} 
 
\begin{ax}\label{ax:centricconsistency}
A method matches the condition of \emph{centric consistency}
if after a $\Gamma^*$ transform it returns the same partition. 
\end{ax}

Hence 
\begin{theorem}{} \label{thm:KleinbergImpossibilityDenied} 
For
each
$n\ge 2$,
there exists   clustering  function 
 $f$
that satisfies
Scale-Invariance,
 near-Richness,
and
  Centric-Consistency\footnote{
Any algorithm being  consistent
 is also refinement-consistent. 
Any algorithm being  inner-consistent
 is also consistent. 
Any algorithm being  outer-consistent
 is also consistent. 
But there are no such subsumptions for the centric-consistency.
}.
\end{theorem}
\begin{proof}
The above-mentioned clustering function is the proof of validity of this theorem. 
\end{proof}
This way of resolving Kleinberg's contradictions differs from earlier approaches in that 
a realistic embedding into an $\mathbb{R}^m$ is considered and the distances are metric.

We created herewith  the possibility of shrinking a single cluster without having to "move" the other ones. 
As pointed out, this was impossible under Kleinberg's $\Gamma$ transform, that is under increase of all distances between objects from distinct clusters. 
In fact intuitively we do not want the objects to be more distant but rather the clusters. 
We proposed to keep the cluster centroid unchanged while decreasing distances between cluster elements proportionally, insisting that no distance of other elements to the closest element of the shrunk cluster should decrease. 
This approach is pretty rigid. 
It assumes that we are capable to embed the objects into some Euclidean space so that the centroid has a meaning. 

\section{$k$-means fitting centric-consistency axiom}
\label{sec:kmeanscentricconsistent}

\begin{figure}
\centering
\includegraphics[width=0.8\textwidth]{\figaddr{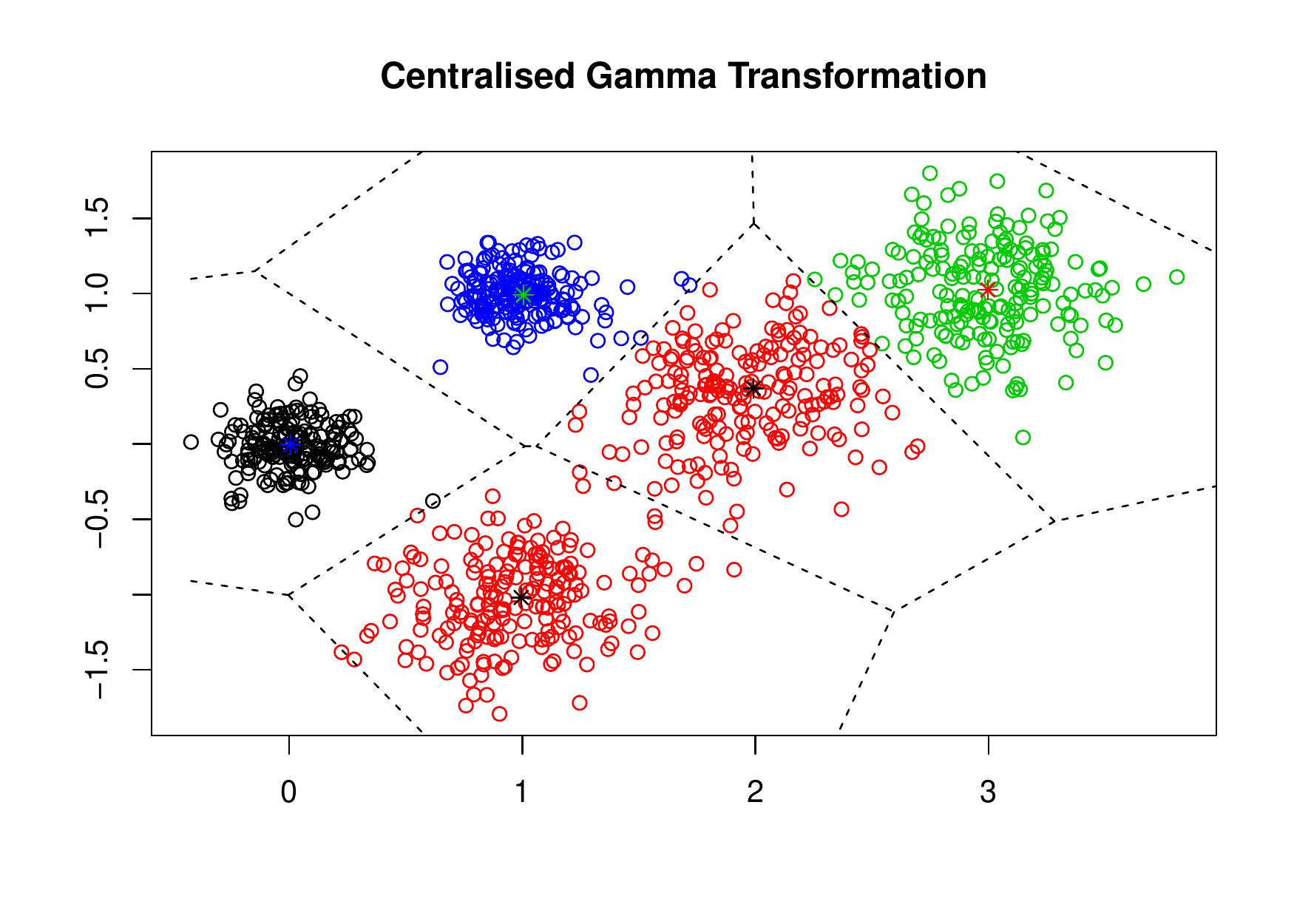}}  %
\caption{Data from 
Figure \ref{fig:5clusters} after a centralized  $\Gamma$-transformation ($\Gamma^*$ transformation),
 clustered by $k$-means algorithm  into 5 groups
}\label{fig:5clustersKlopotek}
\end{figure}

Our proposal of centric-consistency   has a practical background. 
Kleinberg proved that $k$-means does not fit his consistency axiom. 
As shown experimentally in table \ref{tab:comparison}, $k$-means algorithm behaves properly under $\Gamma^*$ transformation. 
Figure \ref{fig:5clustersKlopotek} illustrates a two-fold application of the 
$\Gamma^*$ transform (same clusters affected as by $\Gamma$-transform in the preceding figure).
As recognizable visually and by inspecting the forth column 
of  Table \ref{tab:comparison}, here $k=5$ is the best choice for $k$-means algorithm, so the centric-consistency axiom is followed. 

Let us now demonstrate theoretically, that $k$-means algorithm really fits 
"in the limit" the centric-consistency axiom.
\begin{theorem}{} \label{thm:localCentricConsistencyForKmeans} 
$k$-means algorithm    satisfies
 centric consistency in the following way:  
if the partition $\Gamma$ is a local minimum of $k$-means, 
and the partition $\Gamma$  has been subject to centric consistency yielding $\Gamma'$, then $\Gamma'$ is also a local minimum of $k$-means.  
\end{theorem}

\begin{proof}
The $k$-means algorithm minimizes the sum\footnote{We use here the symbol $Q$ for the cluster quality function instead of $J$ from section \ref{sec:prevWork} because $Q$ does not fit axiomatic system for $J$ - it is not scale-invariant and in case of consistency it changes in opposite direction, and with respect of richness we can only apply $k$-richness.} 
$Q$ from equation (\ref{eq:Q::kmeans}). 
$V(C_j)$ be the sum of squares of distances of all objects of the cluster   $C_j$ from its gravity center. 
Hence $Q(\Gamma)= \sum_{j=1}^k \frac{1}{n_j} V(C_j)$. 
Consider moving a data point $\textbf{x}^*$  from the cluster $C_{j_0}$ 
to cluster $C_{j_l}$ 
As demonstrated by \cite{Duda:1973},  
$V( C_{j_0} -\{\textbf{x}^*\})= V(C_{j_0}) -
  \frac{n_{j_0}}{n_{j_0}-1}\|\textbf{x}^*-\boldsymbol{\mu}_{j_0}\|^2$
and 
$V(C_{j_l}\cup\{\textbf{x}^*\})= V(C_{j_l}) +
 \frac{n_l}{n_l+1}\|\textbf{x}^*-\boldsymbol{\mu}_{j_l}\|^2 $
So it pays off to move a point from one cluster to another if
$\frac{n_{j_0}}{n_{j_0}-1}\|\textbf{x}^*-\boldsymbol{\mu}_{j_0}\|^2 > \frac{n_{j_l}}{n_{j_l}+1}\|\textbf{x}^*-\boldsymbol{\mu}_{j_l}\|^2$.
If we assume local optimality of $\Gamma$, 
this obviously did not pay off.
Now transform this data set to $\mathbf{X'}$ in that 
  we transform elements of cluster $C_{j_0}$ in such a way that 
it has now elements 
$\textbf{x}_i'=\textbf{x}_i+\lambda(\textbf{x}_i - \boldsymbol{\mu}_{j_0})$ for some $0<\lambda<1$,
see figure \ref{fig:XBETWEENOUTSIDE}.
Consider a   partition $\Gamma'$
of $\mathbf{X'}$.  All clusters are the same as in $\Gamma$ except for 
the transformed elements that form now a cluster $C'_{j_0}$.
The question  is: 
does it pay off   to move a data point $\textbf{x'}^*\in C'_{j_0}$ between the clusters? 
Consider the plane containing 
$\textbf{x}^*, \boldsymbol{\mu}_{j_0}, \boldsymbol{\mu}_{j_l}$.
Project orthogonally the point
$\textbf{x}^*$ onto the line $\boldsymbol{\mu}_{j_0}, \boldsymbol{\mu}_{j_l}$, giving a point $\textbf{p}$.
Either $\textbf{p}$  lies  between 
$\boldsymbol{\mu}_{j_0}, \boldsymbol{\mu}_{j_l}$
or 
  $\boldsymbol{\mu}_{j_0}$  lies  between 
$\textbf{p}, \boldsymbol{\mu}_{j_l}$. Properties of $k$-means exclude other possibilities. 
Denote distances 
$y=\|\textbf{x}^*- \textbf{p}\|$,
$x=\|\boldsymbol{\mu}_{j_0}- \textbf{p}\|$,
$d=\|\boldsymbol{\mu}_{j_0}- \boldsymbol{\mu}_{j_l}\|$
In the second case 
the condition that moving the point does not pay off means:
$$\frac{n_{j_0}}{n_{j_0}-1} (x^2+y^2) \le \frac{n_{j_l}}{n_{j_l}+1}((d+x)^2+y^2) $$ 
If we multiply both sides with $\lambda^2$, we have:
\begin{align}
\lambda^2\frac{n_{j_0}}{n_{j_0}-1} (x^2+y^2) 
=& \frac{n_{j_0}}{n_{j_0}-1} ((\lambda x)^2+(\lambda y)^2) 
\nonumber \\
\le & \lambda^2 \frac{n_{j_l}}{n_{j_l}+1}((d+x)^2+y^2) 
\nonumber \\
=&
\frac{n_{j_l}}{n_{j_l}+1}( \lambda^2d^2+\lambda^2 2dx+\lambda^2x^2 +\lambda^2y^2) 
\nonumber \\
\le 
& \frac{n_{j_l}}{n_{j_l}+1}(  d^2+  2d\lambda x+\lambda^2x^2 +\lambda^2y^2) 
\nonumber \\
=& \frac{n_{j_l}}{n_{j_l}+1}(  (d +  \lambda x)^2+ (\lambda y)^2) 
\end{align} 
which means that it does not payoff to move the point $\textbf{x'}^*$ between clusters either.
Consider now the first case and assume that it pays off to move $\textbf{x'}^*$.
So we would have 
$$\frac{n_{j_0}}{n_{j_0}-1} (x^2+y^2) \le \frac{n_{j_l}}{n_{j_l}+1}((d-x)^2+y^2) $$ 
and at the same time
$$\frac{n_{j_0}}{n_{j_0}-1} \lambda^2(x^2+y^2) > \frac{n_{j_l}}{n_{j_l}+1}((d-\lambda x)^2+\lambda^2 y^2) $$ 
Subtract now both sides:
$$\frac{n_{j_0}}{n_{j_0}-1} (x^2+y^2) 
-\frac{n_{j_0}}{n_{j_0}-1} \lambda^2(x^2+y^2)$$
$$< \frac{n_{j_l}}{n_{j_l}+1}((d-x)^2+y^2) 
-\frac{n_{j_l}}{n_{j_l}+1}((d-\lambda x)^2+\lambda^2 y^2)$$
This  implies 
$$\frac{n_{j_0}}{n_{j_0}-1} (1-\lambda^2) (x^2+y^2) 
< 
\frac{n_{j_l}}{n_{j_l}+1}(
 (1-\lambda^2) (x^2+y^2)-2d\lambda x
)
$$ 
 It is a contradiction because 
$$\frac{n_{j_0}}{n_{j_0}-1} (1-\lambda^2) (x^2+y^2) 
>
\frac{n_{j_l}}{n_{j_l}+1} (1-\lambda^2) (x^2+y^2) 
 > 
\frac{n_{j_l}}{n_{j_l}+1}(
 (1-\lambda^2) (x^2+y^2)-2d\lambda x
)
$$
So it does not pay off to move $\textbf{x'}^*$, hence the partition $\Gamma'$ remains  locally optimal for the transformed data set. 
\end{proof}
If the data have one stable optimum only like in case of
"well separated" normally distributed 
   $k$ real clusters, then both turn to global optima. 

\begin{figure}
\centering
\includegraphics[width=0.45\textwidth]{\figaddr{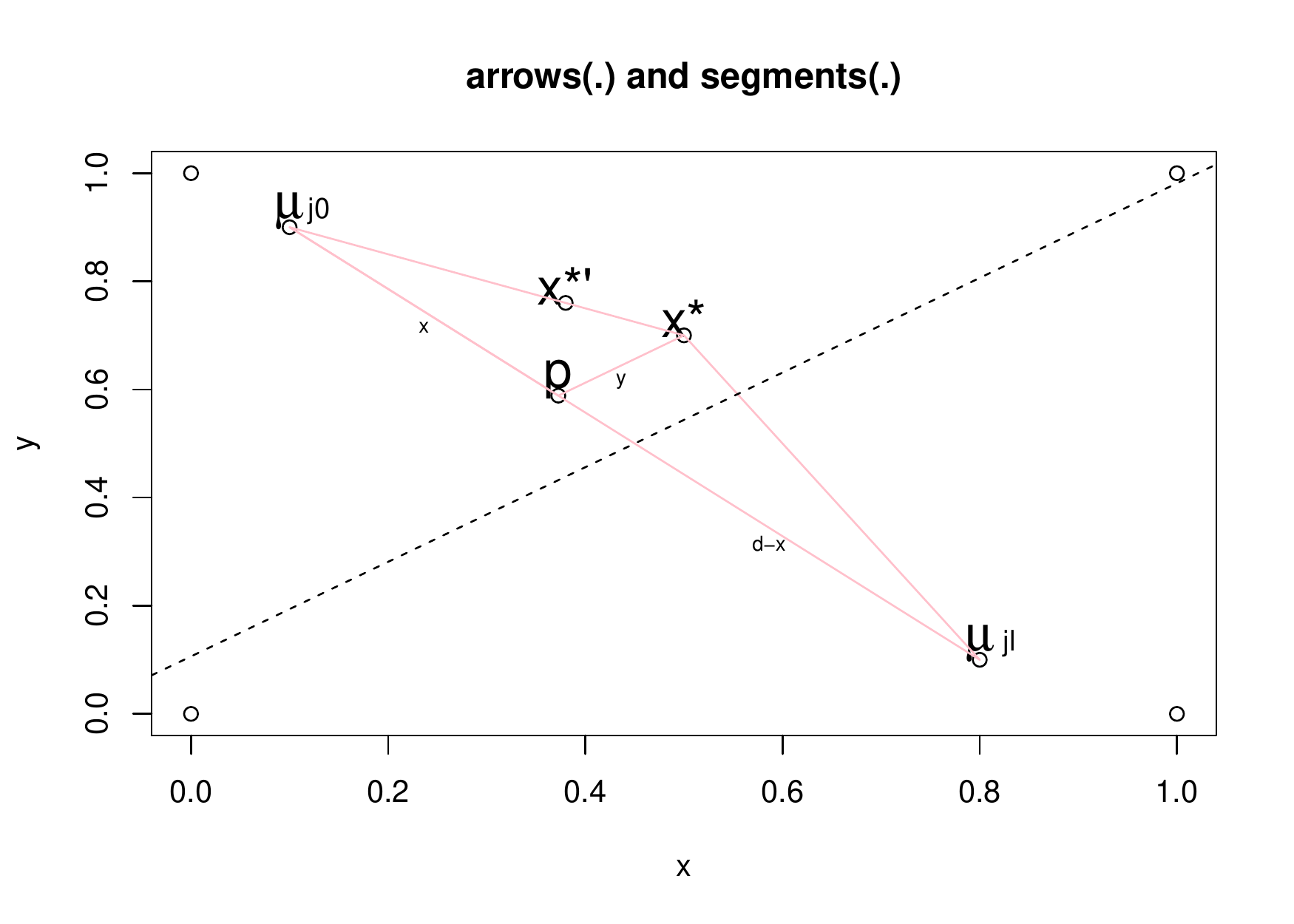}}  %
\includegraphics[width=0.45\textwidth]{\figaddr{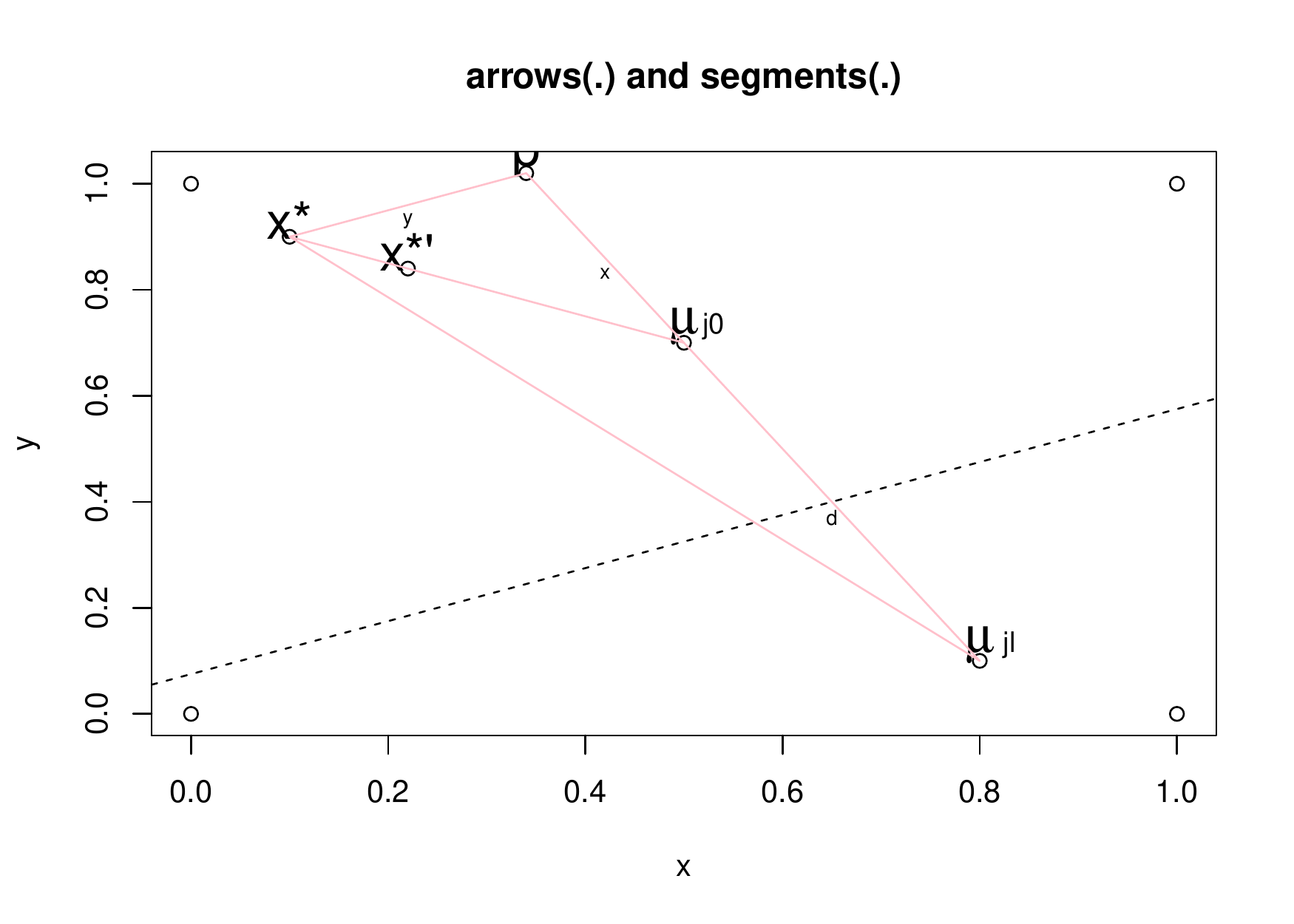}}  %
\caption{Impact of contraction 
towards cluster center by a factor lambda 
- local optimum maintained 
}\label{fig:XBETWEENOUTSIDE}
\end{figure}

%-----------------------------------------------------
However, it is possible to demonstrate that 
the newly defined transform preserves also the global optimum of $k$-means.
\begin{theorem}{} \label{thm:globalCentricCinsistencyForKmeans} 
$k$-means algorithm    satisfies
 centric consistency in the following way:  
if the partition $\Gamma$ is a global minimum of $k$-means, 
and the partition $\Gamma$  has been subject to centric consistency yielding $\Gamma'$, then $\Gamma'$ is also a global minimum of $k$-means.  
\end{theorem}
 
\begin{proof}
Let us consider first the simple case of two clusters only (2-means). 
Let the optimal clustering for a given set of objects $X$ consist of two clusters: $T$ and $Z$. 
The subset $T$ shall have its gravity center at the origin of the coordinate system. 
The quality of this partition 
$Q(\{T,Z\})=n_T Var(T) +n_Z Var(Z)$
where $n_T, n_Z$   denote the cardinalities of $T,Z$ and
$Var(T),Var(Z)$ their variances (averaged squared distances to gravity center). 
We will prove  by contradiction that by applying our $\Gamma$ transform we get partition that will be still optimal for the transformed data points.
We shall assume the contrary that is that we can 
transform the set $T$ by some $1>\lambda>0$ to $T'$ in such a way that optimum of $2$-means clustering is not the partition $\{T',Z\}$ but another one, say 
$\{A'\cup D,B'\cup C\}$ where $Z=C\cup D$, $A'$ and $B'$ are transforms of sets $A,B$ for which in turn $A\cup B=T$. 
It may be easily verified that
\newcommand{\w}{\mathbf{v}}  
$$Q(\{A\cup B, C\cup D\})= 
n_A Var(A) + n_A \w_A^2 
+ n_B Var(B) + n_B \w_B^2 
$$ $$
+ n_C Var(C) + n_D Var(D) +\frac{n_C n_D}{n_C+n_D} (\w_C-\w_D)^2$$  
while  
$$Q(\{A\cup C, B\cup D\})= 
n_A Var(A) + n_D Var(D)+ 
+\frac{n_A n_D}{n_A+n_D} (\w_A-\w_D)^2
$$ $$
+ n_B Var(B) + n_C Var(C)+ 
+\frac{n_B n_C}{n_B+n_C} (\w_B-\w_C)^2
$$  
and 
$$Q(\{A'\cup B', C\cup D\})= 
n_A \lambda^2Var(A) + n_A \lambda^2\w_A^2 
+ n_B \lambda^2Var(B) + n_B \lambda^2\w_B^2 
$$ $$
+ n_C Var(C) + n_D Var(D) +\frac{n_C n_D}{n_C+n_D} (\w_C-\w_D)^2$$  
while  
$$Q(\{A'\cup C, B'\cup D\})= 
n_A \lambda^2Var(A) + n_D Var(D)+ 
+\frac{n_A n_D}{n_A+n_D} (\lambda\w_A-\w_D)^2
$$ $$
+ n_B \lambda^2Var(B) + n_C Var(C)+ 
+\frac{n_B n_C}{n_B+n_C} (\lambda\w_B-\w_C)^2
$$  
The following must hold:
\begin{equation}
Q(\{A'\cup B', C\cup D\})>Q(\{A'\cup D, B'\cup C\})
\end{equation}
and 
\begin{equation}
Q(\{A\cup B, C\cup D\})<Q(\{A\cup D, B\cup C\})
\end{equation}
Additionally also 
\begin{equation}
Q(\{A\cup B, C\cup D\})<Q(\{A\cup B\cup C,  D  \})
\end{equation}
 and 
\begin{equation}
Q(\{A\cup B, C\cup D\})<Q(\{A\cup B\cup D,  C  \})
\end{equation}
These two latter inequalities imply:
\Bem{
$$Q(\{A\cup B, C\cup D\})= 
n_A Var(A) + n_A \w_A^2 
+ n_B Var(B) + n_B \w_B^2 
$$ $$
+ n_C Var(C) + n_D Var(D) +\frac{n_C n_D}{n_C+n_D} (\w_C-\w_D)^2$$  
$$<Q(\{A\cup B\cup C, D\})= 
n_A Var(A) + n_A \w_A^2 
+ n_B Var(B) + n_B \w_B^2 
$$ $$
+ n_C Var(C)+\frac{(n_A+n_B) n_C}{(n_A+n_B)+n_C} \w_C ^2 
+n_D Var(D) 
 $$ 
that is
}%\Bem
$$ 
   \frac{n_C n_D}{n_C+n_D} (\w_C-\w_D)^2 < 
  \frac{(n_A+n_B) n_C}{(n_A+n_B)+n_C} \w_C ^2 
 $$   
and
$$ 
   \frac{n_C n_D}{n_C+n_D} (\w_C-\w_D)^2 < 
  \frac{(n_A+n_B) n_D}{(n_A+n_B)+n_D} \w_D ^2 
 $$  
Consider now an extreme contraction ($\lambda=0$) yielding sets $A",B"$ out of $A,B$.
Then we have 

$$Q(\{A"\cup B", C\cup D\})
- Q(\{A"\cup C, B"\cup D\})
$$ $$
= 
   \frac{n_C n_D}{n_C+n_D} (\w_C-\w_D)^2
-  
 \frac{n_A n_D}{n_A+n_D}  \w_D^2
-\frac{n_B n_C}{n_B+n_C} \w_C^2
$$ $$
= 
   \frac{n_C n_D}{n_C+n_D} (\w_C-\w_D)^2
$$ $$
-  
 \frac{n_A n_D}{n_A+n_D}
\frac{(n_A+n_B)+n_D}{(n_A+n_B) n_D}
\frac{(n_A+n_B) n_D}{(n_A+n_B)+n_D}
  \w_D^2
$$ $$
-\frac{n_B n_C}{n_B+n_C}
\frac{(n_A+n_B)+n_C}{(n_A+n_B) n_C}
\frac{(n_A+n_B) n_C}{(n_A+n_B)+n_C}
 \w_C^2
$$ $$
= 
   \frac{n_C n_D}{n_C+n_D} (\w_C-\w_D)^2
$$ $$
-  
 \frac{n_A  }{n_A+n_D}
\frac{(n_A+n_B)+n_D}{(n_A+n_B)  }
\frac{(n_A+n_B) n_D}{(n_A+n_B)+n_D}
  \w_D^2
$$ $$
-\frac{n_B  }{n_B+n_C}
\frac{(n_A+n_B)+n_C}{(n_A+n_B)  }
\frac{(n_A+n_B) n_C}{(n_A+n_B)+n_C}
 \w_C^2
$$ $$
= 
   \frac{n_C n_D}{n_C+n_D} (\w_C-\w_D)^2
$$ $$
-  
  \frac{n_A   }{ n_A+n_B   }
(1+\frac{n_B}{n_A+n_D})
\frac{(n_A+n_B) n_D}{(n_A+n_B)+n_D}
  \w_D^2
$$ $$
-\frac{n_B  }{n_A+n_B  }
(1+\frac{n_A}{n_B+n_C})
\frac{(n_A+n_B) n_C}{(n_A+n_B)+n_C}
 \w_C^2
$$ $$
< 
   \frac{n_C n_D}{n_C+n_D} (\w_C-\w_D)^2
$$ $$
-  
  \frac{n_A   }{ n_A+n_B   }
\frac{(n_A+n_B) n_D}{(n_A+n_B)+n_D}
  \w_D^2
$$ $$
-\frac{n_B  }{n_A+n_B  } 
\frac{(n_A+n_B) n_C}{(n_A+n_B)+n_C}
 \w_C^2
<0
$$
because the linear combination of two numbers that are bigger than a third yields another number bigger than this. 
Let us define a function 
$$h(x)=
  + n_A x^2\w_A^2 
  + n_B x^2\w_B^2 
    +\frac{n_C n_D}{n_C+n_D} (\w_C-\w_D)^2
$$ $$    
-\frac{n_A n_D}{n_A+n_D} (x\w_A-\w_D)^2  
-\frac{n_B n_C}{n_B+n_C} (x\w_B-\w_C)^2
$$   
It can be easily verified that $h(x)$ is a quadratic polynomial 
with a positive coefficient at $x^2$.
Furthermore
$h(1)=Q(\{A\cup B, C\cup D\})
- Q(\{A\cup C, B\cup D\})<0$, 
$h(\lambda)=Q(\{A'\cup B', C\cup D\})
- Q(\{A'\cup C, B'\cup D\})>0$, 
$h(0)=Q(\{A"\cup B", C\cup D\})
- Q(\{A"\cup C, B"\cup D\})<0$. 
But no quadratic polynomial with a positive coefficient at $x^2$ can be negative at the ends of an interval and positive in the middle.
So we have the contradiction. This proves the thesis
that the (globally) optimal $2$-means clustering remains 
(globally) optimal after transformation. 
%-----------------------------------------------

Let us turn to the general case of $k$-means. 
Let the optimal clustering for a given set of objects $X$ consist of $k$ clusters: $T$ and $Z_1,\dots,Z_{k-1}$. 
The subset $T$ shall have its gravity center at the origin of the coordinate system. 
The quality of this partition 
$Q(\{T,Z_1,\dots,Z_{k-1}\}) =n_T Var(T) +\sum_{i=1}^{k-1}n_{Z_i} Var(Z_i)$,  
where $n_{Z_i} $ is the cardinality of the cluster ${Z_i} $.
We will prove  by contradiction that by applying our $\Gamma$ transform we get partition that will be still optimal for the transformed data points.
We shall assume the contrary that is that we can 
transform the set $T$ by some $1>\lambda>0$ to $T'$ in such a way that optimum of $k$-means clustering is not the partition $\{T',Z_1,\dots,Z_{k-1}\}$ but another one, say 
$\{T'_1\cup Z_{1,1} \cup \dots \cup Z_{k-1,1}
, T'_2\cup Z_{1,2} \cup \dots \cup Z_{{k-1},2}
\dots
, T'_k\cup Z_{1,k} \cup \dots \cup Z_{{k-1},k}
\}$ where $Z_i= \cup_{j=1}^{k} Z_{i,j}$ (where $Z_{i,j}$ are pairwise disjoint), 
$T'_1,\dots,T'_k$   are transforms of disjoint sets 
$T_1,\dots,T_k$ for which in turn $\cup_{j=1}^{k}T_j=T$. 
It may be easily verified that
$$Q(\{T,Z_1,\dots,Z_{k-1}\})= 
\sum_{j=1}^{k}n_{T_j}  Var({T_j}) + \sum_{j=1}^{k}n_{T_j} \w_{T_j}^2 
%% $$ $$
+ \sum_{i=1}^{k-1}n_{Z_i} Var({Z_i})$$  
while (denoting $Z_{*,j}= \cup_{i=1}{k-1}Z_{*,j}$) 
$$Q(\{T_1\cup  Z_{*,1},
\dots, T_k\cup  Z_{*,k}
  \})= $$ $$ = 
\sum_{j=1}^{k}
\left(
n_{T_j} Var({T_j}) + n_{Z_{*,j}}  Var(Z_{*,j})+ 
+\frac{n_{T_j} n_{Z_{*,j}}}{n_{T_j}+n_{Z_{*,j}}} (\w_{T_j}-\w_{Z_{*,j}})^2
\right)
$$ 
whereas
$$Q(\{T',Z_1,\dots,Z_{k-1}\})= 
\sum_{j=1}^{k}n_{T_j}  \lambda^2Var({T_j}) + \sum_{j=1}^{k}n_{T_j} \lambda^2\w_{T_j}^2 
$$ $$
+ \sum_{i=1}^{k-1}n_{Z_i} Var({Z_i})$$  
while
$$Q(\{T'_1\cup  Z_{*,1},
\dots, T'_k\cup  Z_{*,k}
  \})= $$ $$ = 
\sum_{j=1}^{k}
\left(
n_{T_j} \lambda^2Var({T_j}) + n_{Z_{*,j}}  Var(Z_{*,j})+ 
+\frac{n_{T_j} n_{Z_{*,j}}}{n_{T_j}+n_{Z_{*,j}}} (\lambda\w_{T_j}-\w_{Z_{*,j}})^2
\right)
$$ 
The following must hold:
\begin{equation}
Q(\{T',Z_1,\dots,Z_{k-1}\})>Q(\{
T'_1\cup  Z_{*,1},
\dots, T'_k\cup  Z_{*,k} \})
\end{equation} 
and 
\begin{equation}
Q(\{T,Z_1,\dots,Z_{k-1}\})<Q(\{
\{T_1\cup  Z_{*,1},
\dots, T_k\cup  Z_{*,k}\})
\end{equation}

Additionally also 
\begin{equation}
Q(\{T,Z_1,\dots,Z_{k-1}\})<Q(\{
\{T\cup  Z_{*,1},Z_{*,2},  \dots, Z_{*,k}
)\end{equation}
 and  
\begin{equation}
Q(\{T,Z_1,\dots,Z_{k-1}\})<Q(\{
T\cup  Z_{*,2},Z_{*,1},Z_{*,3},  \dots, Z_{*,k} \}
)\end{equation}
 and \dots and   
\begin{equation}
Q(\{T,Z_1,\dots,Z_{k-1}\})<Q(\{
T\cup  Z_{*,k},Z_{*,1},  \dots, Z_{*,k-1} \}
)\end{equation}

These latter $k$ inequalities imply that for $l=1,\dots,k$:
$$Q(\{T,Z_1,\dots,Z_{k-1}\})= 
 n_{T}  Var({T})+
\sum_{j=1}^{k}n_{T_j}  Var({T_j}) + \sum_{j=1}^{k}n_{T_j} \w_{T_j}^2 
$$ $$
+ \sum_{i=1}^{k-1}n_{Z_i} Var({Z_i})
< $$ $$
Q(\{T\cup  Z_{*,l},
Z_{*,1},\dots,Z_{*,l-1},Z_{*,l+1}
\dots,   Z_{*,k}
  \})= $$ $$ = 
 n_{T}  Var({T})+
\sum_{j=1}^{k}n_{Z_{*,j}}  Var(Z_{*,j})
+\frac{n_{T} n_{Z_{*,l}}}{n_{T}+n_{Z_{*,l}}} (\w_{T}-\w_{Z_{*,l}})^2
$$
%--------
  $$
+ \sum_{i=1}^{k-1}n_{Z_i} Var({Z_i})
< $$ $$
\sum_{j=1}^{k}n_{Z_{*,j}}  Var(Z_{*,j})
+\frac{n_{T} n_{Z_{*,l}}}{n_{T}+n_{Z_{*,l}}} (\w_{T}-\w_{Z_{*,l}})^2
$$
%--------
  $$
+  \sum_{i=1}^{k-1}n_{Z_i} Var({Z_i})
-\sum_{j=1}^{k}n_{Z_{*,j}}  Var(Z_{*,j})
< $$ $$
\frac{n_{T} n_{Z_{*,l}}}{n_{T}+n_{Z_{*,l}}} (\w_{Z_{*,l}})^2
$$

Consider now an extreme contraction ($\lambda=0$) yielding sets $T_j"$ out of $T_j$.
Then we have 

$$Q(\{T",Z_1,\dots,Z_{k-1}\})
-Q(\{T"_1\cup  Z_{*,1},
\dots, T"_k\cup  Z_{*,k}
  \})
$$ $$
= 
   \sum_{i=1}^{k-1}n_{Z_i} Var({Z_i}) 
-
\sum_{j=1}^{k}
\left(
  n_{Z_{*,j}}  Var(Z_{*,j}) 
+\frac{n_{T_j} n_{Z_{*,j}}}{n_{T_j}+n_{Z_{*,j}}} (\w_{Z_{*,j}})^2
\right)
$$ $$
= 
 \sum_{i=1}^{k-1}n_{Z_i} Var({Z_i})
-\sum_{j=1}^{k} n_{Z_{*,j}}  Var(Z_{*,j})
$$ $$
-\sum_{j=1}^{k}
\frac{n_{T_j} n_{Z_{*,j}}}{n_{T_j}+n_{Z_{*,j}}}
\frac{n_{T}+n_{Z_{*,j}}} {n_{T} n_{Z_{*,j}}}
\frac{n_{T} n_{Z_{*,j}}}{n_{T}+n_{Z_{*,j}}} 
 (\w_{Z_{*,j}})^2
$$ $$
= 
 \sum_{i=1}^{k-1}n_{Z_i} Var({Z_i})  
-\sum_{j=1}^{k} n_{Z_{*,j}}  Var(Z_{*,j})
$$ $$
-\sum_{j=1}^{k}
\frac{n_{T_j}  }{n_{T_j}+n_{Z_{*,j}}}
\frac{n_{T}+n_{Z_{*,j}}} {n_{T}  }
\frac{n_{T} n_{Z_{*,j}}}{n_{T}+n_{Z_{*,j}}} 
 (\w_{Z_{*,j}})^2
$$ $$
\le  
 \sum_{i=1}^{k-1}n_{Z_i} Var({Z_i})  
-\sum_{j=1}^{k} n_{Z_{*,j}}  Var(Z_{*,j})
-\sum_{j=1}^{k}
\frac{n_{T_j}  } {n_{T}  }
\frac{n_{T} n_{Z_{*,j}}}{n_{T}+n_{Z_{*,j}}} 
 (\w_{Z_{*,j}})^2
<0
$$
because the linear combination of  numbers that are bigger than a third yields another number bigger than this. 
Let us define a function 
$$g(x)=
 \sum_{j=1}^{k}n_{T_j} x^2\w_{T_j}^2 
+ \sum_{i=1}^{k-1}n_{Z_i} Var({Z_i})  
$$ $$
-
\sum_{j=1}^{k}
\left(
  n_{Z_{*,j}}  Var(Z_{*,j})+ 
+\frac{n_{T_j} n_{Z_{*,j}}}{n_{T_j}+n_{Z_{*,j}}} (x\w_{T_j}-\w_{Z_{*,j}})^2
\right)
$$ 
It can be easily verified that $g(x)$ is a quadratic polynomial 
with a positive coefficient at $x^2$.
Furthermore
$g(1)=Q(\{T,Z_1,\dots,z_{k-1}\})
- Q(\{T_1\cup  Z_{*,1},\dots, T_k\cup  Z_{*,k} \})<0$, 
$g(\lambda)=Q(\{T',Z_1,\dots,Z_{k-1}\})
- Q(\{T'_1\cup  Z_{*,1},\dots, T'_k\cup  Z_{*,k} \})>0$, 
$g(0)=Q(\{T",Z_1,\dots,Z_{k-1}\})
- Q(\{T"_1\cup  Z_{*,1},\dots, T"_k\cup  Z_{*,k} \})<0$. 
But no quadratic polynomial with a positive coefficient at $x^2$ can be negative at the ends of an interval and positive in the middle.
So we have the contradiction. This proves the thesis
that the (globally) optimal $k$-means clustering remains 
(globally) optimal after transformation. 
\end{proof}

So summarizing the new $\Gamma$ transformation preserves local and global optima of $k$-means for a fixed $k$.
Therefore $k$-means algorithm is consistent under this transformation.

Hence 
\begin{theorem}{} \label{thm:KleinbergImpossibilityDeniedForKmeans} 
$k$-means algorithm    satisfies
Scale-Invariance,
 $k$-Richness,
and
 centric Consistency.
\end{theorem}

Note that ($\Gamma^*$ based) centric  Consistency 
is not a specialization of Kleinberg's consistency 
as the requirement of increased distance between all elements of different clusters is not required in
$\Gamma^*$ based  Consistency. Note also that the decrease of distance does not need to be 
equal for all elements as long as the gravity center does not relocate. 
Also a limited rotation of the cluster may be allowed for. 
\Bem{
But we could strengthen centric-consistency to be in concordance with Kleinberg['s consistency and under this strengthening $k$-means would of course still behave properly. 
}%Bem
%-----------------------------------------------
\section{Moving clusters - motion consistency} \label{sec:motionconsistency}
As we have stated already, in the $\mathbb{R} ^n$ it is actually impossible to move clusters in such a way as to increase distances to all the other elements of all the other clusters (see Theorem \ref{thm:noouterConsistency}).
However, we shall ask ourselves if we may possibly move away clusters as whole, via increasing the distance between cluster centers and not overlapping cluster regions, which, in case of $k$-means, represent Voronoi-regions. 

\begin{ax}
A clustering method conforms to \emph{motion consistency}, 
if it returns the same clustering when the distances of cluster centers are increased by moving each point of a cluster by the same vector without leading to overlapping of the convex regions of clusters.
\end{ax}

Let us concentrate on the $k$-means case and let us look at two neighboring clusters. 
The Voronoi regions, associated with $k$-means clusters,  are in fact polyhedrons, such that the "outer" polyhedrons (at least one of them) can be moved away from the rest without overlapping any other region.

So is such an operation on regions permissible without changing the cluster structure? 
A closer look at the issue tells us that it is not. 
As $k$-means terminates, the neighboring clusters' polyhedrons touch each other via a hyperplane such that the straight line connecting centers of the clusters is orthogonal to this hyperplane. 
This causes that points on the one side of this hyperplane lie more closely to the one center, and on the other to the other one. 
But if we move  the clusters in such a way that both touch each other along the same hyperplane, then it happens that  some points within the first cluster will become closer to the center of the other cluster and vice versa.%
\footnote{This is by the way the nice trick behind the claim in 
\cite{Ackerman:2014nips} that incremental $k$-means does not identify perfectly separated clusters. 
Clusters in $k$-means are not the points, they are polyhedrons, contrary to the assumptions in \cite{Ackerman:2014nips}. 
} 
So moving the clusters generally will change their structure (points switch clusters) unless the points lie actually not within the polyhedrons but rather within "paraboloids" with appropriate equations. 
Then moving along the border hyperplane will not change cluster membership (locally). 
But the intrinsic cluster borders are now "paraboloids". What would happen if we relocate the clusters allowing for touching along the "paraboloids"? The problem will occur again.

Hence the question can be raised: What shape should have the $k$-means clusters in order to be (locally) immune to movement of whole clusters?

Let us consider the problem of susceptibility to class membership change 
within a 2D plane containing the two cluster centers.
Let the one cluster center be located at a point (0,0) in this plane 
and the other at $(2x_0,2y_0)$. Let further the border of the first cluster be characterized by a (symmetric) function $f(x)$ and le the shape of the border of the other one $g(x)$  
be the same, but properly rotated: $g(x)=2y_0-f(x-2x_0)$ so that the cluster center is in the same.
Let both have a touching point (we excluded already a straight line and want to have convex smooth borders).
From the symmetry conditions one easily sees that the touching point must be 
$(x_0, y_0)$. 
As this point lies on the surface of $f()$, $y_0=f(x_0)$ must hold. 
For any point $(x,f(x)$ of the border of the first cluster with center $(0,0)$ the following must hold:
\begin{equation}
(x-2x_0)^2+(f(x)-2f(x_0))^2-x^2-f^2(x)\ge 0
\end{equation}

That is
$$-2x_0(2x-2x_0)-2f(x_0)\left(2f(x)-2f(x_0)\right) \ge 0$$ 
%$$-2f(x_0)\left(2f(x)-2f(x_0)\right) \ge  2x_0(2x-2x_0)$$ 
$$-f(x_0)\left(f(x)-f(x_0)\right) \ge  x_0(x-x_0)$$ 
Let us consider only positions of the center of the second cluster below the $X$ axis. In this case $f(x_0)< 0$. 
Further let us concentrate on $x$ lower than $x_0$. We get 
$$-\frac{ f(x)-f(x_0)}{x-x_0}  \ge  \frac{x_0}{-f(x_0)}$$ 
In the limit, when $x$ approaches $x_0$. 
$$-f'(x_0)   \ge  \frac{x_0}{-f(x_0)}$$ 
Now turn to   $x$ greater than $x_0$. We get 
$$-\frac{ f(x)-f(x_0)}{x-x_0}  \le  \frac{x_0}{-f(x_0)}$$ 
In the limit, when $x$ approaches $x_0$. 
$$-f'(x_0)   \le  \frac{x_0}{-f(x_0)}$$ 
This implies 
\begin{equation}
-f'(x_0)   =  \frac{-1}{ \frac{f(x_0)}{x_0}}
\end{equation}

Note that $\frac{f(x_0)}{x_0}$ is the directional tangent of the straight line connecting both cluster centers.
As well as it is the directional tangent of the line connecting the center of the first cluster to its surface. 
$f'(x_0)$ is the tangential of the borderline of the first cluster at the touching point of both clusters. 
The equation above means both are orthogonal. 
But this property implies that  $f(x)$ must be definition (of a part) a circle centered at $(0,0)$. 
As the same reasoning applies at any touching point of the clusters, a $k$-means cluster would have to be (hyper)ball-shaped in order to allow the movement of the clusters without elements switching cluster membership.

The tendency of $k$-means to recognize best ball-shaped clusters has been known long ago, but we are not aware of presenting such an argument for this tendency.

It has to be stated however that clusters, even if enclosed in a ball-shaped region, need to be separated sufficiently to be properly recognized. 
Let us consider, under which circumstances a cluster $C_1$ of radius $r_1$ containing $n_1$ elements 
would take over $n_{21}$ elements (subcluster  $C_{21}$) of a cluster  $C_2$ of radius $r_2$ of cardinality $n_{2}$. 
Let $n_{22}=n_2-n_{21}$ be the number of the remaining elements (subcluster  $C_{22}$ of the second cluster.
Let the enclosing balls of both clusters be separated by the distance (gap) $g$. 
Let us consider the worst case that is that 
the center of the $C_{21}$ subcluster lies on a straight line segment connecting both cluster centers.
The center of the remaining $C_{22}$ subcluster would lie on the same line but on the other side of the second cluster center. 
Let $r_{21}, r_{22}$ be distances of centers of $n_{21}$ and $n_{22}$ from the center of the second cluster. 
The relations
$$n_{21}\cdot r_{21}= n_{22}\cdot r_{22}, \ r_{21}\le r_2, \ r_{22}\le r_2$$
must hold. 
Let denote with $SSC(C)$ the sum of squared distances of elements of the set $C$ to the center of this set.

So in order for the clusters to be stable 
$$SSC(C_1)+SSC(C_2) \le SSC(C_1\cup C_{21})+SSC(C_{22}) $$
must hold. But
$$SSC(C_2)=SSC(C_{21})+SSC(C_{22})+n_{21}\cdot r_{21}^2+n_{22}\cdot r_{22}^2$$
$$SSC(C_1\cup C_{21})=SSC(C_1)+SSC(C_{21})+\frac{n_1 n_{21}}{n_1+ n_{21}}(r_1+r_2+g-r_{21})^2$$

Hence
$$SSC(C_1)+SSC(C_{21})+SSC(C_{22})+n_{21}\cdot r_{21}^2+n_{22}\cdot r_{22}^2
$$ $$
 \le SSC(C_1)+SSC(C_{21})+\frac{n_1 n_{21}}{n_1+ n_{21}}(r_1+r_2+g-r_{21})^2+SSC(C_{22}) $$

$$    n_{21}\cdot r_{21}^2+n_{22}\cdot r_{22}^2
 \le    \frac{n_1 n_{21}}{n_1+ n_{21}}(r_1+r_2+g-r_{21})^2  $$
 
$$   \frac{ n_{21}\cdot r_{21}^2+n_{22}\cdot r_{22}^2 }
{   \frac{n_1 n_{21}}{n_1+ n_{21}} }
 \le  (r_1+r_2+g-r_{21})^2  $$

$$ \sqrt{  \frac{ n_{21}\cdot r_{21}^2+n_{22}\cdot r_{22}^2 }
{   \frac{n_1 n_{21}}{n_1+ n_{21}} }}
 \le   r_1+r_2+g-r_{21}   $$

$$ \sqrt{  \frac{ n_{21}\cdot r_{21}^2+n_{22}\cdot r_{22}^2 }
{   \frac{n_1 n_{21}}{n_1+ n_{21}} }}
 -r_1-r_2+r_{21} 
 \le   g   $$

$$ \sqrt{  \frac{ n_{21}\cdot r_{21}^2+n_{21}\cdot r_{21}\cdot r_{22}  }
{   \frac1{1/n_1+ 1/n_{21}} }}
 -r_1-r_2+r_{21} 
 \le   g   $$

$$ \sqrt{  ( n_{21}\cdot r_{21}^2+n_{21}\cdot r_{21}\cdot r_{22}  )
 (1/n_1+ 1/n_{21} )}
  -r_1-r_2+r_{21} 
 \le   g   $$
$$ \sqrt{  (   r_{21}^2+  r_{21}\cdot r_{22}  )
 (n_{21}/n_1+ 1  )}
  -r_1-r_2+r_{21} 
 \le   g   $$

As $r_{22}= \frac{r_{21} n_{21}}{n_2-n_{21}}$

$$ \sqrt{  (   r_{21}^2+  r_{21}\cdot \frac{r_{21} n_{21}}{n_2-n_{21}}  )
 (n_{21}/n_1+ 1  )}
  -r_1-r_2+r_{21} 
 \le   g   $$
$$ r_{21}  \sqrt{  (  1 +    \frac{ n_{21}}{n_2-n_{21}}  )
 (n_{21}/n_1+ 1  )}
  -r_1-r_2+r_{21} 
 \le   g   $$
 
$$ r_{21}  \sqrt{     \frac{ n_{2}}{n_2-n_{21}}   \frac{n_1+n_{21}}{n_1}}
  -r_1-r_2+r_{21} 
 \le   g   $$
$$ r_{21}  \sqrt{     \frac{ n_2}{n_1}   \frac{n_1+n_{21}}{n_2-n_{21}}}
  -r_1-r_2+r_{21} 
 \le   g   $$
Let us consider the worst case when the elements to be taken over are at the "edge" of the cluster region ($r_{21}=r_2$). 
Then
$$ r_2  \sqrt{     \frac{ n_2}{n_1}   \frac{n_1+n_{21}}{n_2-n_{21}}}
  -r_1 
 \le   g   $$
The lower limit on $g$ will grow with $n_{21}$, but  $n_{21}\le 0.5 n_2$, because otherwise $r_{22}$ would exceed $r_2$.
Hence in the worst case 
$$ r_2  \sqrt{     \frac{ n_2}{n_1}   \frac{n_1+n_2/2}{n_2/2}}
  -r_1 
 \le   g   $$
\begin{equation}
 r_2  \sqrt{     2 (1+0.5n_2/n_1)  }
  -r_1 
 \le   g   
\end{equation}

In case of clusters with equal sizes and equal radius this amounts to
$$g\ge r_1(\sqrt{3}-1)\approx 0.7r_1 $$

So we can conclude 
\begin{theorem}{}
$k$-means algorithm conforms (locally) to Motion Consistency axiom. 
\end{theorem}

Note that the motion consistency axiom is a substitute for outer-consistency which is impossible continuously in Euclidean space. 
It is to be underlined that we speak here about local optimum of $k$-means. With the abovementioned gap size the global $k$-means minimum may lie elsewhere, in a clustering possibly without gaps. 
Also the motion consistency transformation preserves as local minimum the partition it is applied to. Other local minima and global minimum can change.    

Note that compared to inner-consistency the centric consistency is quite rigid. 
And so is motion consistency compared to outer-consistency. 

In two subsequent sections we will investigate if the rigidity of these transformations can be weakened under appropriate width of the gaps, and if we can grant these properties under global minimum. 
In particular, 
we shall study, 
how well the clusters need to be separated 
so that it is enough to find the global optimum.

%-----------------------------------------------
\section{Cluster separation versus Kleinberg's axioms of $k$-richness, consistency, scaling invariance} 
\label{sec:kmeansperfectballclustering}

\subsection{Perfect ball clusterings}

The problem with $k$-means (-random and ++) is the discrepancy between the theoretically optimized function ()$k$-means-ideal) 
and the actual approximation of this value. 
It appears to be problematic even for "well-separated" clusters.

%what impact the well-separability as characterized in the %preceding section will have on the attainability of the 
%clustering goal.

First let us point to the fact that "well-separatedness" may keep the algorithm in a local minimum. 

It is commonly assumed that a good initialization of a $k$-means clustering is one 
where the seeds hit different clusters.
It is well known, that under some circumstances the $k$-means does not recover from poor initialization and as a consequence a natural cluster may be split even for "well-separated" data. 

But hitting each cluster may be not sufficient as neighboring clusters may be able to shift the cluster center away from its cluster. 

Hence let us investigate what kind of well-separability would be sufficient to ensure that once clusters are hit by one seed each, would never loose the cluster center. 

Let us 
investigate the working hypothesis 
 that two clusters are well separated if we can draw  a ball of some radius $\rho$ around true cluster center of each of them and there is a gap between these balls. 
We claim that 
\begin{theorem}{}
If the distance between any the cluster centers $A,B$ is at least $4\rho_{AB}$, where $\rho_{AB}$ is the radius of a ball centered 
at $A$ and enclosing its cluster (that is cluster lies in the interior of the ball)
and it also is the radius of a ball centered 
at $B$ and enclosing its cluster% 
,  then once each cluster is seeded the clusters cannot loose their cluster elements for each other during  $k$-means-random and $k$-means++ iterations.
\end{theorem}  

Before starting the proof, let us introduce related definitions.

\begin{definition}
We shall say that clusters centered at $A$ and $B$ and enclosed  in balls centered at $A,B$ and with radius $\rho_{AB}$ each are \emph{nicely ball-separated},
if the distance between  $A,B$  is at least $4\rho_{AB}$.
  If all pairs of clusters are nicely ball separated with the same ball radius, then we shall say that they are  \emph{perfectly ball-separated}. 
\end{definition} 

\begin{proof}
For the illustration of the proof  see Figure \ref{fig:fourradii}.

\begin{figure}
\centering
\includegraphics[width=0.8\textwidth]{\figaddr{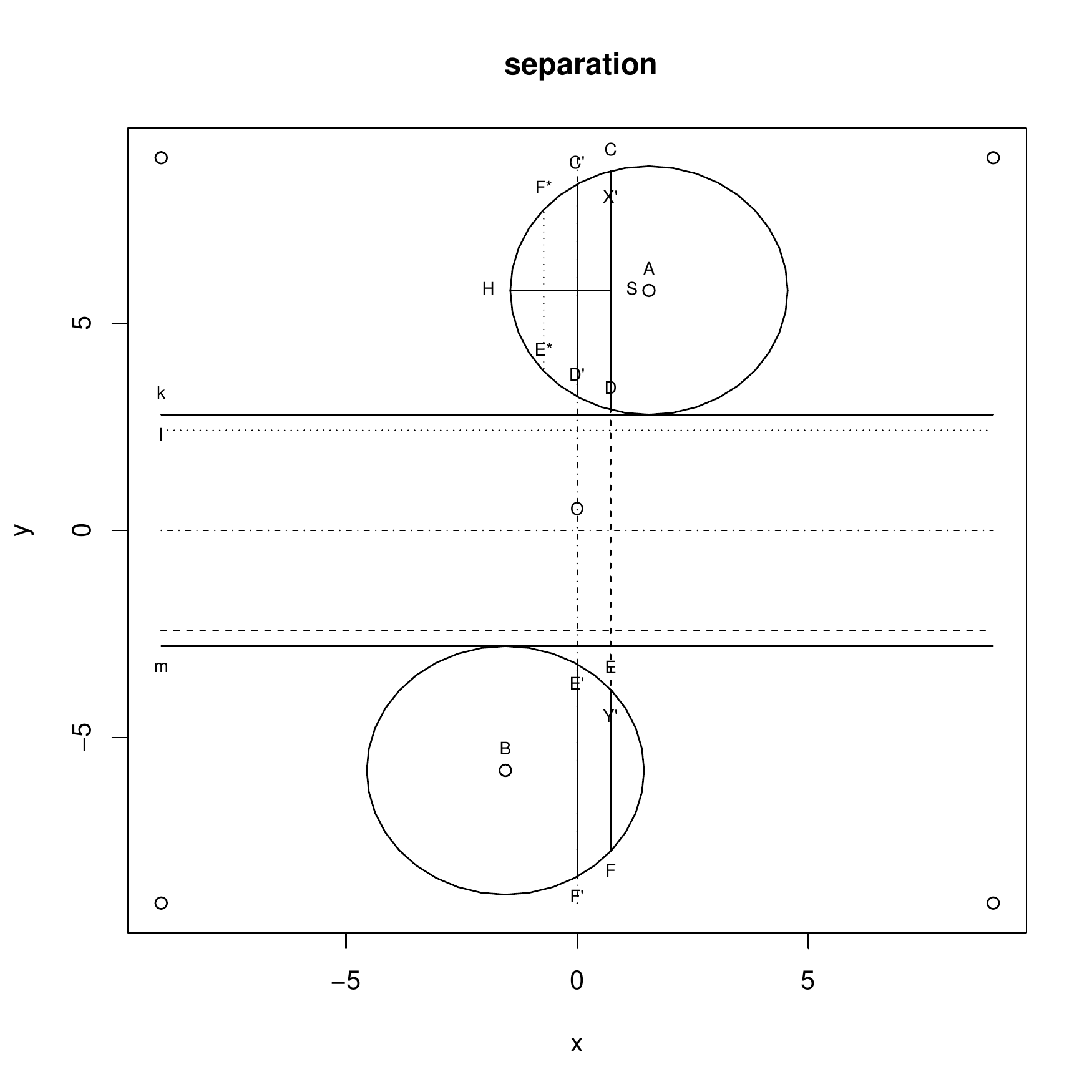}}  %
\caption{An illustrative figure for proof of 4 radius distance ensuring good separability.  
}\label{fig:fourradii}
\end{figure}

Consider the two points $A,B$ being the two ball centers  and two points, $X,Y$, one being in each ball (presumably the cluster centers at some stage of the $k$-means algorithm). To represent their distances faithfully, we need at most a 3D space. 

Let us consider the plane established by the line $AB$ and parallel to the line $XY$. 
Let $X'$ and $Y'$ be projections of $X,Y$ onto this plane. 
Now let us establish that the hyperplane $\pi$ orthogonal to $X,Y$, and passing through the middle of the line segment $XY$, that is the hyperplane containing the boundary between clusters centered at $X$ and $Y$ does not cut any of the balls centered at $A$ and $B$. 
This hyperplane will be orthogonal to the plane of the Figure  \ref{fig:fourradii} and so it will manifest itself as an intersecting line $l$ that should not cross circles around $A$ and $B$, being projections of the respective balls.  Let us draw two solid lines $k,m$ between circles $O(A,\rho)$ and $O(B,\rho)$ tangential to each of them. Line $l$ should lie between these lines, in which case the cluster center will not jump to the other ball. 

Let the line $X'Y'$ intersect with the circles $O(A,\rho)$ and $O(B,\rho)$ at points $C,D,E,F$ as in the figure. 

It is obvious that the line $l$ would get closer to circle A, if the points X', Y' would lie closer to C and E, or closer to circle $B$ if they would be closer to $D$ and $F$. 

Therefore, to show that the line $l$ does not cut the circle $O(A,\rho)$ it is sufficient to consider $X'=C$ and $Y'=E$. (The case with ball $Ball(B,\rho)$ is symmetrical).

Let $O$ be the center of the line segment $AB$. Let us draw through this point a line parallel to $CE$ that cuts the circles at points $C', D', E'$ and $F'$. 
Now notice that centric symmetry through point $O$ transforms the circles $O(A,\rho)$,$O(B,\rho)$ into one another, and point $C'$  into $F'$ and $D'$ into $E'$. Let $E^*$ and $F^*$ be images of points $E$ and $F$ under this symmetry. 

In order for the line $l$ to lie between $m$ and $k$, the middle point of the line segment $CE$ shall lie between these lines.
 
Let us introduce a planar coordinate system centered at $O$ with $\mathcal{X}$ axis parallel to  lines $m,k$, such that $A$ has both coordinates non-negative, and $B$ non-positive. 
Let us denote with  $\alpha$ the angle between the lines $AB$ and $k$. 
As we assume that the distance between $A$ and $B$ equals $4\rho$, then the distance between lines $k$ and $m$ amounts to $2\rho(2\sin(\alpha)-1)$. 
Hence the  $\mathcal{Y}$ coordinate of line $k$ equals $\rho (2 \sin(\alpha)-1)$.

So the  $\mathcal{Y}$ coordinate of the center of line segment $CE$ shall be not higher than this. 
Let us express this in vector calculus:
$$4(y_{OC}+y_{OE})/2 \le  \rho (2\sin(\alpha)-1)$$

Note, however that 
$$y_{OC}+y_{OE}=y_{OA}+y_{AC}+y_{OB}+y_{BE}= y_{AC}+y_{BE} =  y_{AC}-y_{AE^*} 
=y_{AC}+y_{E^*A}$$ 

So let us examine the circle with center at A. 
Note that the lines $CD$ and $E^*F^*$ are at the same distance from  the line C'D'. Note also that the absolute values  of direction coefficients of tangentials of circle A at C' and D' are identical.
The more distant these lines are, as line $CD$ gets closer to $A$, the $y_{AC}$ gets bigger, and $y_{E^*A}$ becomes smaller. 
But from the properties of the circle we see that 
$y_{AC}$ increases at a decreasing rate, while  $y_{E^*A}$ decreases at an increasing rate. 
So the sum $y_{AC}+y_{E^*A}$  has the biggest value when $C$ is identical with $C'$ and we need hence to prove only that 
$$(y_{AC'}+y_{D'A} )/2=y_{AC'} \le  \rho(2\sin(\alpha)-1)$$

Let $M$ denote the middle point of the line segment $C'D'$. 
As point $A$ has the coordinates $(2\rho \cos(\alpha), 2\rho \sin(\alpha))$,
the point $M$ is at distance of $2\rho\cos(\alpha)$ from $A$. 
But $C'M^2= \rho^2-(2\rho\cos(\alpha))^2 $.  

So we need to show that 
$$\rho^2-(2\rho\cos(\alpha))^2 \le  ( \rho (2\sin(\alpha)-1))^2$$
In fact we get from the above 
$$\rho^2-4\rho^2\cos(\alpha)^2 \le  \rho^2(2\sin(\alpha)-1)^2$$
Dividing by $\rho^2$
$$1-4 \cos(\alpha)^2 \le   (2\sin(\alpha)-1)^2$$
$$1-4 \cos(\alpha)^2 \le   4\sin(\alpha)^2 -4\sin(\alpha) +1$$
Adding $4 \cos(\alpha)^2$ to both sides and subtracting 1 we get 
$$0  \le   4  -4\sin(\alpha)  $$
Dividing by 4
$$0  \le   1  - \sin(\alpha)  $$
which is a known trigonometric  relation. 

This means in practice that whatever point from the one and the other cluster is picked randomly as cluster center, then the Voronoi tessellation of the space will contain only points from a single cluster. 
%
%This ends the proof.    
\end{proof}

Let us discuss at this point a bit the notions of "perfect separation" as introduced in \cite{Ackerman:2014nips}. 
In their Theorem 4.4. Ackerman and Dasgupta  \cite{Ackerman:2014nips} show that the incremental $k$-means algorithm, as introduced in  their Algorithm 2.2 , is not able to cluster correctly data that is "perfectly clusterable" (their Definition 4.1). 
However,  it is obvious that under the "perfect-ball-separation"   as introduced here their  incremental $k$-means algorithm\footnote{ 
Algorithm 2.2. (Sequential $k$-means) should be slightly modified:
\\ Set $T = (t_1,\dots,t_k)$ to the first $k$ data points 
\\Initialize the counts $n_1, n_2,\dots, n_k$ to 1 
\\Repeat:
\\ \hspace*{0.5cm}  Acquire the next example, $t_{k+1}$. Set $n_{k+1}=1$ 
\\ \hspace*{0.5cm}  
If $t_i$ is the closest center to $t_j$, $j\ne i$, 
\\ \hspace*{0.5cm} \hspace*{0.5cm}  Replace $t_i=(t_in_i+t_jn_j)/(n_i+n_j)$, thereafter $n_i=n_i+n_j$
\\ \hspace*{0.5cm} \hspace*{0.5cm}  If $j\ne k+1$ then replace $t_j=t_{k+1}$, $n_j=n_{k+1}$. 
}   will discover the structure of the clusters.
The reason is as follows. Perfect ball separation ensures that there exists an $r$ of the enclosing ball such that the distance between any two points within the same ball is lower than $2r$ and between them is bigger than $2r$. So whenever Ackerman's incremental $k$-mean merges two points, they are the points of the same ball. And upon merging the resulting point lies again within the ball. 
So we can conclude

\begin{theorem}{} The incremental $k$-means algorithm will discover the structure of perfect-ball-clustering.
\end{theorem}

Let us note at this point, however, that the incremental $k$-means  algorithm  would return only a set of cluster centers without stating whether or not we got a perfect ball clustering.  But it  is important to know if this is the case because otherwise the resulting set of cluster centers may be arbitrary and under unfavorable conditions it may not correspond to a local minimum of $k$-means ideal at all.  
However, if we are allowed to inspect the data for the second time, such an information  can be provided.\footnote{
One shall proceed as follows on the second pass: 
\\ Let $T = (t_1,\dots,t_k)$ be the resulting set of cluster centers from the first pass.  
\\Initialize the furthest neighbors   $f_1, f_2,\dots, f_k$ with  $t_1, t_2\dots,t_k$ respectively.
\\Repeat:
\\ \hspace*{0.5cm}   Acquire the next example, $x$. 
\\ \hspace*{0.5cm}  
If $t_i$ is the closest center to x,    
\\ \hspace*{0.5cm} \hspace*{0.5cm}  
if $x$ is further away from $t_i$ than $f_i$ then replace $f_i$ with $X$. 
\\Compute distances between corresponding $t_i$ and $f_i$, pick the highest one, compute distances between each pair $t_i,t_j$ and pick the lowest one. 
If the latter is 4 times or more higher than the former one, we got a perfect ball clustering.  
}
A second pass for other algorithms from their section 2 would not yield such a decision.

The difference between our and  their definition of well separatedness lies essentially in their understanding of clustering 
as a partition of data points, while in fact the user is interested in partition of the sample space (in terms of learnability theory of Valiant). 
Hence also a further correction of Kleinberg's axiomatic framework should take this flaw into account. 

Let us further turn to their concept of "nice clustering" (their Def. 3.1.). As they show in their  Theorem 3.8., nice clustering cannot be discovered by an incremental algorithm with memory linear in $k$. 
In Theorem 5.3 they show that their incremental algorithm 5.2. with up to $2^{k-1}$ cluster centers can detect points from each of nice clusters. Again it is not the incremental $k$-means that may achieve it (see their theorem 5.7.) even under "nice convex" conditions. 
Surely our concept of nice-ball-clustering is even more restrictive than their "nice-convex" clustering. 
But if we upgrade their CANDIDATES(S) algorithm so that it behaves like $k$-means that is if we replace the step "Moving bottom-up, assign each internal node the data point in one of its children" with the assignment to the internal node the properly weighted (with respective cardinalities of leaves) average, then the algorithm 5.2. upgraded to incremental $k$-means version will in fact return the "refinement" of the clustering.% 
\footnote{The modified algorithm would look like:
\\CANDIDATES(S) 
\\Run single linkage on S to get a tree (distances between $t$ are used) 
\\Assign each leaf node the corresponding data point 
\\Moving bottom-up, assign each internal node the 
$n=n_L+n_R$, $t=(t_Ln_L+t_Rn_R)/n$, L,R indicating left and right child.
\\ Return all points at distance $< k$ from the root
}
What is more, if we are allowed to have a second pass through the data, then we can pick out the real cluster centers using an upgrade of the CANDIDATES(S) algorithm. The other algorithms considered in their section 5 will fail to do this on the second pass through the data (because of deviations from true cluster center).%      
\footnote{The needed algorithm would look like:
\\Take the tree from the first pass with $t$ values assigned in the first pass. Assign each node an $f$ value identical to $t$ value.  
\\Repeat:
\\ \hspace*{0.5cm}  Acquire the next example, $x$. 
\\ \hspace*{0.5cm} Find the leaf with $t$ closest to $x$.
\\ \hspace*{0.5cm}  Update its $f$ value with $x$ if it is further away from $t$ than $f$.  
\\ \hspace*{0.5cm} Pass $x$ to all  direct and indirect ancestors (internal) nodes  of this leaf 
\\ \hspace*{0.7cm} where in each of these nodes  update its $f$ value with $x$ 
\\ \hspace*{0.7cm} if it is further away from $t$ than $f$.  
\\ For each cut of the tree engaging exactly $k$ nodes check if the nice ball clustering condition is fulfilled for balls rooted at $t$ with radii $\|f-t\|$. 
\\ If for any such a cut the condition holds, the nice ball clustering is found, otherwise it is not.
}

Let us discuss Kleinberg axioms for perfectly ball-separated clusters. 
It is clear that if $k$-means random or $k$-means++ gets initiated in such a way that each initial cluster center hits a different cluster, then upon subsequent steps the cluster centers will not leave the clusters. One gets stuck in a minimum, not necessarily the global one. Let us understand the Kleinberg's phrase "the function returns the clustering" as one of possible (local) minima of the clustering functions. 
$k$-richness is trivially granted if we restrict ourselves to perfectly-ball-separated clusters. 
If one performs the scaling on perfectly ball separated clusters, they will remain perfectly ball separated (scale invariance). 
If one applies moving-consistency transformation   (keeping inner distances and relative positions to the cluster fixed coordinate systems, not bothering about distances between elements in distinct clusters) then the clusters will remain perfectly ball separated.
Also a centric-consistency transformation will keep the partition in the realm of perfect-ball-clusterings. 
Hence 
\begin{theorem}{}
$k$-means, if restricted to perfectly ball separated clusterings, conforms (locally) to $k$-richness, scale-invariance, motion consistency and centric consistency. 
\end{theorem} 

But we gain still something more.

\begin{ax}
A clustering method conforms to \emph{inner cluster consistency}, 
if it returns the same clustering when the positions distances of cluster centers are kept, while the distances within each cluster are decreased 
\end{ax}

Note that inner cluster consistency, as compared to inner-consistency, is less restrictive as one does not need to care about distances between elements in different clusters. 

If one performs an inner cluster consistency transformation, the clusters will remain perfectly ball separated (a kind of inner-consistency). So we get

\begin{theorem}{}\label{thm:perfect3axiomsOnek}
$k$-means, if restricted to perfectly ball separated clusterings, conforms (locally) to $k$-richness,  scale-invariance,  motion consistency and inner cluster  consistency. 
\end{theorem}

As with perfect clustering (see  \cite{Ackerman:2014nips}), also if there exists a perfect ball clustering into $k$ clusters, then there exists only one such clustering. 
Regrettably, via an  inner cluster  consistency  transformation for a data set with perfect ball $k$-clustering one can obtain a data set for which perfect ball $k+l$ clustering is possible for an $l>0$ even if it was impossible prior to transformation. Albeit only nested clusters will emerge.  
If one would choose to have the largest number of clusters with cluster cardinality $\ge 2$, then  one can speak about "refinement inner cluster consistency", with the direction of the refinement towards smaller clusters. 

Similarly, if we move away cluster centers (motion consistency transformation), we can obtain a new perfect ball $k-l$ clustering even if it did not exist prior to the transformation. Again, cluster nesting occurs. So if one would choose to have the lowest number of clusters $k\ge 2$, then  one can speak about "refinement motion consistency", with the direction of the refinement towards larger clusters. 

The very same statements can be made about Kleinberg's axioms for nice ball clustering and $k$-means. 
Except that for a given $k$ the clustering, if exists, does not need to be unique.

Last not least let us make the remark that  even if the perfect-ball-clustering exists, it does not need to be the global optimum of $k$-means ideal, because of possible different cardinalities of these clusters. So in fact the global optimum may be one that is imperfect, even if the perfect clustering exists. 

But let us state one more thing.
Assume that we allow for a broader range of $k$ values with $k$-means. Note that with centric consistency, contrary to inner cluster consistency transform, no new perfect ball structures will emerge.  Therefore: 

\begin{theorem}{}\label{thm:perfect3axiomsManyk}
$k$-means, with $k$ ranging over a set of values, if we assume that it returns the    perfectly/nicely ball separated clusterings for the largest possible $k$ (excluding too small clusters, we call it \emph{max-$k$-means algorithm}), then it conforms (locally) to richness,  scale-invariance,  motion consistency and centric consistency. 
\end{theorem}

\subsection{Core-based clusterings}

But as we have seen in the previous section, for various purposes the distance between the balls enclosing clusters may be smaller. 
So let us  discuss what happens if the distances (gaps) between clusters are smaller.

We claim that 
\begin{theorem}{}
Let $A,B$ be cluster centers. 
Let $\rho_{AB}$ be the radius of a ball centered 
at $A$ and enclosing its cluster 
and it also is the radius of a ball centered 
at $B$ and enclosing its cluster.  
If the distance between   the cluster centers $A,B$
amounts to $2\rho_{AB}+g$, $g>0$ ($g$ being the "gap" between clusters),  
if we pick any two points, $X$ from the cluster of $A$ and $Y$ from the cluster of $B$, then the new clusters will preserve the balls centered at $A$ and $B$ of radius $g/2$ (called subsequently "cores") each ($X$ the core of $A$, $Y$ the core of $B$).  
\end{theorem}  

\begin{definition}
If the gap between each pair of clusters fulfills the condition of the above theorem, then we say that we have core-clustering. 
\end{definition}

\begin{proof}
For the illustration of the proof see Figure \ref{fig:threeballs}.

\begin{figure}
\centering
\includegraphics[width=0.8\textwidth]{\figaddr{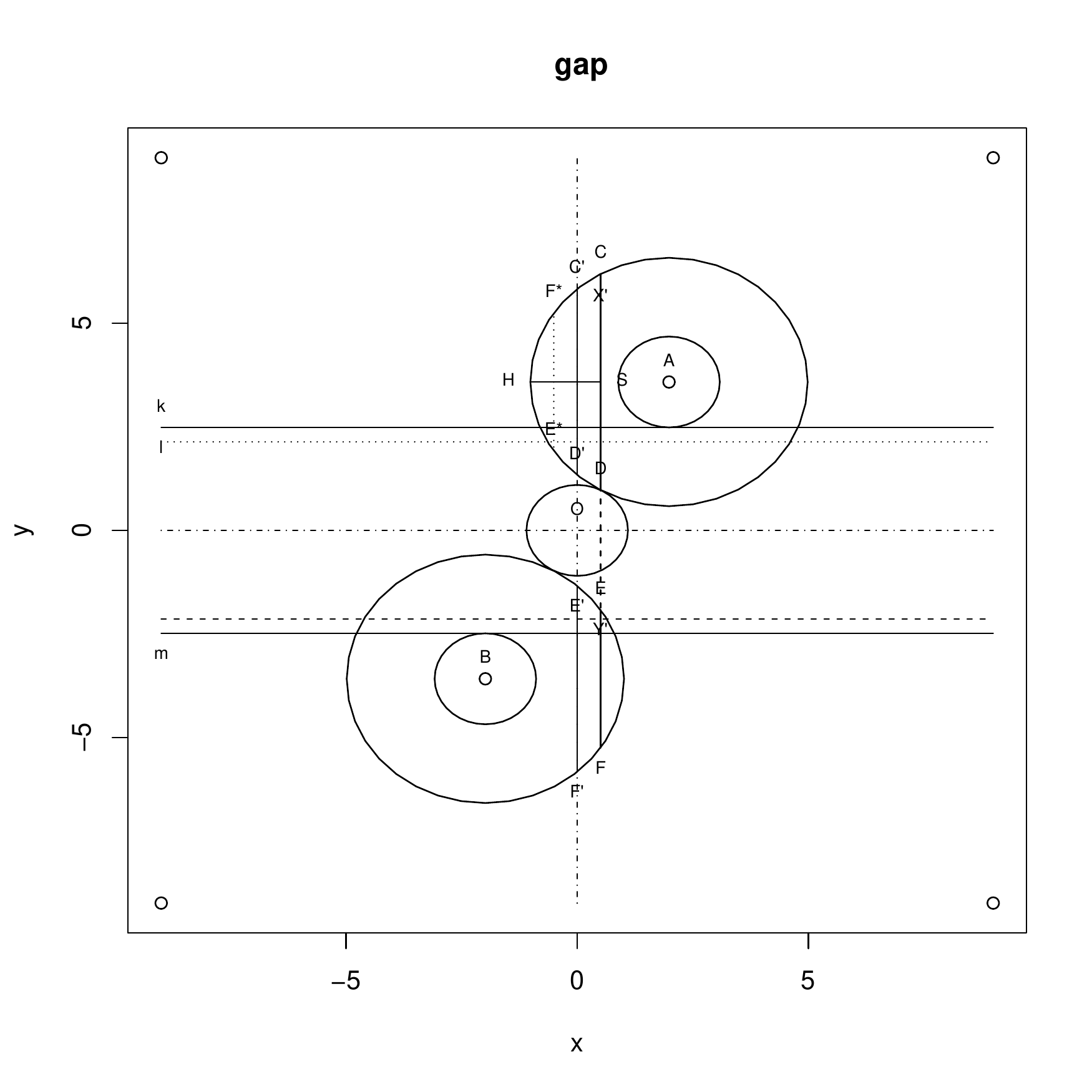}}  %
\caption{An illustrative figure for proof of the core preservation under a gap between cluster enclosing balls. 
}\label{fig:threeballs}
\end{figure}
The proof does not differ too much from the previous one and in fact the previous theorem is a special case when $g=2\rho$.

Consider the two points $A,B$ being the two centers of double balls.
The inner call represents the core of radius $g/2$, the outer ball of radius $\rho$ ($\rho=\rho_{AB}$). 
Consider  two points, $X,Y$, one being in each outer ball  (presumably the cluster centers at some stage of the $k$-means algorithm). To represent their distances faithfully, we need at most a 3D space. 

Let us consider the plane established by the line $AB$ and parallel to the line $XY$. 
Let $X'$ and $Y'$ be projections of $X,Y$ onto this plane. 
Now let us establish that the hyperplane $\pi$ orthogonal to $X,Y$, and passing through the middle of the line segment $XY$, that is the hyperplane containing the boundary between clusters centered at $X$ and $Y$ does not cut any of the balls centered at $A$ and $B$. 
This hyperplane will be orthogonal to the plane of the Figure  \ref{fig:threeballs} and so it will manifest itself as an intersecting line $l$ that should not cross inner circles around $A$ and $B$, being projections of the respective balls.  Let us draw two solid lines $k,m$ between circles $O(A,g/2)$ and $O(B,g/2)$ tangential to each of them. Line $l$ should lie between these lines, in which case the cluster center will not jump to the other ball. 

Let the line $X'Y'$ intersect with the circles $O(A,\rho)$ and $O(B,\rho)$ at points $C,D,E,F$ as in the figure. 

It is obvious that the line $l$ would get closer to circle $A$, if the points $X', Y'$ would lie closer to $C$ and $E$, or closer to circle $B$ if they would be closer to $D$ and $F$. 

Therefore, to show that it does not cut the circle $O(A,g/2)$ it is sufficient to consider $X'=C$ and $Y'=E$. (The case with ball $Ball(B,g/2)$ is symmetrical).

Let $O$ be the center of the line segment $AB$. Let us draw through this point a line parallel to $CE$ that cuts the circles at points $C', D', E'$ and $F'$. 
Now notice that centric symmetry through point $O$ transforms the circles $O(A,\rho)$,$O(B,\rho)$ into one another, and point $C' in F'$ and $D' in E'$. Let $E^*$ and $F^*$ be images of points $E$ and $F$ under this symmetry. 

In order for the line $l$ to lie between $m$ and $k$, the middle point of the line segment $CE$ shall lie between these lines.
 
Let us introduce a planar coordinate system centered at $O$ with $\mathcal{X}$ axis parallel to lines $m,k$, such that $A$ has both coordinates non-negative, and $B$ non-positive. 
Let us denote with  $\alpha$ the angle between the lines $AB$ and $k$. 
As we assume that the distance between $A$ and $B$ equals $2 \rho+g$, then the distance between lines $k$ and $m$ amounts to $ 
2((\rho+g/2)\sin(\alpha)-g/2)$. 
Hence the  $\mathcal{Y}$ coordinate of line $k$ equals $((\rho+g/2)\sin(\alpha)-g/2)$.

So the  $\mathcal{Y}$ coordinate of the center of line segment $CE$ shall be not higher than this. 
Let us express this in vector calculus:
$$4(y_{OC}+y_{OE})/2 \le ((\rho+g/2)\sin(\alpha)-g/2)$$

Note, however that 
$$y_{OC}+y_{OE}=y_{OA}+y_{AC}+y_{OB}+y_{BE}= y_{AC}+y_{BE} =  y_{AC}-y_{AE^*} 
=y_{AC}+y_{E^*A}$$ 

So let us examine the circle with center at A. 
Note that the lines $CD$ and $E^*F^*$ are at the same distance from  the line C'D'. Note also that the absolute values  of direction coefficients of tangentials of circle A at C' and D' are identical.
The more distant these lines are, as line $CD$ gets closer to $A$, the $y_{AC}$ gets bigger, and $y_{E^*A}$ becomes smaller. 
But from the properties of the circle we see that 
$y_{AC}$ increases at a decreasing rate, while  $y_{E^*A}$ decreases at an increasing rate. 
So the sum $y_{AC}+y_{E^*A}$  has the biggest value when $C$ is identical with $C'$ and we need hence to prove only that 
$$(y_{AC'}+y_{D'A} )/2=y_{AC'} \le  ((\rho+g/2)\sin(\alpha)-g/2)$$

Let $M$ denote the middle point of the line segment $C'D'$. 
As point $A$ has the coordinates $((\rho+g/2) \cos(\alpha), (\rho+g/2) \sin(\alpha))$,
the point $M$ is at distance of $(\rho+g/2) \cos(\alpha)$ from $A$. 
But $C'M^2= \rho^2-((\rho+g/2) \cos(\alpha))^2 $.  

So we need to show that 
$$\rho^2-((\rho+g/2) \cos(\alpha))^2 \le ((\rho+g/2)\sin(\alpha)-g/2)^2$$
In fact we get from the above 
$$\rho^2-((\rho+g/2) \cos(\alpha))^2 \le 
((\rho+g/2)\sin(\alpha))^2
+(g/2)^2
-2(\rho+g/2)(g/2)\sin(\alpha)$$
$$\rho^2  \le  (\rho+g/2)^2+(g/2)^2
-2(\rho+g/2)(g/2)\sin(\alpha)$$
$$0  \le  2(\rho+g/2)(g/2)
-2(\rho+g/2)(g/2)\sin(\alpha)$$
$$0  \le  2(\rho+g/2)(g/2)(1-\sin(\alpha))$$
which is obviously true, as $\sin$ never exceeds 1. 
%This ends the proof of the theorem.
\end{proof}

But we have still to ask what is the gain of having an untouched core. 

Consider a cluster $C$ of a clustering $\mathcal{C}$ and let it have the share $p$ of its mass at its core of radius $(g/2)$ and the remaining $1-p$  in the ball of radius $\rho$ (all identical for each cluster from the clustering) and that the gaps between clusters amount to at least $g$.  
Let $X$ be a randomly picked point from this cluster to be used as an initial cluster center for $k$-means. 
If it happens that each initial cluster center lies in the appropriate core, then in the first iteration of $k$-means all clusters are properly formed. 

If however cluster centers lie off core then you have a chance that in the first iteration some clusters possess stranger cluster elements, but these strangers come not from the cores of other clusters. 
Hence we would be interested in getting the cluster centers into the cores in the next iteration. 
In the worst case a cluster $C$ may lose all its off-core elements to other clusters and obtain all the other off-core elements. 
 
The question is now: what portion $(1-p)$ shall be allowed to lie off-core to ensure the convergence of iteration  step. 
The answer is:

$$
 (g/2/\rho)*n_c /((g/2/\rho)*n_c   - (g/2/\rho)*  (n- n_c) +  n)   
$$
\noindent
where $n$ is the total number of elements, $n_c$ is the cardinality of the cluster.

\Bem{
(1-p)nc r=pd nc
offset = (pdnc  + r(k-1)*(1-p)*nc)
/(nc p + (k-1)*(1-p)*nc)
offset = ((1-p)nc r  + r(k-1)*(1-p)*nc)
/(nc p + (k-1)*(1-p)*nc)
offset = r(1-p) n 
/(nc p +  (1-p)*n-(1-p)nc)
offset = r(1-p) n 
/(nc p +  (1-p)*(n- nc))
(offset/r)(nc p +  (1-p)*(n- nc)) = (1-p)n
nc p +  (1-p)*(n- nc)
= nc (1-(1-p)) +  (1-p)*(n- nc)
= nc  -nc(1-p)   +  (1-p)*(n- nc)
= nc      +  (1-p)*(n- 2nc)

(offset/r) nc      +  (offset/r)(1-p)*(n- 2nc)= (1-p)n
(offset/r) nc      = (1-p)n -  (offset/r)(1-p)*(n- 2nc)
(offset/r) nc      = (1-p)(n -  (offset/r)*(n- 2nc))
1-p= (n -  (offset/r)*(n- 2nc))/((offset/r) nc)

Optimistic version:
(1-p)nc r=pd nc
offset = (pdnc  + r*(1-p)*nc)/nc
offset = ((1-p)nc r  + r*(1-p)*nc)/nc
offset = 2((1-p)nc r) /nc 
offset = 2((1-p)  r)   
offset/(2r)=1-p 

extreme but all equal version 
(1-p)nc r=pd nc
offset = (pdnc  + r*(1-p)k*nc)/(pnc+(k-1)(1-p)nc)
offset = ((1-p)nc r  + r*(1-p)k*nc)/(pnc+(k-1)(1-p)nc)
offset = (1-p)r(nc+ k*nc)/(pnc+(k-1)(1-p)nc)

p=0.8;
nc=8;
r=10; 
d= r*(1-p)*nc/(p*nc)
(1-p)*nc* r- p*d *nc
k=6

offset = (p*d*nc  + r*(k-1)*(1-p)*nc)/(nc* p + (k-1)*(1-p)*nc)
n=k*nc
  r*(1-p)* n /(nc* p +  (1-p)*(n-nc)) - offset 
(offset/r)*(nc* p +  (1-p)*(n- nc)) - (1-p)*n
(offset/r)*nc* p + (offset/r)* (1-p)*(n- nc) - (1-p)*n
+(offset/r)*nc    - (offset/r)*nc*(1-p)  + (offset/r)* (1-p)*(n- nc) - (1-p)*n
+(offset/r)*nc    - (1-p)*((offset/r)*nc   - (offset/r)*  (n- nc) +  n)
 (offset/r)*nc /((offset/r)*nc   - (offset/r)*  (n- nc) +  n)   - (1-p) 

}%Bem

\mysvli{
}%mysvli 

Clearly, with this core separation incremental $k$-means will fail usually to recover the clustering.  But    if either of the well-separatedness criterion of  core-clustering, perfect-ball-clustering or nice-ball-clustering   applies, $k$-means-random  and $k$-means++  will find the appropriate clusters, if it is seeded with one representative of each cluster. 
The theorems \ref{thm:perfect3axiomsOnek} and \ref{thm:perfect3axiomsManyk} when substituting "perfect" with "core" clustering, apply.

\subsection{$k$-richness and the problems with  realistic $k$-means  algorithms}

But what is the probability of such a   seeding of the $k$-means that each cluster has a seed?
Let us consider the $k$-means-random. 
If the share of elements in each cluster amounts to $p_1,\dots,p_k$, $p_i\ge p$ respectively, then the probability of appropriate seeding in a single run amounts to at least 
%$1*(1-p)*(1-2p)*(1-(k-1)*p)$
$q=\prod_{j=1}^{k-1}(1-(k-j)p)$. 
After say $m$ runs, we can increase the probability of appropriate seeding to $1-(1-q)^m$, and reach the required success probability of e.g. 96\%. 

Under $k$-means++, in case of at least $4\rho$ distances between clusters (perfect ball clustering) these probabilities amount to 
$$q=\prod_{j=1}^{k-1}
\frac{(3\rho)^2(k-j)p}
{(3\rho)^2(k-j)p+(2\rho)^2(1-(k-j)p)}
$$

\Bem{
Standard deviation and p for normal distribution 
sum(rnorm(n)^2>4)/n = 0.046 for n=100000

2-dim 
sum(rnorm(n)^2+rnorm(n)^2>6) /n
[1] 0.049596

3-dim
sum(rnorm(n)^2+rnorm(n)^2+rnorm(n)^2>8) /n
[1] 0.046379

4-dim
sum(rnorm(n)^2+rnorm(n)^2+rnorm(n)^2+rnorm(n)^2>10) /n

w-dim
w=10
dist=rnorm(n)^2
for (j in 2:w)  dist=dist+rnorm(n)^2; 
print(sum(dist>(w+1)*2)/n); 

print(sum(dist*(dist>(w+1)*2))/sum(dist>(w+1)*2)); 

print(sum(dist*(dist<(w+1)*2))/sum(dist<(w+1)*2));

print(sum(dist)/n); 

print(sum(dist>(w+1)*2)); 

print(sum(dist*(dist>(w+1)*2)) ); 

print(sum(dist*(dist<(w+1)*2)) );

print(sum(dist) ); 

r=sqrt((w+1)*2) 
R=4*r; 
p=1-(sum(dist>(w+1)*2)/n)

print(p * (2*(1-p)*R)^2 +(1-p)* r*(2*R-r))

}%Bem

Now it becomes obvious why the $k$-richness axiom does not make much sense. Even if the clusters should turn out to be well separated (perfect ball clustering existent), the probability of hitting a cluster with $1$ element out of $n$ with growing sample size $n$ is prohibitively small. Under $k$-means random for $l$ such small clusters it is lower than $\frac1{n^l}$.
So the number of required restarts of $k$-means will grow approximately linearly   with $n^{k-1}$, which is better than the exhaustive search with at least $k^{n-k}$ possibilities, but it is still prohibitive. 
This would render  $k$-means useless.
Respective retrial counts look significantly better for $k$-means++  but are  still unacceptable.

\subsection{$k$-means++ with dispersion off-core elements}

Alternatively we can consider the off-core elements as noise that does not need to be bounded by any ball.
The cores then are parts of the cluster such that they are enclosed into balls centered at cluster center where the distance to the other ball centers is four times the own radius of the core.  
In this case we can apply $k$-means++ with the provision of rejecting $p\cdot n$ most distant elements 
upon initialization. 
$p$ must be surely lower than the core of the smallest cluster. 
By rejecting $p$ share of elements we run at risk of removing parts of most distant cluster. 
So to keep it to be likely included in seeding 
we must keep bounded the ration of noise contribution and cluster contribution. 
Noise would be at distance $4\rho$ while the cluster at $2.5\rho$ in unfavorable case.
So to balance the contribution the noise to cluster minus noise ratio should be $2.5^2/4^2=1/2.56$
So that the noise to smallest cluster ration should be 1:3.56. 

This speaks again against the $k$-richness. 

Again   theorems analogous to  \ref{thm:perfect3axiomsOnek} and \ref{thm:perfect3axiomsManyk} apply, but now limited to the cores and not entire clusters. 
The noise allowed should not push  cluster centers off core if other clusters are seeded in cores.

%-----------------------------------------------
 \section{$k$-richness versus global minimum of $k$-means}
\label{sec:kmeansabsoluteballclustering}

Last not least let us discuss the issue whether or not we can tell that the well-separated clusters constitute the global minimum of $k$-means (recall that perfect ball clustering did not). 

We will investigate below under what circumstances it is possible to tell, without exhaustive check that the well separated clusters are the global minimum of $k$-means. We will see that the ratio between the largest and the smallest cluster cardinality plays here an important role. Therefore $k$-richness is in fact not welcome. 

In particular, let us consider
the set of 
 $k$ clusters $\mathcal{C}=\{C_1,\dots,C_k\}$ of cardinalities $n_1,\dots,n_k$ and with radii of balls enclosing the clusters (with centers located at cluster centers) $r_1,\dots, r_k$.  

We are interested in a gap $g$ between clusters such that it does not make sense to split each cluster $C_i$ into subclusters $C_{i1},\dots, C_{ik}$ and to combine them into
a set of 
 new clusters $\mathcal{S}=\{S_1,\dots,S_k\}$ such that 
$S_j=\cup_{i=1}^k C_{ij}$. 

We seek a $g$ such that the highest possible central sum of squares combined over the clusters $C_i$ would be lower than the lowest conceivable combined sums of squares around respective centers of clusters $S_j$. 
Let $Var(C)$ be the variance of the cluster $C$ (average squared distance to cluster gravity center). 
Let $r_{ij}$ be the distance of the center of subcluster $C_{ij}$ to the center of cluster $C_i$.
Let $v_{ilj}$ be the distance of the center of subcluster $C_{ij}$ to the center of subcluster $C_{lj}$.
So the total $k$-means function for the set of clusters $(C_1,\dots,C_k)$ will amount to:
\begin{equation}
Q(\mathcal{C})=\sum_{i=1}^k \sum_{j=1}^k (n_{ij}Var(C_{ij})+n_{ij}r_{ij}^2)
\end{equation}
And the total $k$-means function for the set of clusters $(S_1,\dots,S_k)$ will amount to:
\begin{equation}
Q(\mathcal{S})=\sum_{j=1}^k \left(\left(\sum_{i=1}^k n_{ij}Var(C_{ij})\right)+
({\sum_{i=1}^k n_{ij}} )
\left(\sum_{i=1}^{k-1} \sum_{l=i+1}^k  \frac{n_{ij}}{\sum_{i=1}^k n_{ij}}\frac{n_{lj}}{\sum_{i=1}^k n_{ij}} v_{ilj}^2
\right)
\right)
\end{equation}

Should $(C_1,\dots,C_k)$ constitute the absolute minimum of the $k$-means target function, then $Q(\mathcal{S})\ge Q(\mathcal{C})$  should hold, that is:
$$\sum_{j=1}^k \left(\left(\sum_{i=1}^k n_{ij}Var(C_{ij})\right)+
({\sum_{i=1}^k n_{ij}} )
\left(\sum_{i=1}^{k-1} \sum_{l=i+1}^k    \frac{n_{ij}}{\sum_{i=1}^k n_{ij}}\frac{n_{lj}}{\sum_{i=1}^k n_{ij}} v_{ilj}^2
\right)
\right)
%%%%%
$$ $$
\ge 
%%%%%
\sum_{i=1}^k \sum_{j=1}^k (n_{ij}Var(C_{ij})+n_{ij}r_{ij}^2)
$$

This implies:   
\begin{equation}\label{eq:e1}
\sum_{j=1}^k   
\left(\sum_{i=1}^{k-1} \sum_{l=i+1}^k   \frac{n_{ij}n_{lj}}{\sum_{i=1}^k n_{ij}}  v_{ilj}^2
\right)
%%%%%
\ge
%%%%%
\sum_{i=1}^k \sum_{j=1}^k  n_{ij}r_{ij}^2
\end{equation}

To maximize $\sum_{j=1}^k  n_{ij}r_{ij}^2$ for a single cluster $C_i$ of enclosing ball radius $r_i$, note that you should set $r_{ij}$ to $r_i$. Let $m_j=\arg \max_{j \in \{1,\dots,k\}} n_{ij}$. 
If we set $r_{ij}=r_i$ for all   $j$ except $m_j$, then the maximal $r_{i{m_j}}$ is delimited by the relation
$\sum_{j=1; j\ne m_j}^k  n_{ij}r_{ij}\ge n_{i{m_j}}r_{i{m_j}}$.
So 
\begin{align}\label{eq:e2}
\sum_{j=1}^k  n_{ij}r_{ij}^2\le
 (\sum_{j=1; j\ne m_j}^k  n_{ij}) r_i^2\min(2,(1+\frac{\sum_{j=1; j\ne m_j}^k  n_{ij}}{n_{i{m_j}}} ))
\\ \le &
 2 (\sum_{j=1; j\ne m_j}^k  n_{ij}) r_i^2 
\nonumber 
\end{align}

So if we can guarantee that the gap between cluster balls (of clusters from $\mathcal{C}$) amounts to $g$  then surely 

\begin{equation}\label{eq:e3}
 \sum_{j=1}^k   
\left(\sum_{i=1}^{k-1} \sum_{l=i+1}^k   \frac{n_{ij}n_{lj}}{\sum_{i=1}^k n_{ij}}  v_{ilj}^2
\right)
\ge 
g^2
\sum_{j=1}^k   
\left(\sum_{i=1}^{k-1} \sum_{l=i+1}^k   \frac{n_{ij}n_{lj}}{\sum_{i=1}^k n_{ij}} 
\right)
\end{equation}
because in such case $g\le v_{ilj}$ for all $i,l,j$.

By combining inequalities (\ref{eq:e1}),  (\ref{eq:e2}) and (\ref{eq:e3}) we see 
that the global minimum is granted if the following holds:
\begin{equation}\label{eq:globalg}
g^2
\sum_{j=1}^k   
\left(\sum_{i=1}^{k-1} \sum_{l=i+1}^k   \frac{n_{ij}n_{lj}}{\sum_{i=1}^k n_{ij}}   
\right)
%%%%%
\ge
%%%%
 2 \sum_{i=1}^k (\sum_{j=1; j\ne m_j}^k  n_{ij}) r_i^2 
\end{equation}

One can distinguish two cases: either 
(1) there exists a cluster $S_t$ containing two subclusters $C_{pt}$, $C_{qt}$ 
such that $t=\arg \max_j |C_{pj}|$ 
and $t=\arg \max_j |C_{qj}|$ 
(maximum cardinality   subclasses of their respective original clusters $C_p, C_q$
 or (2) not. 

Consider the first case. Let $C_p,C_q$ be the two clusters where $C_{pt}$ and $C_{qt}$ be two subclusters of highest cardinality within $C_p,C_q$ resp. 
This implies that $n_{pt}\ge \frac 1k n_p, n_{qt}\ge \frac 1k n_q$. 
Also this implies that for $i\ne p, i\ne q$  $n_{it}\le n_i/2$.

$$
\sum_{j=1}^k   
 \sum_{i=1}^{k-1} \sum_{l=i+1}^k   \frac{n_{ij}n_{lj}}{\sum_{i=1}^k n_{ij}}  
$$
$$
\ge   
 \sum_{i=1}^{k-1} \sum_{l=i+1}^k   \frac{n_{it}n_{lt}}{\sum_{i=1}^k n_{it}}  
$$
$$
\ge 
    \frac{n_{pt}n_{qt}}{\sum_{i=1}^k n_{it}}  
$$
$$
\ge 
  \frac{n_{pt}n_{qt}} {n_p/2+n_q/2+\sum_{i=1}^k n_{i}/2 }  
=   \frac{n_{pt}n_{qt}}{n_p/2+n_q/2+n/2}   
$$

$$
\ge \frac1{k^2}  \frac{n_{p}n_{q}}{n_p/2+n_q/2+n/2}   
$$

Note that 
$$
2 \sum_{i=1}^k (\sum_{j=1; j\ne m_j}^k  n_{ij}) r_i^2 
\le 2 \sum_{i=1}^k  n_{i} r_i^2 
$$
So, in order to fulfill inequality (\ref{eq:globalg}), it is sufficient to require that  
\begin{equation}\label{eq:globalgcase1}
g\ge 
\sqrt{
\frac{2 \sum_{i=1}^k  n_{i} r_i^2 }
{ \frac1{k^2}  \frac{n_{p}n_{q}}{n_p/2+n_q/2+n/2} } 
}
= 
k\sqrt{n_p/2+n_q/2+n/2} \sqrt{
\frac{2 \sum_{i=1}^k  n_{i} r_i^2 }{    n_{p}n_{q} }
}
= 
k\sqrt{n_p +n_q +n } \sqrt{
\frac{  \sum_{i=1}^k  n_{i} r_i^2 }{    n_{p}n_{q} }
}
\end{equation}
This of course maximized over all combinations of $p,q$. 

Let us proceed to the second case.
Here each cluster $S_j$ contains a subcluster of maximum cardinality of a different cluster $C_i$. 
As the relation between $S_j$ and $C_i$ is unique, we can reindex $S_j$ in such a way that actually $C_j$ contains its maximum cardinality subcluster $C_{jj}$. 
Let us rewrite the  inequality (\ref{eq:globalg}). 

$$
g^2
\sum_{j=1}^k   
\left(\sum_{i=1}^{k-1} \sum_{l=i+1}^k   \frac{n_{ij}n_{lj}}{\sum_{i=1}^k n_{ij}}   
\right)
-
 2 \sum_{i=1}^k (\sum_{j=1; j\ne m_j}^k  n_{ij}) r_i^2 
%%%%%
\ge 0
%%%%
$$

This is met if 

$$
g^2
\sum_{j=1}^k   
\left(\sum_{i=1}^{j-1}  \frac{n_{ij}n_{jj}}{\sum_{i=1}^k n_{ij}} 
+
\sum_{l=j+1}^k   \frac{n_{jj}n_{lj}}{\sum_{i=1}^k n_{ij}}   
\right)
-
 2 \sum_{i=1}^k (n_i-  n_{ii}) r_i^2 
%%%%%
\ge 0
%%%%
$$
This is the same as:

$$
g^2
\sum_{j=1}^k   
\left(\sum_{i=1,\dots, {j-1},{j+1},\dots,k}   \frac{n_{ij}n_{jj}}{\sum_{i=1}^k n_{ij}}    
\right)
-
 2 \sum_{i=1}^k (n_i-  n_{ii}) r_i^2 
%%%%%
\ge 0
%%%%
$$

This is fulfilled if:

$$
g^2
\sum_{j=1}^k   
\left(\sum_{i=1,\dots, {j-1},{j+1},\dots,k}   \frac{n_{ij}n_{j}/k}
{n_j/2+\sum_{i=1}^k n_{i}/2}    
\right)
-
 2 \sum_{i=1}^k (n_i-  n_{ii}) r_i^2 
%%%%%
\ge 0
%%%%
$$

Let $M$ be the maximum over $n_1,\dots,n_k$. The above holds if 

$$
g^2
\sum_{j=1}^k   
\left(\sum_{i=1,\dots, {j-1},{j+1},\dots,k}   \frac{n_{ij}n_{j}/k}
{M/2+n/2}    
\right)
-
 2 \sum_{i=1}^k (n_i-  n_{ii}) r_i^2 
%%%%%
\ge 0
%%%%
$$
Let $m$ be the minimum over $n_1,\dots,n_k$. The above holds if 
$$
g^2
\sum_{j=1}^k   
\left(\sum_{i=1,\dots, {j-1},{j+1},\dots,k}   \frac{n_{ij}m/k}
{M/2+n/2}    
\right)
-
 2 \sum_{i=1}^k (n_i-  n_{ii}) r_i^2 
%%%%%
\ge 0
%%%%
$$
This is the same as 
$$
g^2   \frac{m/k}
{M/2+n/2}    
\left(
\sum_{j=1}^k   
\sum_{i=1,\dots, {j-1},{j+1},\dots,k}   {n_{ij} }   
\right)
-
 2 \sum_{i=1}^k (n_i-  n_{ii}) r_i^2 
%%%%%
\ge 0
%%%%
$$

$$
g^2   \frac{m/k}
{M/2+n/2}    
\left(
\sum_{j=1}^k
\left(   
\left(   
\sum_{i=1}^k   {n_{ij} } \right) - n_{jj}  
\right)
-
 2 \sum_{i=1}^k (n_i-  n_{ii}) r_i^2 \right)
%%%%%
\ge 0
%%%%
$$

$$
g^2   \frac{m/k}
{M/2+n/2}    
\left(   
\left(   
\sum_{j=1}^k\sum_{i=1}^k   {n_{ij} } \right) - (\sum_{j=1}^kn_{jj} ) 
\right)
-
 2 \left( \sum_{i=1}^k (n_i-  n_{ii}) r_i^2 \right)
%%%%%
\ge 0
%%%%
$$

$$
g^2   \frac{m/k}{M/2+n/2}    
\left(   
\left(   
\sum_{i=1}^k   {n_{i} } \right) - (\sum_{j=1}^kn_{jj} ) 
\right)
-
 2 \sum_{i=1}^k (n_i-  n_{ii}) r_i^2 
%%%%%
\ge 0
%%%%
$$

$$
g^2   \frac{m/k}
{M/2+n/2}    
\left(      
\sum_{i=1}^k  \left( {n_{i} -n_{ii}} \right)  
\right)
-
 2 \sum_{i=1}^k (n_i-  n_{ii}) r_i^2 
%%%%%
\ge 0
%%%%
$$

$$
\sum_{i=1}^k  \left( {n_{i} -n_{ii}} \right) \left(
g^2   \frac{m/k}
{M/2+n/2}    
-
 2   r_i^2 
\right)
%%%%%
\ge 0
%%%%
$$
The above will hold, if for every $i=1,\dots,k$
$$g \ge r_i \sqrt{\frac{2}{  \frac{m/k}{M/2+n/2} }} $$
\begin{equation}\label{eq:globalgcase2} 
g \ge r_i \sqrt{k\frac{M+n}{   m} }
\end{equation}

So the  inequality (\ref{eq:globalg}) is fulfilled, if both 
 inequality (\ref{eq:globalgcase1}) and  inequality (\ref{eq:globalgcase2}) are held by an appropriately chosen $g$.

So we may call the above-mentioned "well-separatedness" as "absolute clustering". 
One sees immediately that 
inner cluster consistency is kept, this time in terms of global optimum, under the restraint to $k$ clusters.

\begin{theorem}{}\label{thm:absolute3axiomsOnek}
$k$-means, if restricted to absolute clusterings,
 conforms globally to $k$-richness,  scale-invariance, 
 motion consistency and inner cluster  consistency. 
\end{theorem}

 Regrettably, a structure may emerge upon such consistency and therefore 
the maximal number of possible absolute clusters is not kept. 
However, if we apply centric consistency, the max-$k$-means[absolute] keeps the richness/invariance/motion consistency axioms.

\Bem{
If we consider a variant of $k$-means with varying $k$ over a broad spectrum of $k$, and take as the final clustering the absolute ball clustering into the largest number of clusters possible, and instead of inner cluster consistency the centric consistency is used then an approximation to near-richness can be achieved. 
}%Bem

\begin{theorem}{}
$k$-means, with $k$ ranging over a set of values, if we assume that it returns the    absolutely separated clusterings for the largest possible $k$ (excluding cluster sizes considered as too small), then it conforms  globally to richness,  scale-invariance, motion consistency and centric consistency. 
\end{theorem} 

In the end let us make a remark on Theorem \ref{thm:absolute3axiomsOnek}.
If one applies Kleinberg's consistency transformation in Euclidean space, 
not continuously of course, because it is not possible, as already shown, 
but in a discrete manner, with "jumping" clusters, 
then this transform can be represented as (again in discrete manner) a superposition of 
motion consistency transform and inner cluster consistency transform. 
The reason is as follows:
Consider a cluster $C$ and a point $\mathbf{x}$ from another cluster $C_2$. 
Let us compute the distance between $\mathbf{x}$ and the cluster center $\boldsymbol\mu_C$ of the cluster $C$ prior and after Kleinberg's consistency transformation to see that it increases. 
Consider the  distance 
$\|\mathbf{x}-\boldsymbol\mu_C\|^2$.
It may be expressed as a multiple (factor $\frac{|C|+1}{|C|}$) of the distance between $\mathbf{x}$ and the center $\boldsymbol\mu$ of the   data set $C\cup \{\mathbf{x}\}$. 
And $\|\mathbf{x}-\boldsymbol\mu \|^2= \frac{1}{|C|+1}\sum_{\mathbf{y}\in C} \|\mathbf{x}-\mathbf{y}\|^2$.
Hence it is obvious that  increasing distance between $\mathbf{x}$ and elements of $C$, we increase also the distance of $\mathbf{x}$ to the cluster center of $C$.  

So one can generalize that also the distances between clusters $C,C_2$ increase under Kleinberg's consistency transformation. Hence in fact any Kleinberg's consistency transformation can be represented as a superposition of the mentioned transforms. 
 
This means that 
\begin{theorem}{} 
$k$-means, if restricted to absolute clusterings,
 conforms globally to $k$-richness, invariance and    consistency axioms. 
\end{theorem} 

Furthermore, let us relax a bit the centric consistency.

\begin{ax}\label{ax:innerclusterproportionalconsistency}
A method matches the condition of \emph{inner cluster proportional consistency}
if after 
decreasing distances within a cluster by the same factor, 
specific to each cluster, while keeping the position of cluster center in space,
   it returns the same partition. 
\end{ax}

\begin{theorem}{}
$k$-means, with $k$ ranging over a set of values, if we assume that it returns the    absolutely separated clusterings for the largest possible $k$ (excluding cluster sizes considered as too small), then it conforms  globally to richness,  scale-invariance, motion consistency and inner cluster proportional consistency. 
\end{theorem} 

Note that motion consistency and inner cluster proportional consistency include as a special case the outer-consistency. 
So in this way we denied the theorem from \cite{Ackerman:2010NIPS} that 
"no  general  clustering  function  can  simultaneously  satisfy  outer-consistency,  scale-
invariance, and richness".

Let us make at this point a remark why we insist on 
inner cluster proportional  consistency.
A reasonable assumption for consistency transformation would be that no possible partition of a given cluster being subject to consistency transformation would take advantage of the consistency transformation, so that no new substructures would occur in the cluster. In the context of $k$-means this would mean the following. Consider a cluster $C$ of a partition $\Gamma$ of a whole set, say $S$, a distance $d$ prior to a consistency transformation and a distance $d_\Gamma$ after the consistency transformation.     
Consider alternative   partitions $\Gamma_1$ and $\Gamma_2$ of $C$. 
Let $Q(\Gamma,d)$ denote the quality function $Q(\Gamma)$ under the distance $d$. 
So we would expect that 
$\frac{Q(\Gamma_1,d_\Gamma)}{Q(\Gamma_2,d )}=\frac{Q(\Gamma_2,d_\Gamma)}{Q(\Gamma_2,d )}$ unless we have a trivial partition such that each element is in separate cluster. 
This should hold for any pair of partitions of $C$, including the following ones:
$\Gamma_1$ puts all points into separate clusters except for $\mathbf{x}, \mathbf{y}$ which go into a single cluster, 
$\Gamma_2$ puts all points into separate clusters except for $\mathbf{x}, \mathbf{z}$. Let $\lambda_1,\lambda_2\in (0,1]$ be coefficients by which distances $\|\mathbf{x}-\mathbf{y}\|$, $\|\mathbf{x}-\mathbf{z}\|$ are shortened respectively under consistency transformation. 
So we will have the requirement 
$\frac{\|\lambda_1(\mathbf{x}-\mathbf{y})\|^2/2} 
      {\|          \mathbf{x}-\mathbf{y} \|^2/2} 
=\lambda_1^2=
 \frac{\|\lambda_2(\mathbf{x}-\mathbf{z})\|^2/2} 
      {\|          \mathbf{x}-\mathbf{y} \|^2/2} 
=\lambda_2^2
$ 
which means that $\lambda_1=\lambda_2$. 
By induction over the whole set $C$ we see that consistency transformation would need to shorten all distances within $C$ by the same factor.
This result justifies inner cluster proportional consistency concept, with a special case of centric consistency.

With respect to Kleinberg's consistency, we can say  

\begin{theorem}{}
$k$-means, with $k$ ranging over a set of values, if we assume that it returns the    absolutely separated clusterings for the largest possible $k$ (excluding cluster sizes considered as too small), then it conforms  globally to richness,  scale-invariance and unidirectional refinement consistency. 
\end{theorem}

Let us inspect the effect of $k$-richness in both described cases.
From inequality (\ref{eq:globalgcase2}) we see that a large discrepancy 
between the maximum and minimum cluster size implies that the 
gap $g$ between clusters needs to grow to get absolute clustering.
From inequality (\ref{eq:globalgcase1}) we see something similar, but this time the relation between the smallest cluster and the overall number of elements in the sample play the dominant role.
Additionally, the gap size is impacted by the number of clusters.  

So once again it is visible that $k$-richness is unfavorable for clustering process.

%-----------------------------------------------

\section{Conclusions and future work} \label{sec:conclusions} 

\mysvli{ %PROOF THAT THE CONSTRUCTION WORKS FOR RANDOM KMEANS - the k-richness construction 
%We see that there exist many ways to reconcile $k$-means with Kleinberg's axioms ...
}%mysvli 

In this paper, contrary to results of former researchers,
we reached the conclusion, that $k$-means algorithm 
can comply simultaneously to Kleinberg's  $k$-richness, 
scale-invariance and consistency axioms. 
A variant of $k$-means 
can comply simultaneously to Kleinberg's   richness, 
scale-invariance and refinement consistency axioms. 
The same variant of $k$-means can even comply to 
richness, 
scale-invariance and motion plus inner proportional consistency axioms. 
where the last two axioms pretty well approximate Kleinberg's consistency without creating a risk of emergence of new structures within a cluster. 

These new results emerged from the insight that our understanding of clustering process is to separate clusters with gaps. 

%In this paper we have investigated reasons why the Kleinberg's axiomatic system for clustering functions does not describe properly a real life clustering algorithm, $k$-means.  

As has been pointed at in earlier work of other researchers. 
$k$-means, like many other algorithms, is  appropriately described neither by the richness-axiom nor by the consistency axiom of Kleinberg.  

%Therefore, its relaxation to $k$-richness was suggested. 
As richness is concerned, 
  already Ackerman \cite{Ackerman:2010NIPS} showed that properties like stability against malicious attacks requires balanced clusters, hence $k$-richness is counterproductive when seeking stable clusterings. 

In this paper we pointed at a number of further problems with  the richness or near-richness axiom by itself. The major ones are: (a) the huge space to search through under "hostile" clustering criterion, 
(b) problems with ensuring learnability of the concept of a clustering for the population, (c) richness and scaling-invariance alone may lead to a contradiction for a special case. 

But we showed also that resorting to $k$-richness, which was deemed as a remedy to Kleinberg's Impossibility Theorem, does not resolve all problems:
\begin{itemize}
\item The initial seeding of cluster centers becomes extremely difficult both for $k$-means-random and $k$-means++ given that the cluster sizes differ extremely.
\item Even if we restrict ourselves to perfect ball clusterings realm, large differences in cluster sizes are prohibitive for a successful seeding.
\item For perfect ball clusterings with noise, even the smallest clusters require a high cluster size to noise size ratio.
\item In the realm of absolute clusterings, a high ratio between the lowest and the largest cluster result in high required gaps between clusters. 
\end{itemize}

We showed also that the consistency axiom constitutes a problem: 
neither consistency, nor inner-consistency nor outer-consistency can be executed continuously in Euclidean space of limited dimension. 
Therefore, as a substitute of the inner-consistency, we proposed centric consistency and showed that $k$-means has the property of centric consistency. 

When investigating a substitute for outer-consistency, the motion consistency, we showed that (a) a gap between clusters is necessary for them to have a motion consistency with $k$-means, (b) the shape of the cluster counts - it has to be enclosed in a ball for $k$-means. 

Therefore we investigated further the impact of the gap on the behavior of the $k$-means in the light of Kleinberg's axioms. 
We showed that perfect ball clustering is a local minimum for $k$-means function so that for perfect ball clusterings axioms of invariance, k-richness, inner cluster consistency and motion consistency hold 
(the last pair as a fair substitute of the consistency). 
If we consider a variant of $k$-means with varying $k$ over a broad spectrum of $k$, and take as the final clustering the perfect ball clustering into the largest number of clusters possible, and instead of inner cluster consistency the centric consistency is used then an approximation to near-richness can be achieved.%  
\footnote{We would exclude clusters with several cluster members on the grounds of the fact that statistically speaking we want to be sure that the probability of an element occurring in the gap should be smaller than in a cluster, say p times. So if we have n elements in a cluster and none in the gap. then we should have $\left(\frac{p}{p+1}\right)^n \le 0.05$ for example. With p=10, we need a minimal cluster size of at least n=32.}

Again for $k$-means-random and $k$-means++ the $k$-richness (big variation of cluster sizes) constitutes a problem for appropriate seeding. Seeding becomes more important with gaps because gaps may prohibit recovery from inappropriate seeding.

We investigated absolute clustering realm that is space where perfect ball clusterings turn to global minimum for $k$-means. $k$-richness requirement widens the gaps between clusters that are necessary. 
Axiomatic behavior does not differ much from  that of perfect ball clusterings except for the fact that after the transformations we remain in the real of absolute clusterings.

The introduction of gaps draws our attention  to one important issue: the broader the gaps the more are the clustering properties close to Kleinberg's axioms. But this happens at a price of violating some Kleinberg's implicit assumptions: that the clustering function always returns a clustering. 
Let us illustrate the point with the incremental clustering algorithms of Ackerman and Dasgupta. They prove theorems of the form: "If a perfect clustering exists, then the algorithm returns it". 
But the question is not raised: what does the algorithm return if the clustering is not perfect? 
Their algorithms return "something". 
We do not agree with such an approach. If the clustering type the algorithm is good looking at may not exist, the algorithm should state: "I found the clustering of this type / I did not find the clustering of this type / The clustering of the given type does not exist". This would be a response from an ideal algorithm. A worse one, but still usable, would give one of the first two answers. 
In this investigation we show that a post-processing for $k$-means would be capable to answer the question, whether the found clustering is a nice-ball-clustering, perfect-ball-clustering,  or absolute-clustering, or none of them. 

Better types of algorithms should provide with more diagnostics, concerning violations of shape, gap sizes, risks resulting from unbalanced cluster sizes and/or radii. 

So the first conclusion is that the clustering algorithm should respond that either a clustering of required type was found or not found (along with the clustering). 

The second question is what is the type of clustering we are looking for?
It is a bad habit to run $k$-means over and over again and stop when the lowest value of the quality function was reached. 
But this clustering may be worse than ones generated in-between, e.g. if a perfect-ball-clustering exists, it may become a victim of the unbalanced cluster sizes. 

But what we are looking for may become also a victim of the transformations Kleinberg is proposing. 
As the natural clusters returned by $k$-means are preferably ball-shaped, centric consistency transformation and the motion consistency transformation and Kleinberg consistency transform preserve them, when the gaps conform to perfect-ball or absolute separation. 
If Voronoi diagrams are to be shapes, then motion consistency transformation and Kleinberg's consistency transformations are destructive. 
If, however, any connected, well separated area would be deemed  a good cluster, then even centric consistency transformation may turn out to be disastrous. 
Kleinberg's consistency transformation is disastrous by itself (especially under richness expectation), as it can create new cluster like structures not present in the original data. 

A similar statement may be made about the richness or any related axioms. The requirement of a too rich space of hypotheses imposes a too heavy burden on the clustering algorithms. One shall instead envision hypotheses spaces that are just rich enough and are still learnable, and where the decision is possible if we are still in the hypotheses space with our solution.

So, we disagree to some extent with the opinion expressed in \cite{vanLaarhoven:2014,Ben-David:2009} that axiomatic systems deal with either  clustering  functions, or  clustering quality function, or relations of quality of partitions. 
The particular formulations of axiomatic systems state rather the equivalence relations between the clusterings themselves.
Hence we must first have an imagination what kind of clusters we are looking for and only then formulate the axioms with transformations that are reasonable within the target class of clusterings, do not lead outside of this class and equivalence or other relations between clusterings makes sense within this class and does not need to be defined outside.   
 
Hence there is still much space for research on clustering axiomatization, especially for clarification, what types of clusters are of real interest and whether or not all of them can be axiomatised in the same way. 
Kleinberg pointed at the problem and is was a good starting point.

\section*{Acknowledgments}
%\begin{acknowledgment}
The authors wish to thank to the Institute of Computer Science of Polish Academy of Sciences for promoting and financing this research.
Research done by  Robert A. K{\l}opotek was financed by research fellowship within Project 'Information technologies: research and their interdisciplinary applications', agreement number UDA-POKL.04.01.01-00-051/10-00. 
 
%\end{acknowledgment}

\bibliography{centricconsistencyRAKMAK_bib}

\end{document}